\documentclass[lettersize,journal]{IEEEtran}
\usepackage{amsmath,amsfonts}
\usepackage{algorithm}
\usepackage{amsthm}
\usepackage{array}
\usepackage[caption=false,font=normalsize,labelfont=sf,textfont=sf]{subfig}
\usepackage{stfloats}
\usepackage{url}
\usepackage{enumitem}
\usepackage{graphicx}
\usepackage{cite}
\usepackage[table]{xcolor}
\usepackage{multirow}
\usepackage{rotating}
\usepackage{makecell}
\usepackage{hyperref}
\usepackage{natbib}
\usepackage{comment}
\usepackage{booktabs} 
\usepackage{wrapfig}
\usepackage{tikz}
\usepackage{overpic}  
\usepackage{algpseudocode}
\usepackage{amssymb}
\newtheorem{theorem}{Theorem}

\newtheorem*{theorem*}{Theorem}
\newtheorem*{proof*}{Proof}
\definecolor{academicblue}{RGB}{235, 245, 255}
\usepackage{xspace}

\makeatletter
\DeclareRobustCommand\onedot{\futurelet\@let@token\@onedot}
\def\@onedot{\ifx\@let@token.\else.\null\fi\xspace}

\makeatother

\definecolor{darkorange}{RGB}{255,140,0}

\title{Unbiased Diffusion Variational Inversion via Principled Posterior Matching}

\author{
  \IEEEauthorblockN{Weimin Bai\textsuperscript{*},}
  \and
  \IEEEauthorblockN{Yuxuan Gu\textsuperscript{*},}
  \and
  \IEEEauthorblockN{Yifei Wang,}
  \and
  \IEEEauthorblockN{Weijian Luo,}
  \and
  \and
  \IEEEauthorblockN{He Sun\textsuperscript{\dag}}
  \\
  \IEEEauthorblockA{Peking University}
  \thanks{\textsuperscript{*}Equal contribution}
  \thanks{\textsuperscript{\dag}Correspondence to hesun@pku.edu.cn}
}

\begin{document}

\maketitle

\begin{abstract}

Existing score-based methods for inverse problems often resort to approximate minimization of the KL divergence between the inversion distribution and the Bayesian posterior. Such an approximation leads to severe mode collapse and unreliable uncertainty quantification. In this paper, we propose \emph{Principled Posterior Matching (PPM)}, a framework that returns to the fundamentals of variational inference, rather than using tricky approximations. Instead of relying on heuristic approximations, we rigorously formulate the exact optimization of the KL divergence via the \emph{integration of Fisher divergence}. We derive a tractable, equivalent gradient form of this integral, enabling precise optimization without the biases introduced by prior approximations. Our analysis clearly reveals that the mode collapse in previous methods stems directly from this approximation gap. \emph{Supported by our theoretical solution, PPM unifies two complementary paradigms}: (1) In \emph{variational inference}, PPM adopts mass-covering divergences that significantly improve the inversion diversity and uncertainty quantification; (2) In \emph{amortized inference}, it enables the training of an efficient reconstruction network for rapid, single-step reconstruction. Furthermore, our formulation naturally extends to a broader family of divergence measures by generalizing the integral of the Fisher divergence. We validate PPM across challenging computational imaging tasks, including inpainting, super-resolution fluorescent microscopy, and radio interferometric black-hole imaging. In all experiments, PPM achieves superior reconstruction fidelity, faithful multimodal posterior recovery, and well-calibrated uncertainty estimates, establishing a robust framework for scientific imaging.

\end{abstract}






\vspace*{-0.2cm}
\section{Introduction}
\label{sec:introduction}
\IEEEPARstart{C}{omputational} imaging tasks seek to recover an unknown signal $x$ from noisy, indirect measurements,
$\boldsymbol{y} = \mathcal{A}(\boldsymbol{x}) + \eta$,
where $\mathcal{A}$ is a linear or nonlinear, ill-posed forward operator, and $\eta$ denotes measurement noise. 
Computational imaging finds applications in diverse scientific fields such as astronomy~\cite{chael2019eht}, optical microscopy~\cite{choi2007tomographic}, medical imaging~\cite{lustig2008compressed}, and fluid dynamics~\cite{iglesias2013ensemble}.
A crucial aspect of imaging in this context is the incorporation of an image prior, or so-called regularization, to guide the reconstruction towards desired image characteristics.
From a Bayesian perspective, the prior shapes the posterior distribution, namely the uncertainty and diversity, of the reconstructed images.
Classical methods impose handcrafted priors on $x$—for example, sparsity~\cite{candes2007sparsity}, total variation (TV)~\cite{vogel1996iterative}, or wavelet-based regularizers—to constrain the solution space. 
While effective in some settings, these analytical priors usually fail to capture the rich statistics of natural or scientific images.

Recent advances in generative AI have established diffusion models~\cite{ho2020denoising, song2020score, sohl2015deep} as powerful, data-driven image priors. 
A diffusion model learns to approximate the distribution of clean images by progressively adding noise (the forward process) and then training a neural network to reverse this corruption through denoising (the reverse process). When integrated into inverse problem solvers, diffusion priors not only yield high-quality point estimates but also facilitate full posterior exploration for uncertainty quantification, a capability crucial in scientific and medical imaging applications.

Existing approaches for posterior sampling from $p(\boldsymbol{x}|\boldsymbol{y})$ with a pretrained diffusion prior primarily fall into three categories: gradient-guided Monte Carlo sampling, optimization-based variational inference, and amortized inference. The first category comprises gradient-guided Monte Carlo (MC) techniques, exemplified by Diffusion Posterior Sampling (DPS)~\cite{chung2022diffusion}. DPS iteratively interleaves denoising updates from a score-based model with gradient-driven data-consistency steps within a Markov Chain Monte Carlo (MCMC) framework~\cite{brooks2011handbook}. While this procedure seeks to satisfy both the measurement likelihood $p(\boldsymbol{y}|\boldsymbol{x})$ and the learned image prior, it enforces data consistency by approximating the intractable time-dependent likelihood with strong assumptions. This often results in unstable sampling trajectories and biased posterior estimates. Moreover, these methods remain inherently inefficient, as their sequential nature incurs substantial computational costs during inference.

\begin{figure*}[t]
    \centering
    \setlength{\tabcolsep}{1pt}
    \setlength{\fboxrule}{1pt}
    \resizebox{0.99\textwidth}{!}{
    \begin{tabular}{c}
    \begin{tabular}{ccc@{\hspace{2pt}}!{\color{black}\vrule width 1.0pt}@{\hspace{2pt}}cccc}
        GT Prior & Diffusion Prior & Data Likelihood &
        RED-Diff & Pixel-RLSD & Ours & GT Posterior
        \\
        \includegraphics[width=0.2\textwidth]{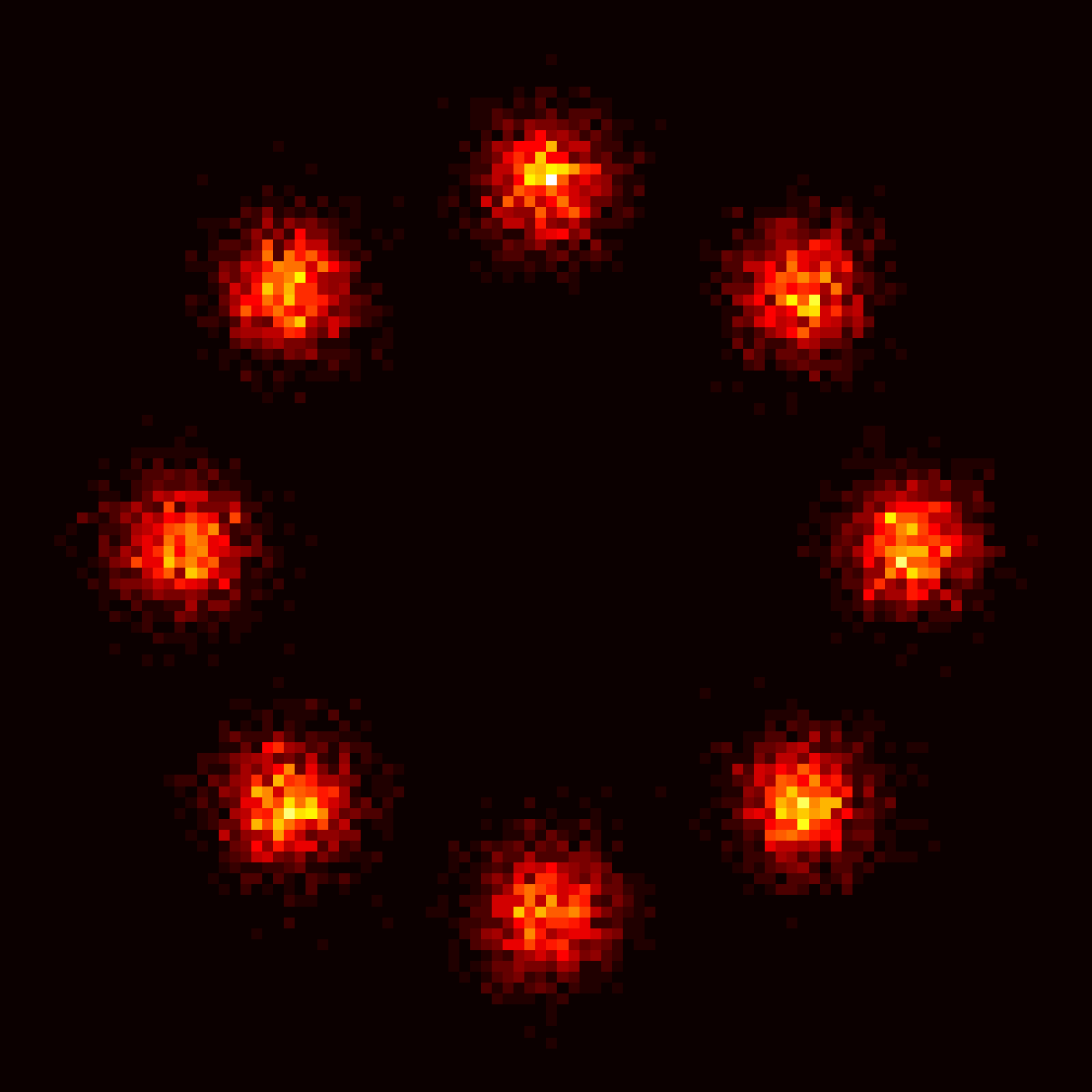} &
        \includegraphics[width=0.2\textwidth]{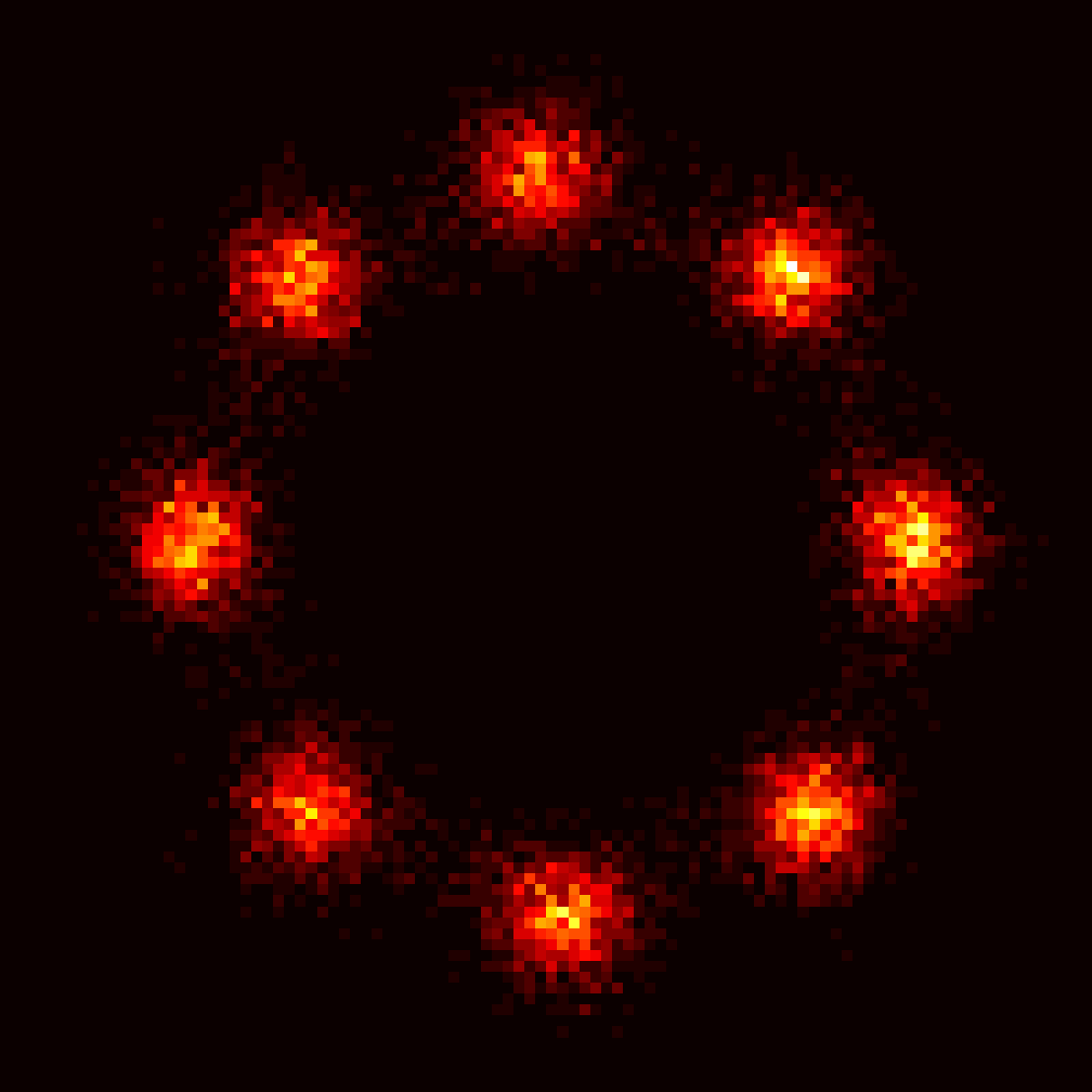} &
        \includegraphics[width=0.2\textwidth]{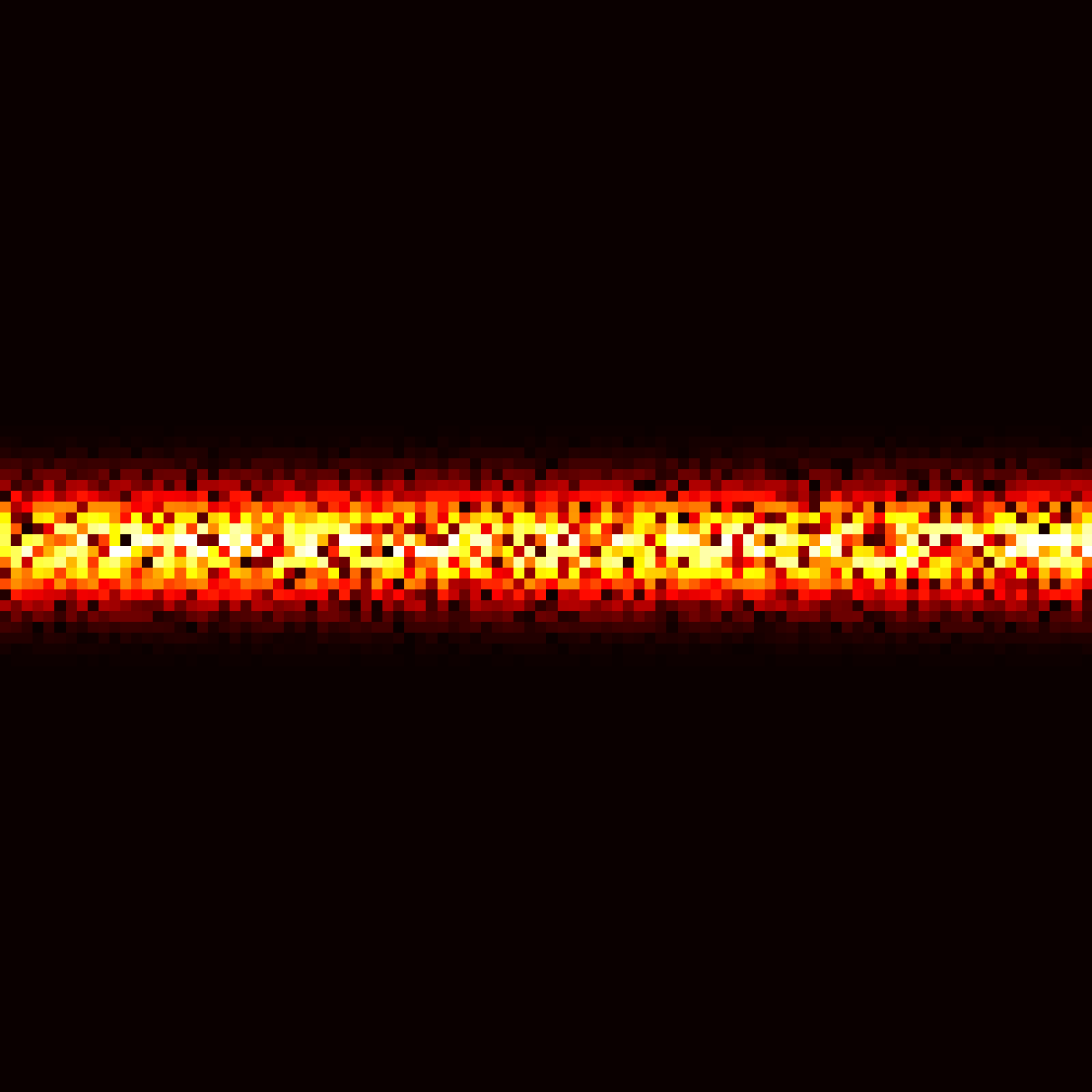} &
        \includegraphics[width=0.2\textwidth]{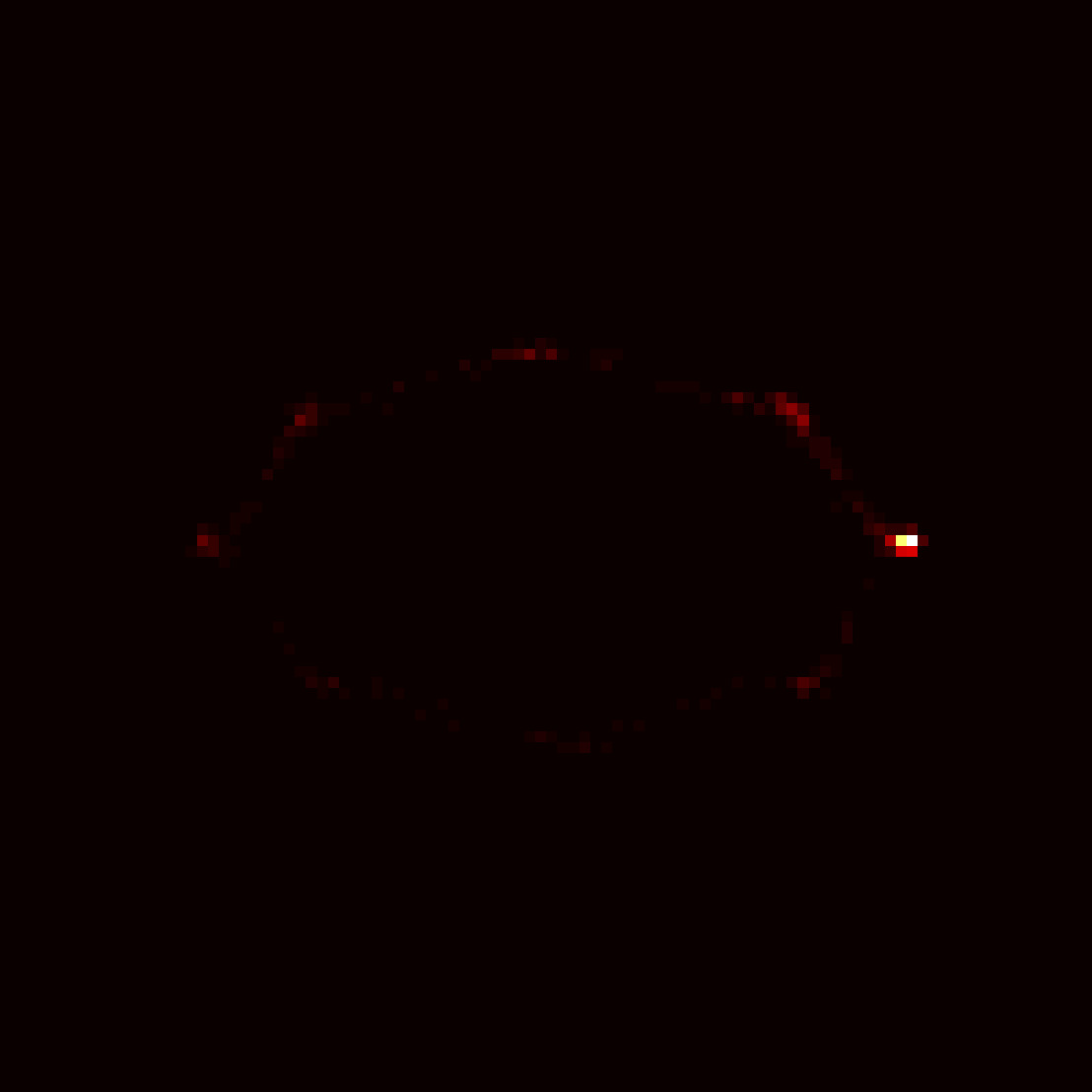} &
        \includegraphics[width=0.2\textwidth]{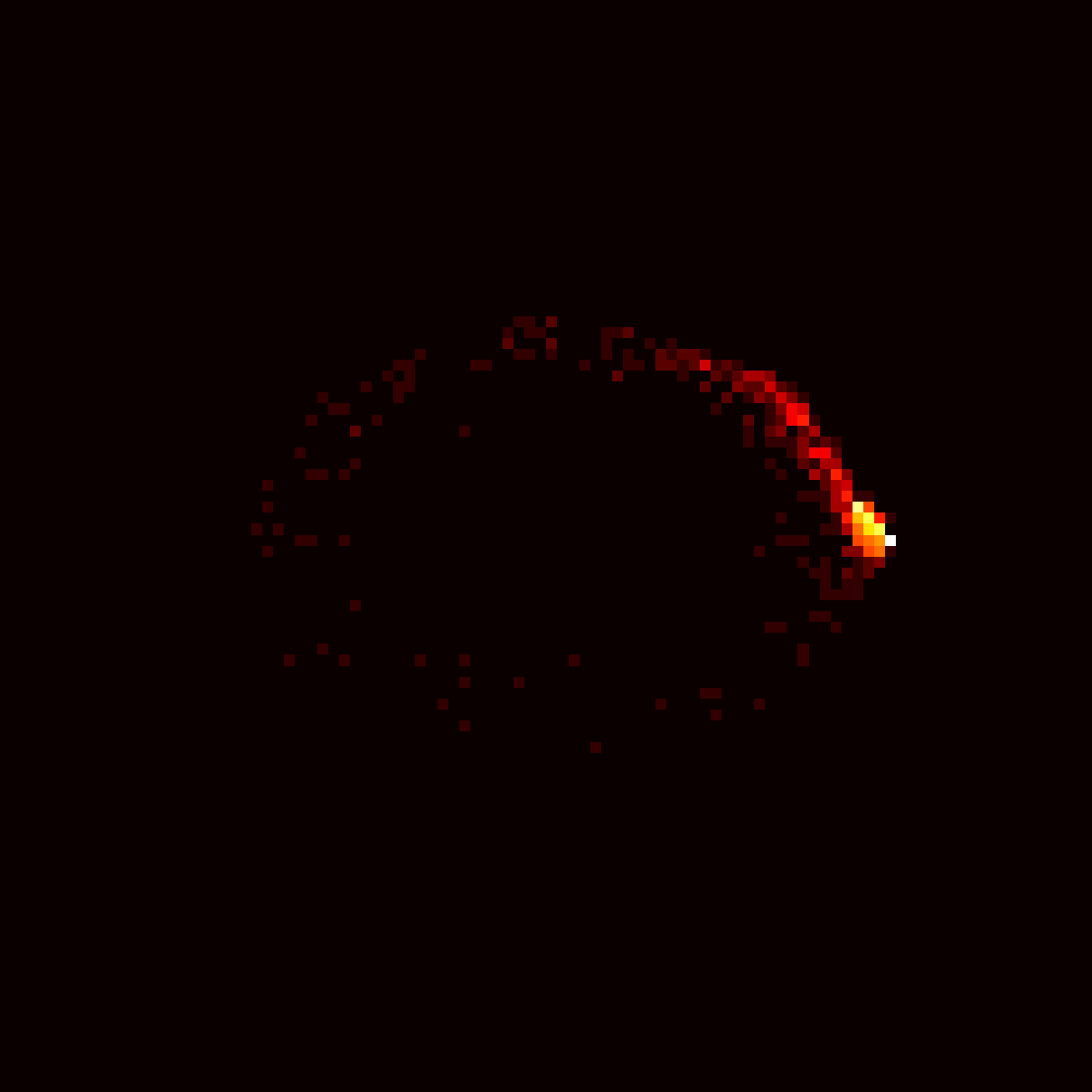} &
        \includegraphics[width=0.2\textwidth]{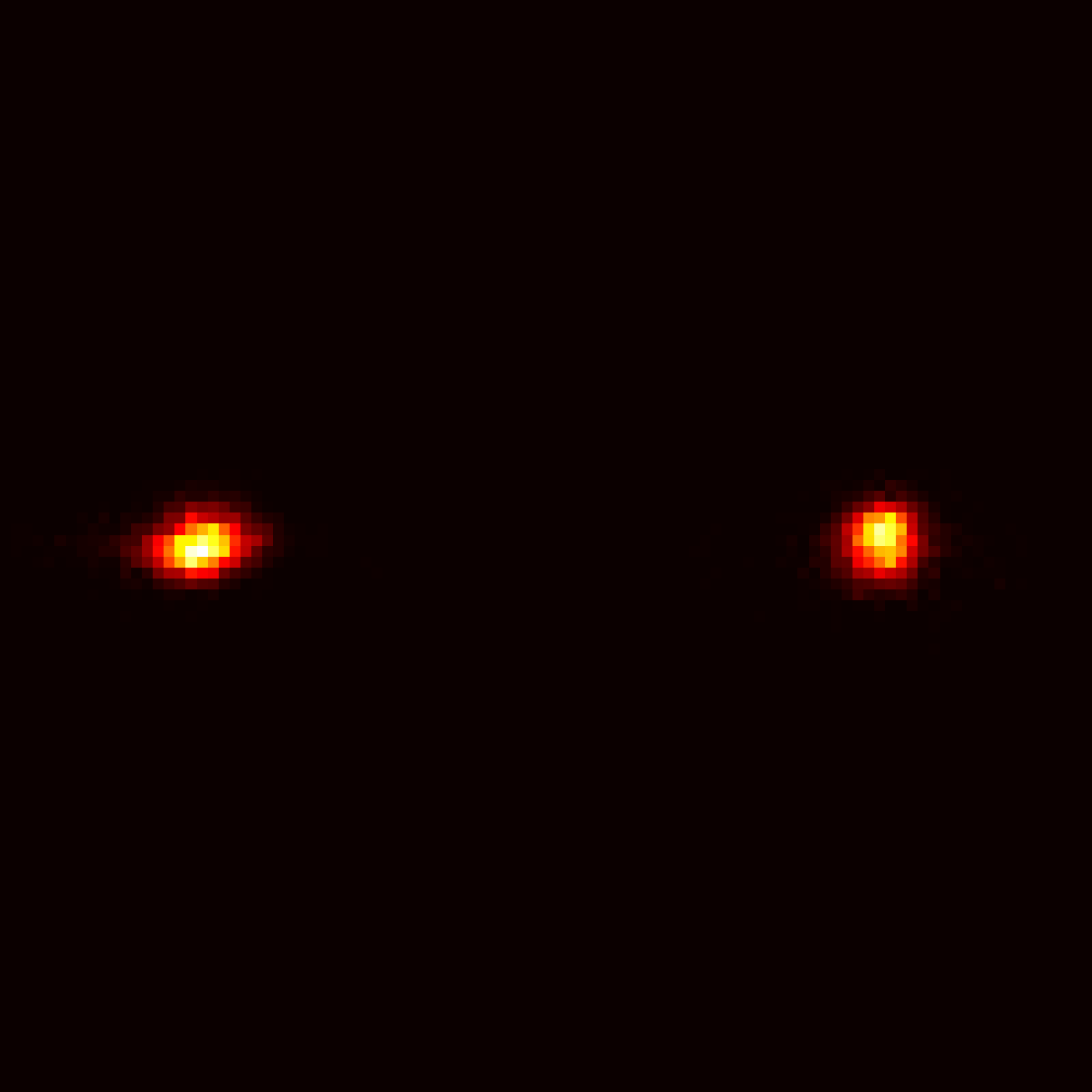} &
        \includegraphics[width=0.2\textwidth]{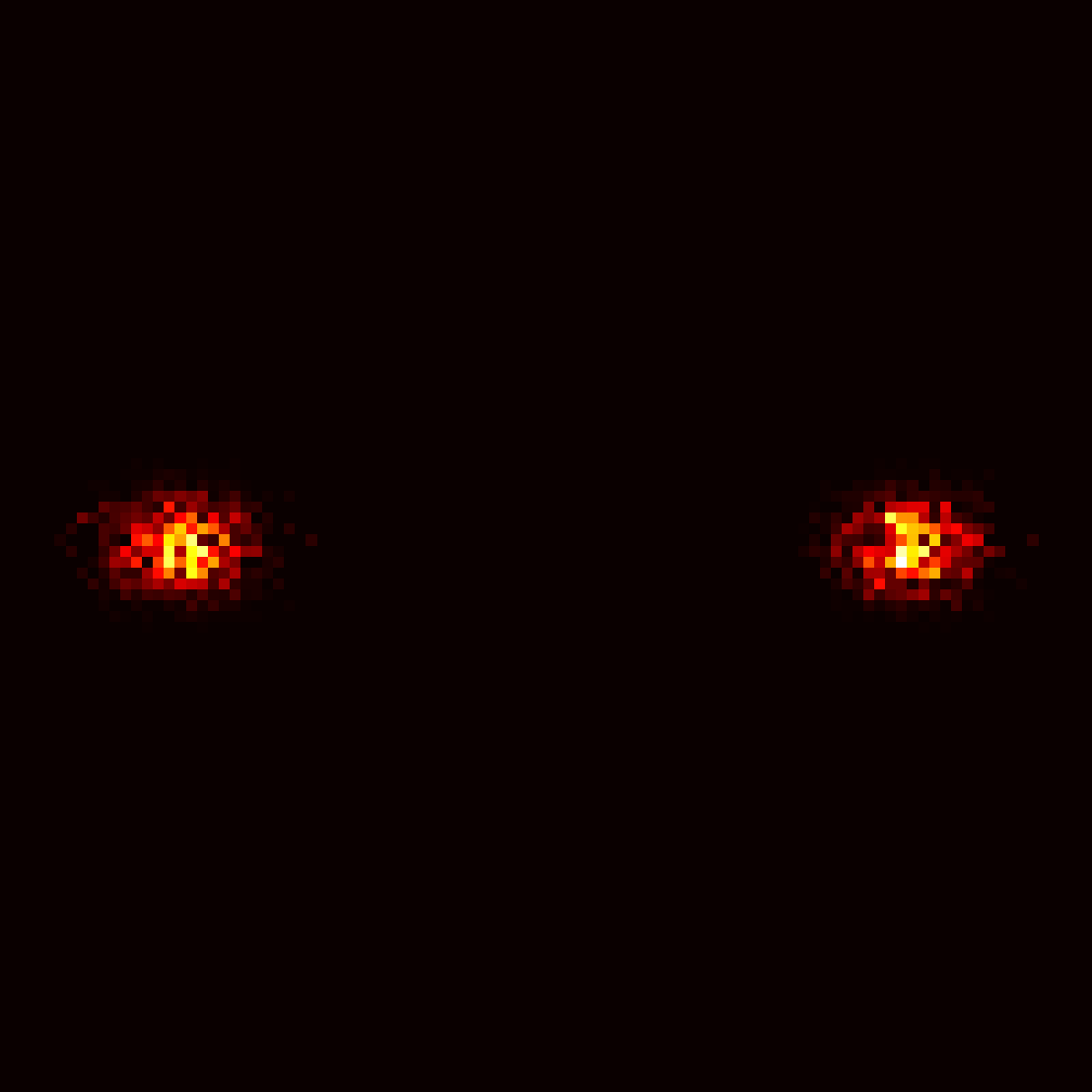} 
        \\
        \includegraphics[width=0.2\textwidth]{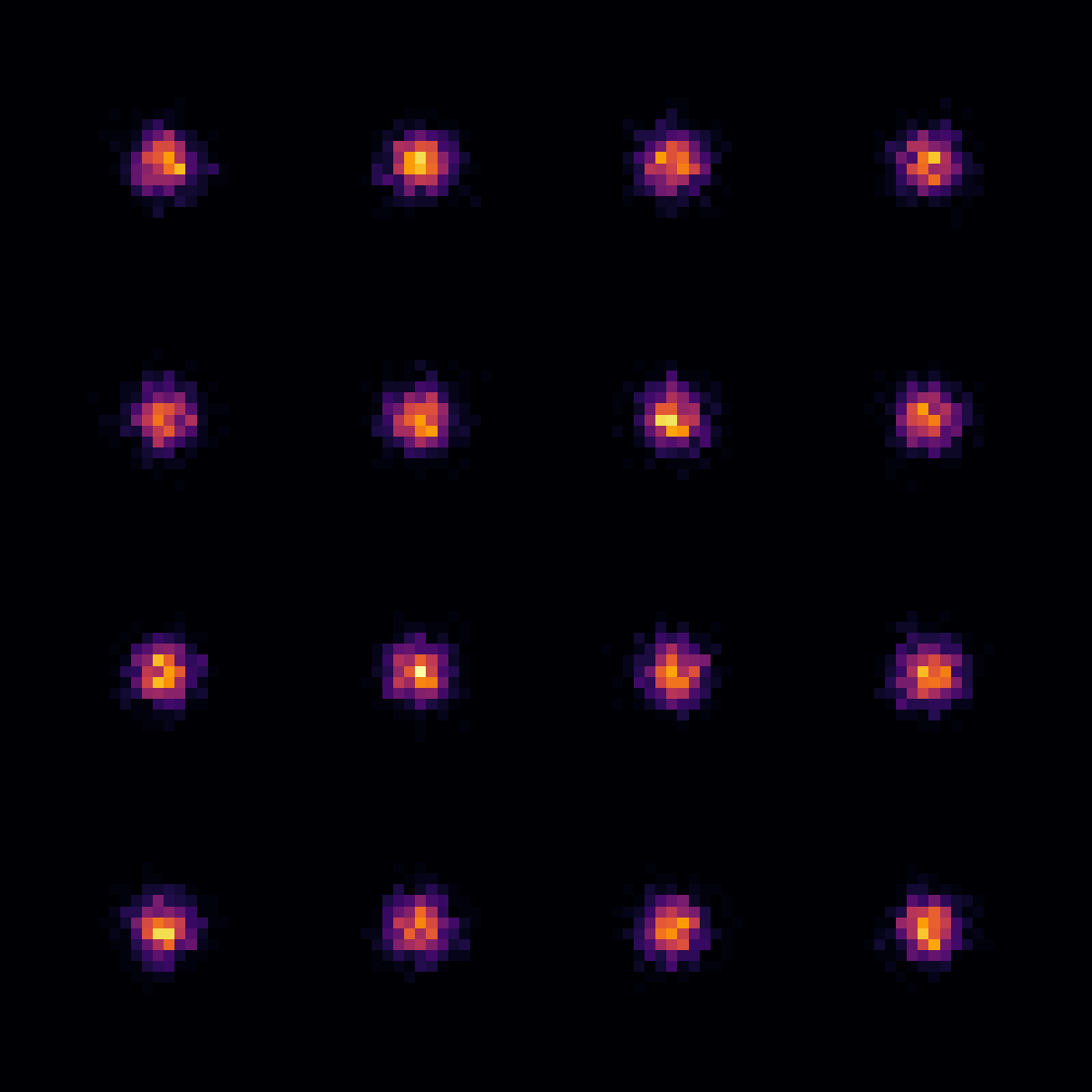} &
        \includegraphics[width=0.2\textwidth]{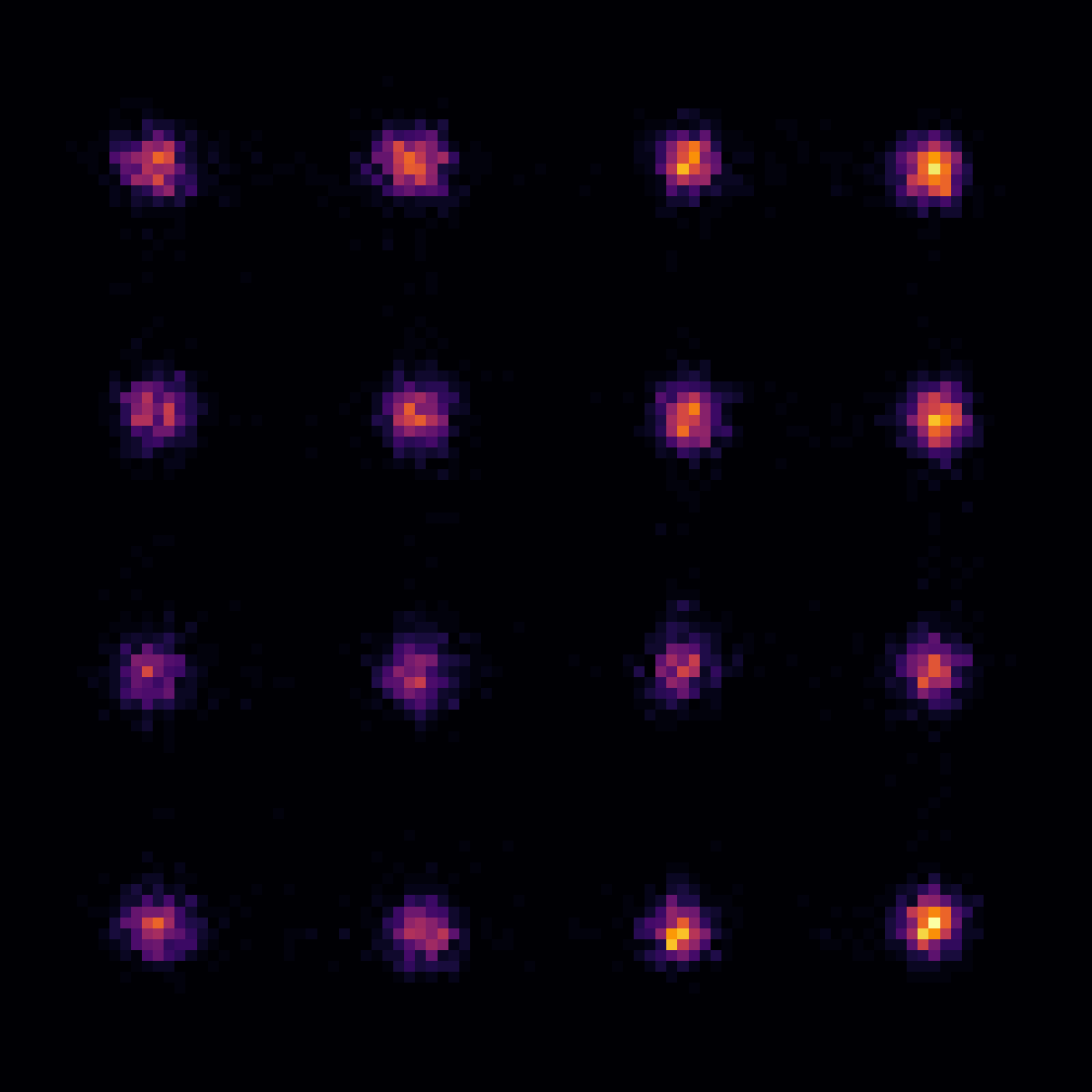} &
        \includegraphics[width=0.2\textwidth]{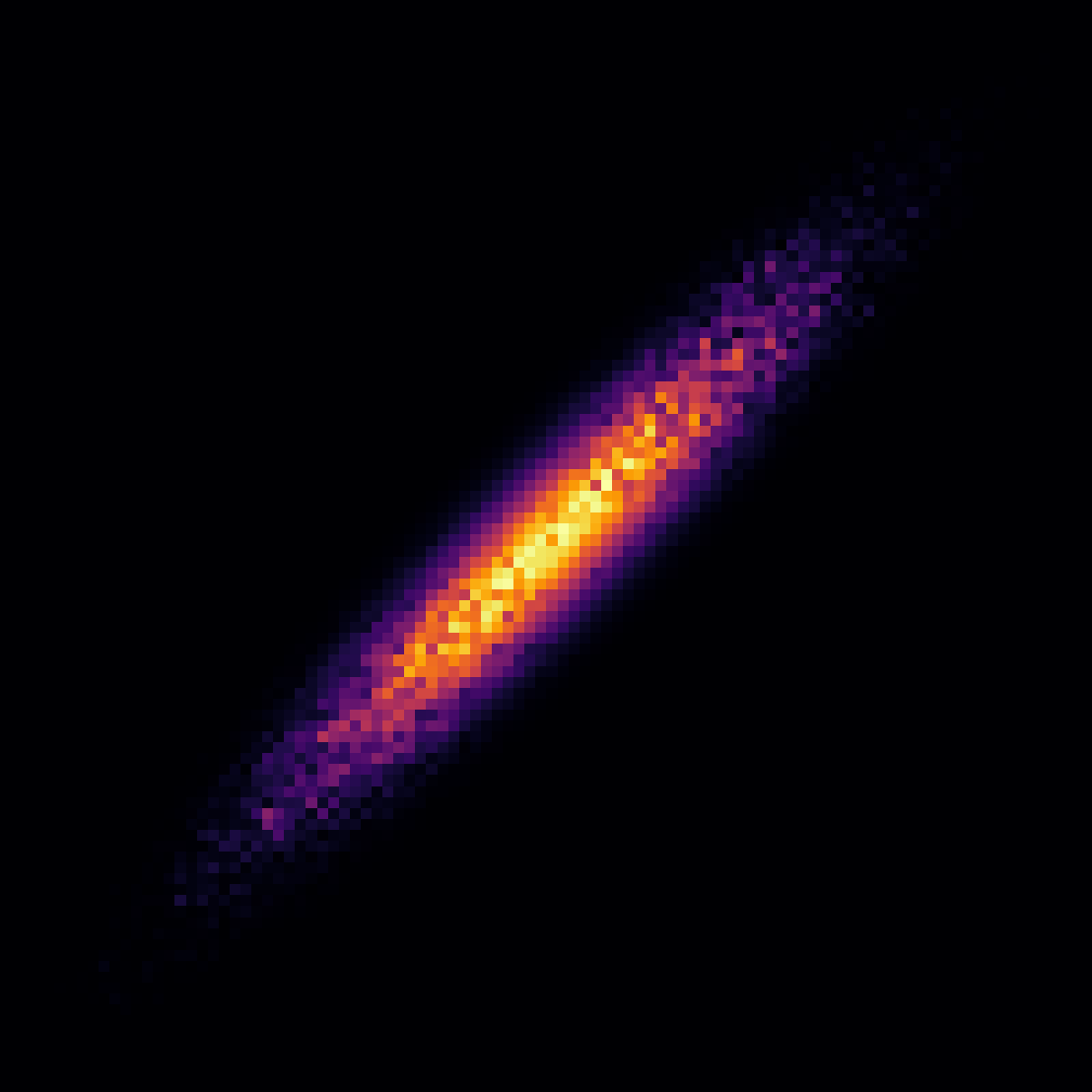} &
        \includegraphics[width=0.2\textwidth]{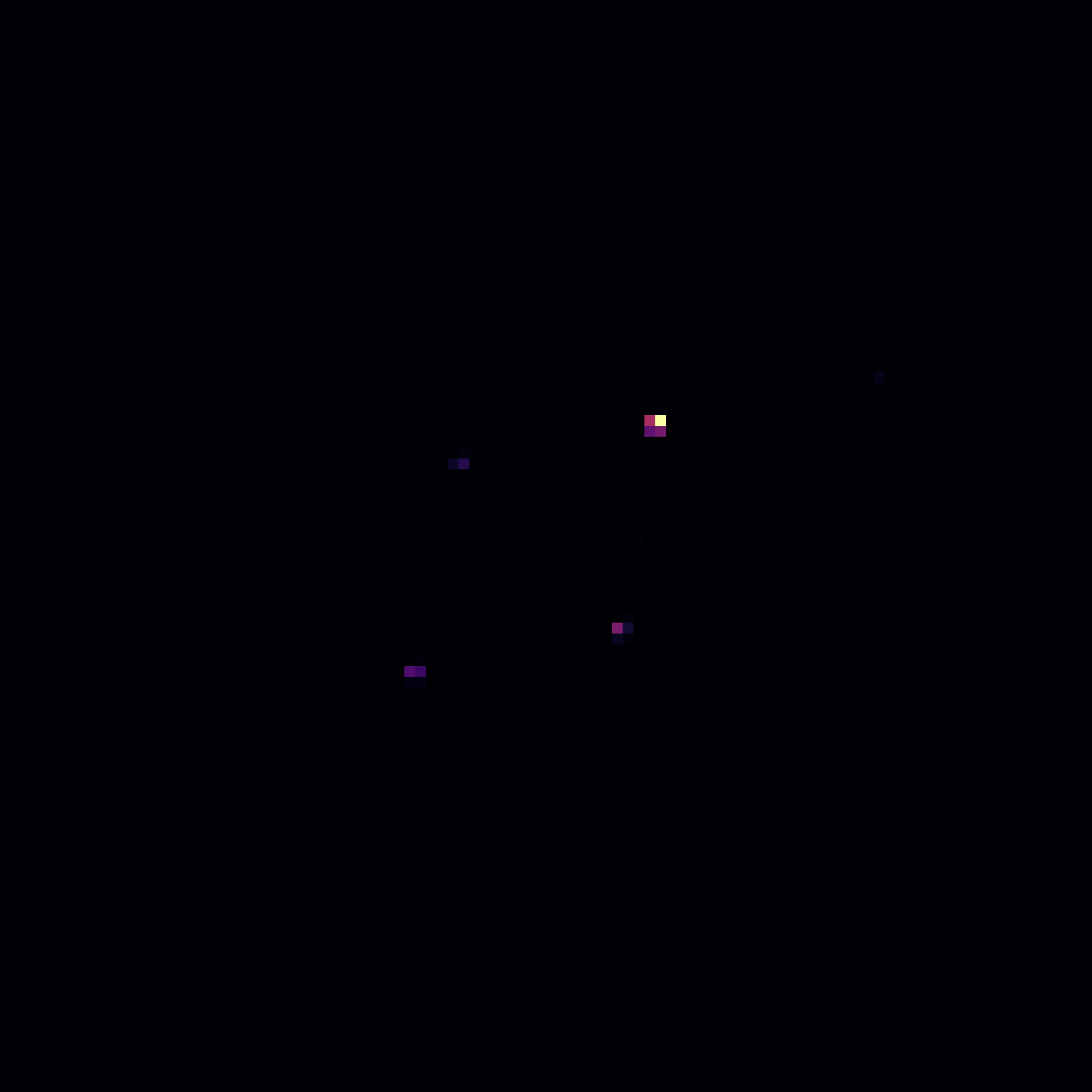} &
        \includegraphics[width=0.2\textwidth]{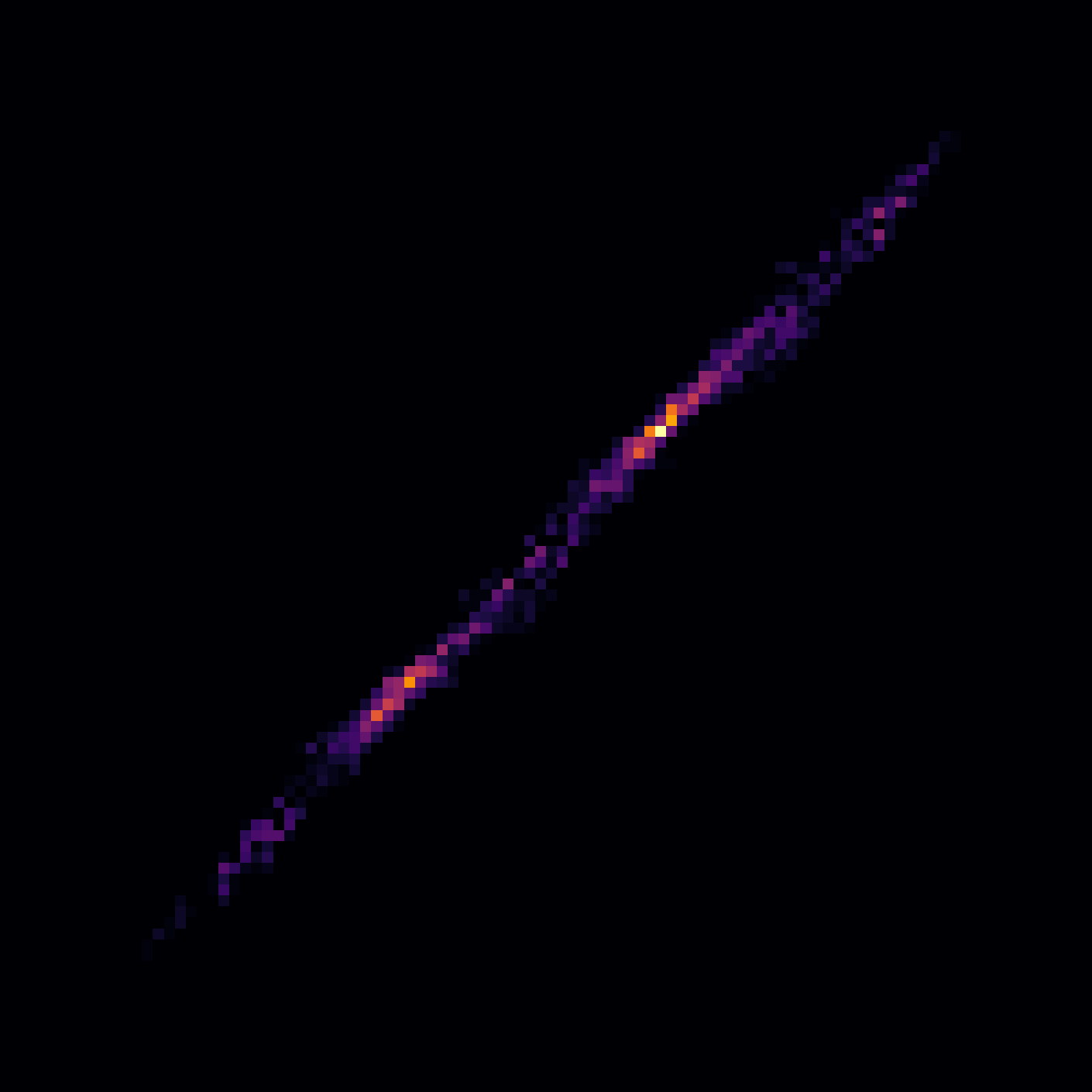} &
        \includegraphics[width=0.2\textwidth]{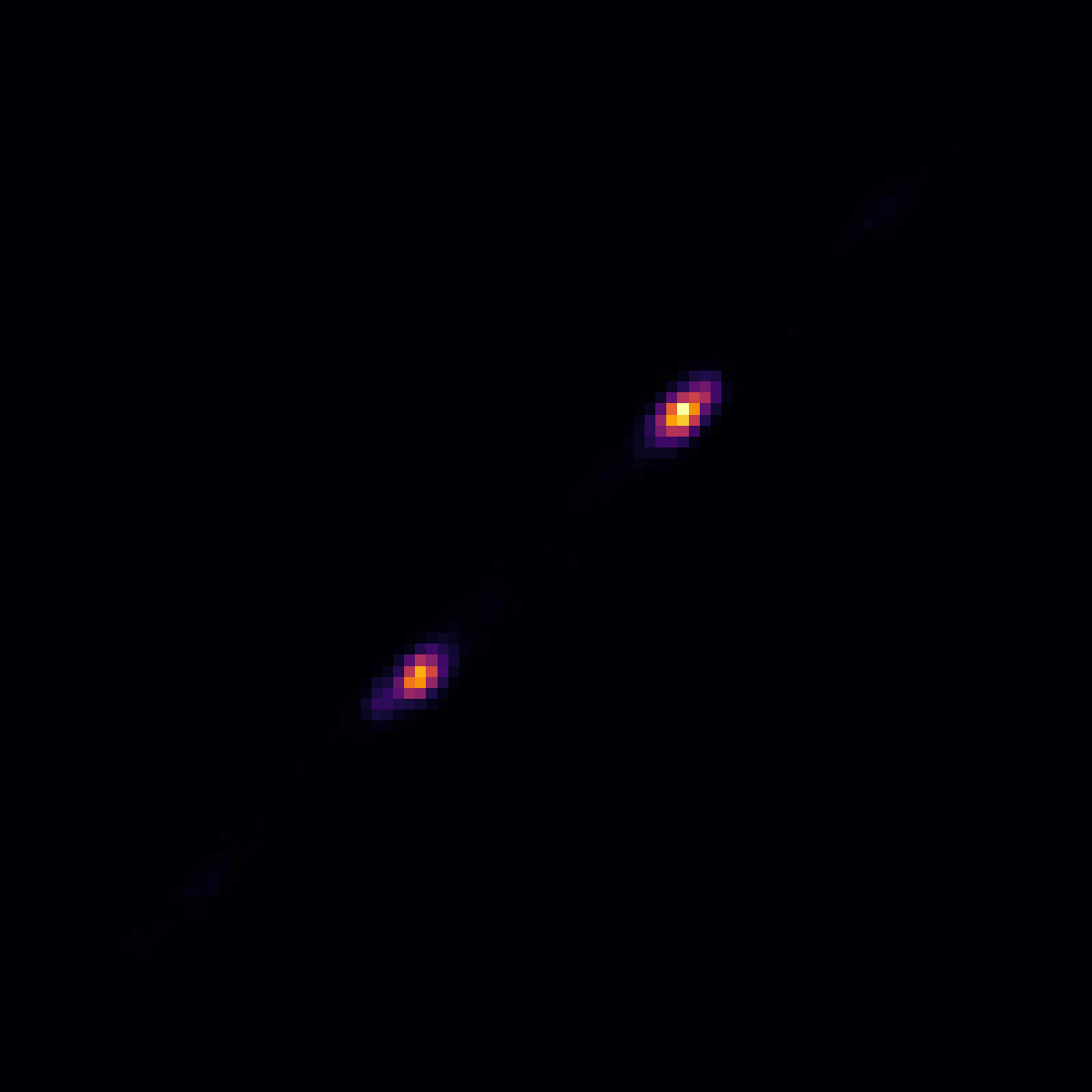} &
        \includegraphics[width=0.2\textwidth]{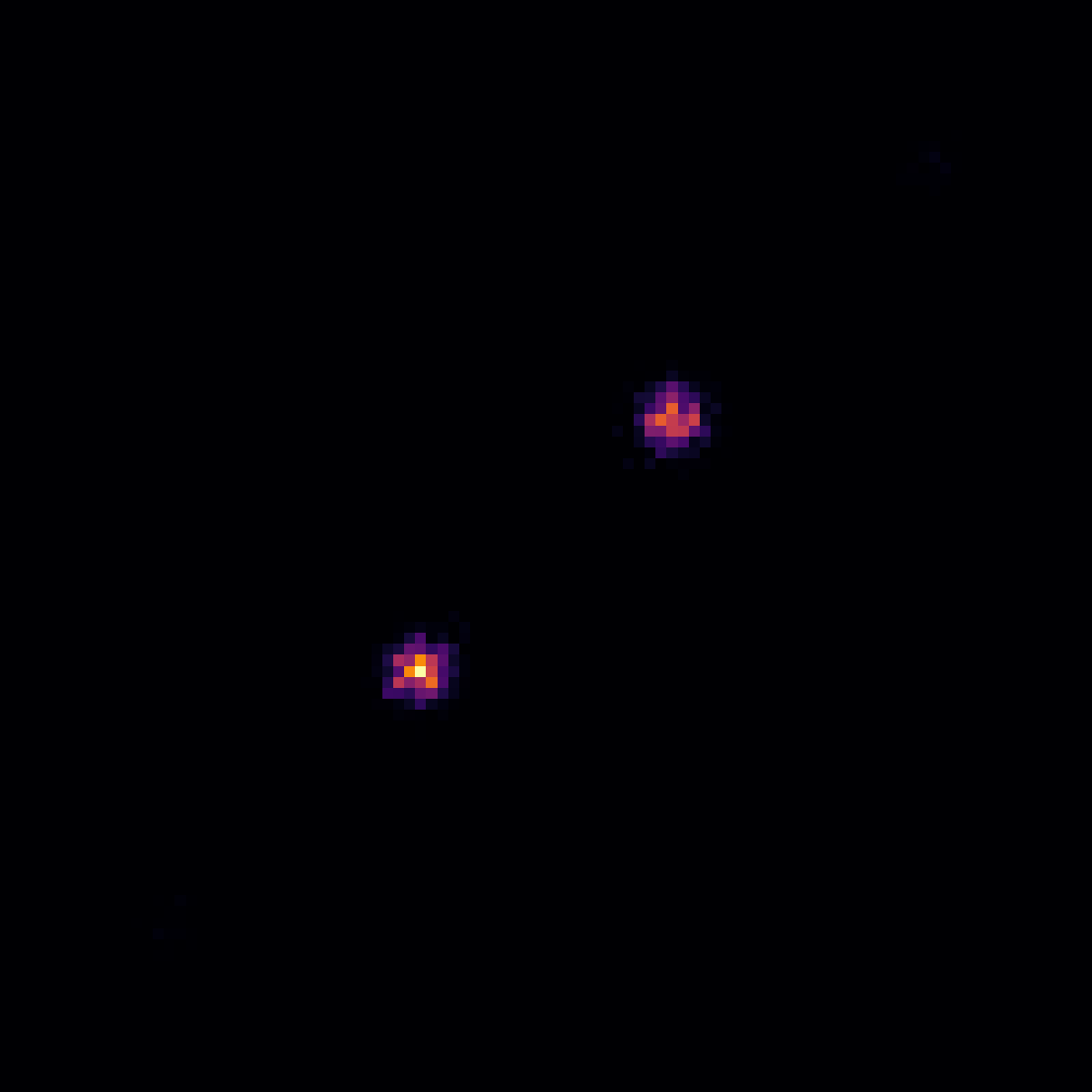} 
    \end{tabular}
    \end{tabular}}
    \caption{\textbf{Comparison of PPM and diffusion‐based VI baselines (RED-Diff and RLSD) on 2D posterior estimation tasks.} Priors are multimodal Gaussian mixtures, and observations are linear projections of the 2D latent state, producing inherently multimodal ground‐truth posteriors. PPM accurately recovers all modes, whereas RED-Diff suffers from mode collapse, and RLSD exhibits measurement inconsistency (top) or poor prior adherence (bottom).}
    \label{fig:teaser}
    \vspace{-0.1em}
\end{figure*}

The second category encompasses optimization-based \textbf{variational inference (VI)}, such as RED-Diff~\cite{mardani2023variational} and Score Distillation Sampling (SDS)~\cite{poole2022dreamfusion}. 
These methods typically formulate the inverse problem as optimizing a weighted denoising score matching objective. 
However, a critical theoretical limitation arises from their implicit assumption of a degenerate variational distribution (e.g., a Dirac delta). 
Such a point-estimate assumption effectively eliminates the entropy term from the variational objective, theoretically degrading the optimization into Maximum A Posteriori (MAP) estimation rather than probabilistic inference. 
Consequently, without the entropy term to encourage diversity, the learned posterior suffers from severe mode-seeking behavior~\cite{luo2024one} and structural collapse, leading to an underestimation of uncertainty and a failure to capture the complex, multi-modal structure of the posterior.

The third paradigm focuses on \textbf{amortized inference (AI)}, represented by methods like DAVI~\cite{lee2024diffusion}. 
These approaches aim to train a feed-forward network to predict the posterior directly for rapid inference.

Though practically useful, DAVI actually deviates from the standard unsupervised inverse problem framework -- solving the Bayesian posterior problem -- in two critical aspects. First, DAVI typically rely on \textbf{paired training data}, introducing supervised objectives—including pixel-wise reconstruction loss against ground truth and adversarial (GAN) losses—to align the reconstruction network for reconstruction. 
This reliance on ground truth supervision restricts their applicability in scientific imaging, where obtaining paired data is often impractical.
Second, regarding the optimization objective, these methods fail to minimize the exact KL divergence. 
Instead, they adopt the \textbf{Integral KL (IKL)} divergence proposed in Diff-Instruct~\cite{luo2023diff}, which serves as a heuristic approximation rather than a rigorous variational lower bound. This deviation from exact KL minimization, as well as the dependence on supervision, shifts the training paradigm from sampling from the exact Bayesian posterior distribution to some approximated distribution with a significant bias, which fundamentally limits their ability to provide the reliable, zero-shot generalization required for solving inverse problems with generic priors.


To overcome the limitations imposed by the inexact approximations of the KL divergence in existing works~\cite{mardani2023variational, zilberstein2024repulsive}, we propose \textbf{Reconstruction Score Matching (PPM)}. 
Instead of relying on the delta distribution assumption or IKL objectives, PPM returns to the principled minimization of the exact KL divergence. 
We achieve this via a classical theoretical result: the Kullback-Leibler divergence can be expanded as the integration of Fisher divergence along the diffusion process. Which leads us to optimize the integral of Fisher divergence rather than the integral of KL divergence in order to sample from the exact Bayesian posterior distribution.
Such a problem formulation provides a rigorous theoretical grounding: we prove that minimizing the exact KL via Fisher integration inherently encourages mass-covering behavior, effectively mitigating mode collapse. 
Furthermore, we demonstrate that this Fisher-based perspective can be naturally generalized to a broader family of divergences via generalized score matching, theoretically unifying existing score-based methods as special cases within our framework. 
Table~\ref{tab:bias_comparison} provides a systematic theoretical comparison, summarizing the specific sources of bias in these baselines—ranging from likelihood approximation in DPS to the independence assumption in IKL—and contrasting them with the unbiased nature of PPM. The detailed theoretical analysis of baseline bias is provided in Section \ref{sec:TheoreticalAnalysis}. 

To make the Fisher divergence optimization tractable, we derive a novel equivalent gradient formula from the Fisher divergence integral. Unlike prior methods that rely on biased estimators, our derivation yields a precise gradient with respect to the variational parameters, enabling efficient optimization via stochastic gradient descent.
In PPM, we model the variational posterior $p(\boldsymbol{x}|\boldsymbol{y})$ flexibly to support two paradigms: either as an amortized reconstruction network $\{\boldsymbol{x} = g_\varphi(\boldsymbol{z}), \boldsymbol{z} \sim \mathcal{N}(0, 1)\}$ for rapid, single-step inference, or as an ensemble of image particles $\{\boldsymbol{\mu}_k\}_{k=1}^K$ for high-fidelity particle-based VI. To accurately compute the divergence, we introduce an auxiliary score network $s_\phi$, adapted from the pre-trained diffusion score $s_p$ via Low-Rank Adaptation (LoRA)~\cite{ryu2023low, hu2022lora}, which bridges the gap between the variational score and the true posterior score.

\begin{table*}[h]
\centering
\caption{Theoretical comparison of optimization objectives across different frameworks.}
\label{tab:bias_comparison}
\resizebox{\textwidth}{!}{%
\begin{tabular}{l|c|c|c|c}
\toprule
\textbf{Method} & \textbf{Objective / Mechanism} & \textbf{Bias Source} & \textbf{Posterior Type} & \textbf{Uncertainty} \\ \midrule
\textbf{DPS}~\cite{chung2022diffusion} & Likelihood Approx. & Laplacian Approximation & Approx. Sampler & Unreliable \\
\textbf{RED-Diff}~\cite{mardani2023variational} & Weighted Score Matching & No Entropy (Dirac) & Point Estimate (MAP) & None \\
\textbf{RLSD}~\cite{zilberstein2024repulsive} & SDS + Repulsion & Heuristic Repulsion (Proxy Entropy) & Particles + Forced Diversity & Artificial \\
\textbf{DAVI (IKL)}~\cite{lee2024diffusion} & $\int D_{\text{KL}}(q_t \| p_t) dt$ & Ignores Temporal Dependency & Amortized Approx. & Biased \\ \midrule
\rowcolor{academicblue}
\textbf{PPM (Ours)} & $\int \text{Fisher}(q_t \| p_t) dt$ & Unbiased (Exact KL) & True Variational & Exact \& Calibrated \\ \bottomrule
\end{tabular}%
}
\end{table*}

We evaluate PPM on a comprehensive suite of inverse imaging benchmarks. To ensure robust and comparable assessments, we first validate our method on standard tasks—including image inpainting, super-resolution, and deblurring—using established datasets such as FFHQ~\cite{karras2019style} and ImageNet~\cite{deng2009imagenet} at $256\times256$ resolution. Beyond natural images, we rigorously test PPM on two challenging scientific imaging applications: super-resolution fluorescence microscopy~\cite{qiao2021evaluation} and radio-interferometric imaging of black holes~\cite{chael2019eht, sun2021deep, sun2022alpha}. Across all experiments, PPM consistently outperforms leading baselines, including DPS~\cite{chung2022diffusion}, KL-based VI methods~\cite{mardani2023variational, zilberstein2024repulsive}, and the amortized approach DAVI~\cite{lee2024diffusion}. Crucially, PPM achieves these results without relying on paired training data, delivering superior reconstruction fidelity, markedly greater sample diversity, and more reliable uncertainty quantification.

\vspace*{-0.2cm}
\section{Background}
\label{sec:bg}


\subsection{Computational Imaging Inverse Problems}
In general, the inverse problem aims to reconstruct underlying image signals $\boldsymbol{x} \in \mathbb{R}^d$ from corrupted observations $\boldsymbol{y} \in \mathbb{R}^m$, where the image formation process is probabilistically modeled as:
\vspace{-0.05in}
\begin{equation} 
\boldsymbol{y} \sim p(\boldsymbol{y}|\boldsymbol{x}). 
\label{eq:forward}
\end{equation} 
Since the observation are usually under-determined ($m \leq d$) and observation noise is inevitable, inverse problems in computational imaging are typically ill-posed, with the inverse mapping $\boldsymbol{y} \rightarrow \boldsymbol{x}$ being one-to-many. To address this complexity, Bayesian inference framework introduces a prior distribution of underlying images, $p(\boldsymbol{x})$, to constrain the solution space for the image posterior, $p(\boldsymbol{x}|\boldsymbol{y})$, as illustrated by:
\begin{equation}
\label{eq:bayesian_inference}
    p(\boldsymbol{x}|\boldsymbol{y}) = \frac{p(\boldsymbol{y}|\boldsymbol{x})p(\boldsymbol{x})}{p(\boldsymbol{y})} \propto \underbrace{p(\boldsymbol{y}|\boldsymbol{x})}_{\text{Likelihood}} \underbrace{p(\boldsymbol{x})}_{\text{Prior}}
\end{equation}
Employing Maximum a Posteriori (MAP) estimation, one can derive a point estimate of the underlying image by maximizing $\log p(\boldsymbol{x}|\boldsymbol{y})$. Alternatively, posterior image samples of reconstructed images can be obtained through methods like Markov Chain Monte Carlo (MCMC)~\cite{brooks2011handbook} or Variational Inference (VI)~\cite{blei2017variational, sun2021deep, sun2022alpha}. However, the performance of many computational imaging solvers is limited by their reliance on oversimplified, handcrafted priors such as sparsity~\cite{candes2007sparsity} and total variation (TV)~\cite{vogel1996iterative}. These priors fail to capture the true complexity of natural image distributions, hindering the solvers' ability to achieve high-quality reconstructions.

\subsection{Diffusion Models}
Diffusion models \cite{ho2020denoising, song2020score, sohl2015deep} formulate generation as the reverse of a continuous‐time diffusion process defined by a stochastic differential equation (SDE). The forward SDE gradually corrupts data by injecting noise:
\begin{equation}
\label{eq:forward_sde}
\mathrm{d}\boldsymbol{x}_t = {f}(\boldsymbol{x}_t, t)\mathrm{d}t + g(t)\mathrm{d}\boldsymbol{w},
\end{equation}
where \( t \in [0, T] \) indexes the diffusion time, \( {f}(\cdot,t): \mathbb{R}^d \to \mathbb{R}^d \) controls the drift coefficient, \( g(t) \) scales the Brownian motion \( \boldsymbol{w} \), and \( \boldsymbol{x}_0 \sim p_{\text{data}} \). This process gradually transforms data samples into a tractable Gaussian distribution \( \boldsymbol{x}_T \sim \mathcal{N}(\boldsymbol{0}, \boldsymbol{I}) \).
The generative process then follows the corresponding reverse‐time SDE:
\begin{equation}
\label{eq:reverse_sde}
\mathrm{d}\boldsymbol{x}_t = \left[{f}(\boldsymbol{x}_t, t) - g(t)^2 \nabla_{\boldsymbol{x}_t}\log p_t(\boldsymbol{x}_t)\right]\mathrm{d}t + g(t)\mathrm{d}\overline{\boldsymbol{w}},
\end{equation}
where \( \nabla_{\boldsymbol{x}_t}\log p_t(\boldsymbol{x}_t) \) is the score function estimated by a neural network \( s_\theta(\boldsymbol{x}_t, t) \). Training such a neural network usually involves optimizing a score matching objective~\cite{song2019generative}:
\begin{equation}
\label{eq:score_loss}
\mathcal{L}(\theta) = \mathbb{E}_{t,\boldsymbol{x}_t} \left[ \lambda(t) \left\| s_\theta(\boldsymbol{x}_t, t) - \nabla_{\boldsymbol{x}_t}\log p_{t}(\boldsymbol{x}_t|\boldsymbol{x}_0) \right\|_2^2 \right],
\end{equation}
where \( \lambda(t) \) reweights time steps and \( p_{t} \) is the perturbation kernel of the forward process. Once trained, we can plug \(s_\theta(\boldsymbol{x}_t, t) \) into Eq.~\ref{eq:reverse_sde} and sample images from a random noise following Eq.~\ref{eq:reverse_sde} or variants \cite{lu2022dpm,karras2022elucidating,xue2023sa,li2025self}. Supported by solid theories, the diffusion model has successes in a wide range of applications~\cite{saharia2022photorealistic,chi2025diffusion,janner2022planning,chen2023diffusiondet,zhang2023enhancing,ye2024schedule,deng2024variational,bai2025vision,bai2025dive3d}. In the next section, we will focus on diffusion models for inverse problems.

\subsection{Diffusion Model for Inverse Problems}
Diffusion models \cite{ho2020denoising, song2020score, sohl2015deep}, owing to their strong ability to accurately approximate complex image distributions, have emerged as powerful data‐driven priors for imaging inverse problems. There are two main categories of diffusion‐model–based solvers for imaging inverse problems. The first builds on MCMC sampling~\cite{brooks2011handbook}, using score functions to guide gradient‐based samplers that steer reconstructions toward the learned image prior. The second employs variational inference frameworks with score distillation sampling (SDS) techniques~\cite{poole2022dreamfusion, mardani2023variational, zilberstein2024repulsive}, enforcing similarity between reconstructed images and the diffusion prior by minimizing the KL divergence between them. Below, we provide an overview of each approach.


\paragraph{Monte Carlo sampling methods}
After training on large image datasets, a diffusion model provides the unconditional score $\nabla_{\mathbf{x}_t}\log p_t(\mathbf{x}_t)$. Posterior sampling replaces this with the conditional score $\nabla_{\mathbf{x}_t}\log p_t(\mathbf{x}_t\mid y)$ during reverse diffusion, via Bayes’ rule:
\begin{equation}
\begin{split}
\nabla_{\mathbf{x}_t}\log p_t(\mathbf{x}_t\mid y)
&= \nabla_{\mathbf{x}_t}\log p_t(\mathbf{x}_t) + \nabla_{\mathbf{x}_t}\log p_t(y\mid \mathbf{x}_t)\\
&\approx s_\theta(\mathbf{x}_t,t) + \nabla_{\mathbf{x}_t}\log p_t(y\mid \mathbf{x}_t),
\label{eq:dps1}
\end{split}
\end{equation}
where $s_\theta(\mathbf{x}_t,t)$ is the learned score network. The principal challenge of posterior sampling lies in approximating the time-dependent likelihood term $\nabla_{\mathbf{x}_t}\log p_t(y\mid \mathbf{x}_t)$.

A popular solution—Diffusion Posterior Sampling (DPS) \cite{chung2022diffusion}—approximates
\begin{equation}
p_t(y \mid \mathbf{x}_t)
= \int p(y\mid \mathbf{x}_0)\,p(\mathbf{x}_0\mid \mathbf{x}_t)\,d\mathbf{x}_0
\approx p\bigl(y\mid \hat{\mathbf{x}}_0(\mathbf{x}_t)\bigr),
\label{eq:dps2}
\end{equation}
where $\hat{\mathbf{x}}_0(\mathbf{x}_t)=\mathbb{E}[\mathbf{x}_0\mid \mathbf{x}_t]$. This point‐estimate approach is computationally efficient and yields strong empirical performance. Alternative schemes approximate both $p(y\mid \mathbf{x}_0)$ and $p(\mathbf{x}_0\mid \mathbf{x}_t)$ as Gaussians to better capture uncertainty \cite{zhu2023denoising}. However, most of these approaches assume linear forward models and do not readily extend to nonlinear inverse problems.

Building on Plug-and-Play (PnP) optimization \cite{graikos2022diffusion, zhu2023denoising}, stochastic PnP Monte Carlo algorithms—such as Generative PnP (GPnP) \cite{bouman2023generative} and PnP Monte Carlo (PMC) \cite{sun2023provable}—alternate between data‐consistency and prior‐refinement steps to approximate the full posterior. Structurally, these methods resemble DPS, but by avoiding point‐estimate approximations, they admit theoretical convergence to the true posterior (albeit at a higher cost). Recent advances further improve sampling via Sequential Monte Carlo \cite{wu2023practical, trippe2022diffusion, cardoso2023monte, dou2024diffusion} or variable‐splitting techniques \cite{coeurdoux2024plug, chen2022improved, lee2021structured, wu2024principled, zhang2024improving, xu2024provably, song2023solving, li2024decoupled, cai2025tone, codes2026}, and extend applicability to nonlinear inverse problems.

Despite their strengths, Monte Carlo sampling methods still face key limitations: they may require many iterations (leading to high computational cost and slow convergence in high dimensions), their approximation errors can introduce bias, performance often depends sensitively on hyperparameters, and they lack amortized inference for rapid repeated use.

\paragraph{Variational inference methods via Weighted score matching objective}
Inverse imaging problems are inherently ill-posed, where a single observation $\boldsymbol{y}$ can be consistent with multiple latent ground-truth images $\mathbf{x}_0$. By combining the measurement forward model with a learned diffusion prior via Bayes' rule, one can define the posterior distribution $p(\mathbf{x}_0|\boldsymbol{y}) \propto p(\boldsymbol{y}|\mathbf{x}_0)p(\mathbf{x}_0)$. 
However, directly sampling from this posterior is intractable. Variational Inference (VI) addresses this by approximating the true posterior $p(\mathbf{x}_0|\boldsymbol{y})$ with a tractable variational distribution $q(\mathbf{x}_0|\boldsymbol{y})$. The objective is to minimize the Kullback-Leibler (KL) divergence between this variational approximation and the true posterior:
\begin{equation}
\label{eq:vi_objective}
\begin{split}
    &\min_q D_{\text{KL}}(q(\mathbf{x}_0\vert\mathbf{y})\;\Vert\; p(\mathbf{x}_0\vert\mathbf{y})) \\
    &= \underbrace{-\mathbb{E}_{q(\mathbf{x}_0|\mathbf{y})} \left[ \log p(\mathbf{y}|\mathbf{x}_0) \right]}_{\mathcal{L}_{\text{data}}} 
    + \underbrace{D_{\text{KL}}\!\left( q(\mathbf{x}_0|\mathbf{y}) \;\Vert\; p(\mathbf{x}_0) \right)}_{\mathcal{L}_{\text{prior}}} + \underbrace{\log p(\mathbf{y})}_{\text{const}}.
\end{split}
\end{equation}
The first term, $\mathcal{L}_{\text{data}}$, enforces data fidelity consistent with the forward operator. The core challenge lies in minimizing the second term, $\mathcal{L}_{\text{prior}}$, which aligns the variational distribution with the diffusion prior.

Existing methods, such as Score Distillation Sampling (SDS)~\cite{poole2022dreamfusion} and RED-Diff~\cite{mardani2023variational}, simplify this optimization by implicitly assuming that the variational posterior $q(\mathbf{x}_0|\mathbf{y})$ is a degenerate Dirac delta distribution $q(\mathbf{x}_0|\mathbf{y}) = \delta(\mathbf{x}_0 - \boldsymbol{\mu})$ (or a Gaussian with vanishing variance $\sigma \to 0$) centered at the estimated parameters $\boldsymbol{\mu}$. 
Under this point-estimate assumption, the entropy term of the variational distribution is effectively discarded. Consequently, the minimization of the KL divergence simplifies to a weighted score matching objective for the point estimate $\boldsymbol{\mu}$:
\begin{equation}
\label{eq:red_diff_loss}
    \min_{\boldsymbol{\mu}} \quad \underbrace{\| \boldsymbol{y} - \mathcal{A}(\boldsymbol{\mu}) \|^2}_{\mathcal{L}_{\text{data}}} + \lambda \underbrace{\mathbb{E}_{t, \boldsymbol{\epsilon}} \left[ \omega(t) \left\| \boldsymbol{\epsilon}_\theta(\alpha_t \boldsymbol{\mu} + \sigma_t \boldsymbol{\epsilon}, t) - \boldsymbol{\epsilon} \right\|_2^2 \right]}_{\mathcal{L}_{\text{prior}}},
\end{equation}
where $\mathcal{A}(\cdot)$ represents the forward operator, and $\omega(t)$ is a weighting function (often chosen heuristically based on SNR).
While computationally tractable, this formulation represents a **biased approximation** of the true variational objective. 
By enforcing a degenerate distribution and neglecting the entropy term, the optimization theoretically degrades into Maximum A Posteriori (MAP) estimation. This induces significant optimization bias, manifesting as mode-seeking behavior, where the single estimate $\boldsymbol{\mu}$ fails to capture the necessary diversity and uncertainty of the full solution space.

Several recent methods have adopted Eq.~\ref{eq:red_diff_loss} for VI posterior estimation. 
For instance, ~\cite{feng2023score, feng2023efficient} integrate normalizing flows~\cite{kingma2018glow, dinh2016density} with diffusion models for accurate posterior modeling. However, their performance is constrained to lower-dimensional signals (e.g., $64\times64$) due to inherent limitations in normalizing flow's scalability. 
Recently, RED-Diff~\cite{mardani2023variational} proposes a variational approach that combines the prior loss with a data fidelity term to optimize an estimate of the clean image $\boldsymbol{x}$. 
VSS~\cite{he2024solving} manages to adopt the VI approach to solve zero-shot sparse-view CT reconstruction with a latent diffusion model.
However, these methods have been observed to usually suffer from mode collapse issues. 
To address this, RLSD \cite{zilberstein2024repulsive} adds a repulsive penalty between similar reconstructions. Although this increases sample diversity, its empirical assumptions limit gains in full‐posterior recovery, and mode collapse remains an issue. Our method aims to fundamentally overcomes these challenges by embedding a score‐based divergence distillation loss within a variational inference framework.

\paragraph{Amortized Inference for Inverse Problems via Integral KL Divergence}

Unlike optimization-based methods that solve for a specific instance, amortized inference aims to learn a parametric reconstruction network $\boldsymbol{x}_0 = g_\varphi(\boldsymbol{y})$ that maps observations directly to the posterior samples. 
Recent approaches like DAVI~\cite{lee2024diffusion} adopt the training objective from Diff-Instruct~\cite{luo2023diff}, replacing the standard KL divergence with a heuristic metric known as the IKL divergence. 
IKL modifies the objective by manually introducing a time-weighting function $\omega_t$ and integrating the marginal KL divergences over the entire diffusion process:
\begin{equation}
\label{eq:ikl_objective}
\begin{split}
    &\mathcal{L}_{\text{IKL}}(\varphi) = \int_0^T \omega_t \, D_{\text{KL}}(q_{\varphi,t}(\boldsymbol{x}_t|y) \| p_t(\boldsymbol{x}_t)) \mathrm{d}t \\
    &= \int_0^T \omega_t \, \mathbb{E}_{q_{\varphi,t}(\boldsymbol{x}_t|y)} \left[ \log \frac{q_{\varphi,t}(\boldsymbol{x}_t|y)}{p_t(\boldsymbol{x}_t)} \right] \mathrm{d}t,
\end{split}
\end{equation}
where $q_t(\boldsymbol{x}_t)$ is the distribution of the generated sample diffused to time $t$. 
Assuming the reconstruction network output is deterministic given $\boldsymbol{y}$ (or the implicit distribution is approximated as Gaussian), the marginal distribution $q_{\varphi,t}(\boldsymbol{x}_t)$ becomes a Gaussian centered at $\alpha_t g_\varphi(\boldsymbol{y})$. Consequently, its score $\nabla_{\boldsymbol{x}_t} \log q_{\varphi,t}(\boldsymbol{x}_t|y)$ is analytically computable.
Following the derivation in Diff-Instruct~\cite{luo2023diff}, the gradient of this objective with respect to the reconstruction network parameters $\varphi$ avoids backpropagation through the frozen score network $p_t$, and is given by:
\begin{equation}
\label{eq:ikl_gradient}
\begin{split}
    &\int_0^T \omega_t \, \mathbb{E}_{\boldsymbol{x}_0 = g_\varphi(y), \boldsymbol{\epsilon} \sim \mathcal{N}(0,\mathbf{I}), \boldsymbol{x}_t = \alpha_t \boldsymbol{x}_0 + \sigma_t \boldsymbol{\epsilon} } \\
    &\Biggl[
        \left( \nabla_{\boldsymbol{x}_t} \log q_{\varphi,t}(\boldsymbol{x}_t|y) \right.
        \left. - \nabla_{\boldsymbol{x}_t} \log p_t(\boldsymbol{x}_t) \right)^\top
        \frac{\partial \boldsymbol{x}_t}{\partial\boldsymbol{\varphi}} 
    \Biggr] \mathrm{d}t,
\end{split}
\end{equation}
where $\omega_t$ is the weight of different time step $t$.

While $\mathcal{L}_{\text{IKL}}$ provides a gradient signal for aligning the reconstruction network with the diffusion prior, it is crucial to note that Eq.~\ref{eq:ikl_objective} is not an exact estimation of the true posterior KL divergence $D_{\text{KL}}(q_0(\boldsymbol{x}_0|\boldsymbol{y}) \| p(\boldsymbol{x}_0))$. 
The transformation from the original KL to the time-integrated IKL relies on heuristic weighting $\omega_t$ and ignores the temporal dependencies of the diffusion trajectory. 
This discrepancy implies that minimizing $\mathcal{L}_{\text{IKL}}$ does not guarantee minimization of the actual variational bound, leading to biased posterior estimation and limited sample diversity compared to exact optimization. Besides the IKL divergence minimization, some other works have studied amortized inference in the context of diffusion acceleration ~\cite{luo2024one,luo2024diffpp,zhou2024score,zhou2024adversarial,luo2024diffstar,luodavid,yin2024one,wang2024integrating,wang2025uni,luo2025reward,luo2023comprehensive}.

\vspace*{-0.1cm}
\section{Method}
\label{sec:method}

In this section, we present PPM, a principled framework for posterior recovery in computational imaging inverse problems. 
PPM introduces a novel score-based divergence guides optimization via an unbiased gradient estimator, which demonstrates unparalleled performance in both variational inference and amortized inference.
Unlike the asymmetric mathematical form of KL divergence—whose tendency toward mode collapse limits reliable posterior estimation—our divergence provides a stable, unbiased surrogate objective that extends naturally to amortized settings, where an inference network learns to approximate posterior samples across instances. 
This formulation enables PPM to enhance both VI-based and AI-based methods, outperforming existing approaches including Score Distillation Sampling (SDS)~\cite{poole2022dreamfusion}, RED-Diff~\cite{mardani2023variational}, and Diffusion Prior-Based Amortized Variational Inference (DAVI)~\cite{lee2024diffusion}.

\begin{figure*}[htbp]
  \centering 
  \includegraphics[width=0.95\textwidth]{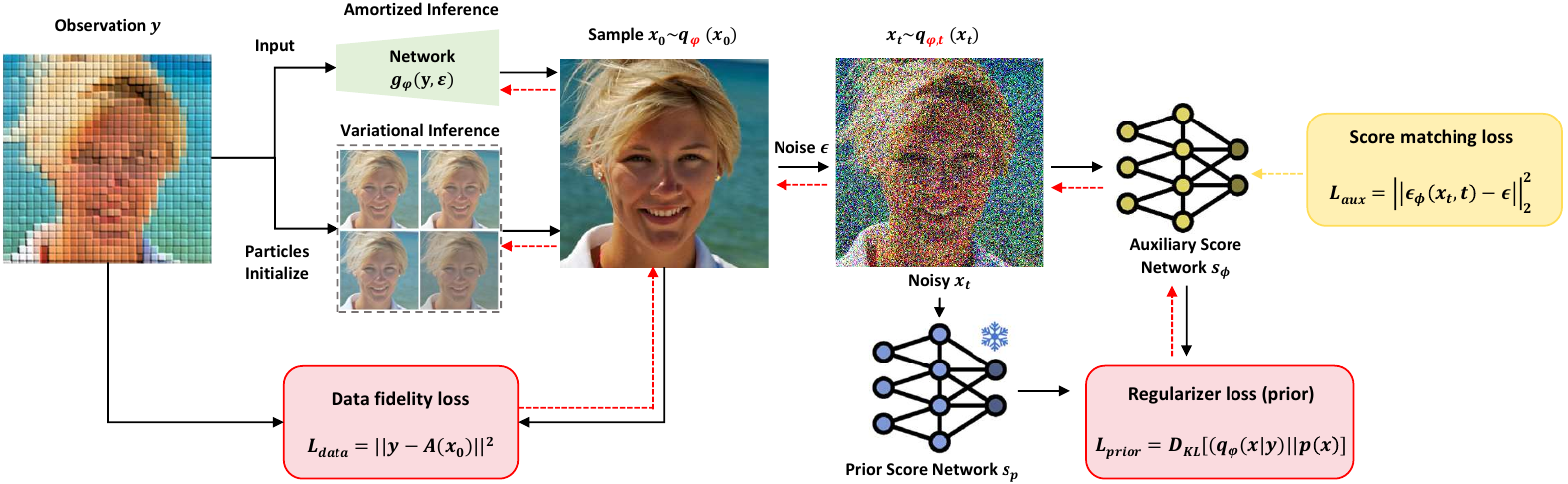} 
  \caption{\textbf{Overview of PPM.} PPM approximates the posterior by optimizing a variational posterior distribution $q_\varphi$, which is parameterized either by a set of particles or a neural network. The optimization minimizes a loss function composed of a data fidelity term and a novel score-based divergence. This divergence is computed between a pre-trained prior score model $\boldsymbol{s}_\theta$ and an auxiliary score network $\boldsymbol{s}_\phi$, which approximates the score of $q_\varphi$. Computing this divergence requires an auxiliary score network $\boldsymbol{s}_\phi$. The parameters of $q_\varphi$ and $\boldsymbol{s}_\phi$ are optimized alternatively.}
  \label{fig:pipeline} 
\end{figure*}

\subsection{Problem Formulation}
Following Bayes’ rule in Eq.~\ref{eq:vi_objective}, the optimization problem of classical VI-based inverse imaging solvers can be formulated as (neglecting the constant $\log p(\boldsymbol{y})$):
\begin{equation}
\label{equ:vi-inverse}
\begin{split}
    \varphi^*:
    &=\arg \min_{\varphi} D_{\text{KL}}(q_\varphi(\boldsymbol{x}|\boldsymbol{y})\vert\vert p(\boldsymbol{x}\vert\boldsymbol{y})) \\
    &=\arg \min_{\varphi} \frac{1}{2\sigma^2}||\boldsymbol{y}-\mathcal{A}(\boldsymbol{x})||^2 + D_{\text{KL}}\left( q_\varphi(\boldsymbol{x}|\boldsymbol{y}) \,\|\, p(\boldsymbol{x}) \right),
\end{split}
\end{equation}
where the first term enforces data fidelity consistent with observations $\boldsymbol{y}$ via the forward model $\mathcal{A}(\cdot)$. The second term encourages the variational posterior distribution $q_\varphi(\boldsymbol{x}|\boldsymbol{y})$ to align with the prior distribution $p(\boldsymbol{x})$ implicitly learned by a pre-trained diffusion model.

However, directly optimizing Eq.~\ref{equ:vi-inverse} with approximate objectives often induces mode collapse.
To address this, we replace the standard KL term with a formulation based on the integration of Fisher divergence, which allows for exact and unbiased optimization ~\cite{song2021maximum}.

\subsection{Exact Optimization via Reconstruction Score Matching}
We reformulate the KL divergence in Eq.~\ref{equ:vi-inverse} using the integral of the Fisher divergence. This leads to the following exact optimization objective:
\begin{equation}
\label{equ:fisher-div}
\begin{split}
    &D_{\text{KL}}(q_\varphi(\boldsymbol{x}|\boldsymbol{y}) \| p(\boldsymbol{x})) \\
    &= \frac{1}{2}\int_0^T g(t)^2 \, \mathbb{E}_{\substack{ 
        \boldsymbol{x}_0 \sim q_{\varphi}(\boldsymbol{\cdot|y}) \\
        \boldsymbol{x}_t | \boldsymbol{x}_0 \sim p(\boldsymbol{x}_t | \boldsymbol{x}_0)
        }} \left[ d(s_{q_{\varphi,t}}(x_t|y) - s_{p_t}(x_t)) \right] \, \mathrm{d}t \\
    &\text{with} \quad s_{q_{\varphi,t}}(x_t|y) = \nabla_{\boldsymbol{x}_t} \log q_{\varphi,t}(\boldsymbol{x}_t|\boldsymbol{y}), \\
    &s_{p_t}(x_t) = \nabla_{\boldsymbol{x}_t} \log p_t(\boldsymbol{x}_t),\\
    &d = ||\cdot||^2_2,
\end{split}
\end{equation}
where $q_{\varphi}(\boldsymbol{y})$ denotes sampling from the variational posterior (e.g., $\boldsymbol{x}_0 = \mu(\boldsymbol{y})$), $s_{p_t}(\boldsymbol{x}_t)$ and $s_{q_{\varphi,t}}(\boldsymbol{x}_t)$ are the scores of the prior and posterior distribution respectively.

Optimizing this objective requires differentiating through the score of the variational distribution, which depends on $\varphi$. To make this tractable, we derive the following gradient equivalence theorem.

\begin{theorem}[Gradient Equivalence Theorem]
\label{THE:PPM}
   If distribution \( q_\varphi(\boldsymbol{x}|\boldsymbol{y}) \) satisfies mild regularity conditions, for any score function \( s_{p_t}(\cdot) \), the following gradient equivalence holds:
    \begin{equation}
    \begin{split}
        &\mathbb{E}_{\substack{ 
        \boldsymbol{x}_0 \sim q_{\varphi}(\boldsymbol{\cdot|y}) \\
        \boldsymbol{x}_t | \boldsymbol{x}_0 \sim p(\boldsymbol{x}_t | \boldsymbol{x}_0)
        }} \frac{\partial}{\partial \varphi} \left[  d\left( s_{q_{\varphi,t}}(\boldsymbol{x}_t |\boldsymbol{y}) - s_{p_t}(\boldsymbol{x}_t) \right) \right]\\
        &= \mathbb{E} \bigg[ d'
        \left( s_{q_{\varphi,t}}(\boldsymbol{x}_t|\boldsymbol{y}) - s_{p_t}(\boldsymbol{x}_t) \right) \frac{\partial}{\partial \varphi} s_{q_{\varphi,t}}(\boldsymbol{x}_t|\boldsymbol{y}) \\
        &\quad +d'\left( s_{q_{\varphi,t}}(\boldsymbol{x}_t|\boldsymbol{y}) - s_{p_t}(\boldsymbol{x}_t) \right)\frac{\partial \boldsymbol{x}_t}{\partial \varphi}  \bigg] \\
        &= \frac{\partial}{\partial \varphi} \mathbb{E} \Bigg[ -\left\{ d' \left( s_{q_{\text{sg}[\varphi],t}}(\boldsymbol{x}_t|\boldsymbol{y}) - s_{p_t}(\boldsymbol{x}_t) \right) \right\}^T \\
        &\quad \cdot \left\{ s_{q_{\text{sg}[\varphi],t}}(\boldsymbol{x}_t|\boldsymbol{y}) - \nabla_{\boldsymbol{x}_t} \log p_t(\boldsymbol{x}_t | \boldsymbol{x}_0) \right\} \\
        &\quad + d \left( s_{q_{\text{sg}[\varphi],t}}(\boldsymbol{x}_t|\boldsymbol{y}) - s_{p_t}(\boldsymbol{x}_t) \right)\Bigg],
    \end{split}
    \end{equation}
    where $sg$ denotes the stop gradient operator. 
\end{theorem}

\begin{proof}
The proof is based on the Score-projection identity, which bridges denoising score matching and denoising auto-encoders. Let $\boldsymbol{u}(\cdot,\varphi)$ be a vector-valued function. Using the notations of Theorem \ref{THE:PPM}, under mild conditions, the following identity holds:
\begin{equation}
\label{eq:score_identity}
\begin{split}
    &\mathbb{E}_{\substack{\boldsymbol{x}_0 \sim q_{\varphi}(\cdot|\boldsymbol{y}) \\ \boldsymbol{x}_t|\boldsymbol{x}_0 \sim p_t(\boldsymbol{x}_t|\boldsymbol{x}_0)}} \bigg[ \boldsymbol{u}(\boldsymbol{x}_t,\varphi)^T \Big\{ s_{q_{\varphi,t}}(\boldsymbol{x}_t|\boldsymbol{y}) \\
    &\quad - \nabla_{\boldsymbol{x}_t} \log p_t(\boldsymbol{x}_t|\boldsymbol{x}_0) \Big\} \bigg] = 0.
\end{split}
\end{equation}
We start by applying the chain rule for the total derivative with respect to $\varphi$. The function $d(\cdot)$ depends on $\varphi$ both directly through the score function $s_{q_{\varphi,t}}$ and indirectly through the distribution $\boldsymbol{x}_t \sim q_{\varphi,t}$ (as $\boldsymbol{x}_t$ depends on $\boldsymbol{x}_0 \sim q_{\varphi}(\cdot|\boldsymbol{y})$). This gives two terms:
\begin{equation}
\begin{split}
&\mathbb{E}_{\boldsymbol{x}_t \sim q_{\varphi,t}} \frac{\partial}{\partial \varphi} d(s_{q_{\varphi,t}}(\boldsymbol{x}_t|\boldsymbol{y}) - s_{p_t}(\boldsymbol{x}_t)) \\
&= \mathbb{E} \bigg[ d'(s_{q_{\varphi,t}}(\boldsymbol{x}_t|\boldsymbol{y}) - s_{p_t}(\boldsymbol{x}_t))^T \frac{\partial}{\partial \varphi} s_{q_{\varphi,t}}(\boldsymbol{x}_t|\boldsymbol{y}) \\
&\quad + \frac{\partial}{\partial \boldsymbol{x}_t} d(s_{q_{\varphi,t}}(\boldsymbol{x}_t|\boldsymbol{y}) - s_{p_t}(\boldsymbol{x}_t))^T \frac{\partial \boldsymbol{x}_t}{\partial \varphi} \bigg].
\end{split}
\end{equation}
To resolve the first term, we differentiate Eq. \eqref{eq:score_identity} with respect to $\varphi$. Since the expectation is zero for all $\varphi$, its derivative is also zero. We apply the total derivative:
\begin{equation}
\begin{split}
0 &= \frac{\partial}{\partial \varphi} \mathbb{E} \left[ \boldsymbol{u}(\boldsymbol{x}_t,\varphi)^T \{s_{q_{\varphi,t}}(\boldsymbol{x}_t|\boldsymbol{y}) - \nabla_{\boldsymbol{x}_t} \log p_t(\boldsymbol{x}_t|\boldsymbol{x}_0)\} \right] \\
&= \mathbb{E} \bigg[ \frac{\partial}{\partial \varphi} \{ \boldsymbol{u}(\boldsymbol{x}_t,\varphi)^T (s_{q_{\text{sg}[\varphi],t}}(\boldsymbol{x}_t|\boldsymbol{y}) - \nabla_{\boldsymbol{x}_t} \log p_t(\boldsymbol{x}_t|\boldsymbol{x}_0)) \} \\
&\quad + \boldsymbol{u}(\boldsymbol{x}_t,\varphi)^T \frac{\partial}{\partial \varphi} \{s_{q_{\varphi,t}}(\boldsymbol{x}_t|\boldsymbol{y})\} \bigg].
\end{split}
\end{equation}
Rearranging the terms yields:
\begin{equation}
\label{eq:rearranged_identity}
\begin{split}
&\mathbb{E} \left[ \boldsymbol{u}(\boldsymbol{x}_t,\varphi)^T \frac{\partial}{\partial \varphi} \{s_{q_{\varphi,t}}(\boldsymbol{x}_t|\boldsymbol{y})\} \right] \\
&= - \frac{\partial}{\partial \varphi} \mathbb{E} \left[ \boldsymbol{u}(\boldsymbol{x}_t,\varphi)^T \{s_{q_{\text{sg}[\varphi],t}}(\boldsymbol{x}_t|\boldsymbol{y}) - \nabla_{\boldsymbol{x}_t} \log p_t(\boldsymbol{x}_t|\boldsymbol{x}_0)\} \right].
\end{split}
\end{equation}
Let $\boldsymbol{u}(\boldsymbol{x}_t,\varphi) = d'(s_{q_{\text{sg}[\varphi],t}}(\boldsymbol{x}_t|\boldsymbol{y}) - s_{p_t}(\boldsymbol{x}_t))$. Substituting this specific function $\boldsymbol{u}$ into Eq. \eqref{eq:rearranged_identity} allows us to replace the first term in our objective expansion. Furthermore, $\varphi$ does not appear in the differentiation with respect to $\boldsymbol{x}_t$ for the second term:
\begin{equation}
\begin{split}
&\mathbb{E} \left[ \frac{\partial}{\partial \boldsymbol{x}_t} d(s_{q_{\varphi,t}}(\boldsymbol{x}_t|\boldsymbol{y}) - s_{p_t}(\boldsymbol{x}_t))^T \frac{\partial \boldsymbol{x}_t}{\partial \varphi} \right] \\
&= \frac{\partial}{\partial \varphi} \mathbb{E} \left[ d(s_{q_{\text{sg}[\varphi],t}}(\boldsymbol{x}_t|\boldsymbol{y}) - s_{p_t}(\boldsymbol{x}_t)) \right].
\end{split}
\end{equation}
Combining these results yields exactly the Gradient Equivalence Theorem.
\end{proof}

In Theorem~\ref{THE:PPM}, the variational score $\boldsymbol{s}_{q_{\varphi,t}}(\boldsymbol{x}_t|\boldsymbol{y})$ is estimated by an auxiliary neural network $\boldsymbol{s}_\phi(\boldsymbol{x}_t, t)$. This network is trained on the current reconstructions $\boldsymbol{x} \sim q_\varphi(\cdot|\boldsymbol{y})$ using a standard denoising score matching objective. We refer to $\boldsymbol{s}_\phi$ as the \emph{auxiliary model}, and its training loss is:
\begin{equation}
\label{eq:aux_loss}
\begin{split}
    &\mathcal{L}_{aux}(\phi) \\
    &= \int_0^T \lambda(t)\mathbb{E}_{\substack{ 
        \boldsymbol{x}_0 \sim q_{\varphi}(\boldsymbol{\cdot|y}) \\
        \boldsymbol{x}_t | \boldsymbol{x}_0 \sim p_t(\boldsymbol{x}_t | \boldsymbol{x}_0)
        }} \left\| s_\phi(\mathbf{x}_t, t) - \nabla_{\mathbf{x}_t}\log p_{t}(\mathbf{x}_t|\mathbf{x}_0) \right\|_2^2 \mathrm{d}t.
\end{split}
\end{equation}
We implement $\boldsymbol{s}_\phi$ as a copy of the pre-trained model $\boldsymbol{s}_{p}$. This preserves the prior information in $\boldsymbol{s}_{p}$ while adapting to the conditional distribution $q_{\varphi,t}(\boldsymbol{x}_t|\boldsymbol{y})$ during optimization.

\begin{algorithm*}[t]
\caption{Principled Posterior Matching (PPM) for Inverse Problems}
\label{alg:ppm_unified}
\begin{algorithmic}[1]
    \Require Pretrained diffusion network $\boldsymbol{s}_{\theta}$, Auxiliary network $\boldsymbol{s}_\phi$ (init via a copy of $\boldsymbol{s}_{p}$), Noising schedule $\{\alpha_t, \sigma_t\}$, Any distance function $d(\cdot)$, Noise scale $\boldsymbol{h}$, Data fidelity weight $\lambda$, Time steps weight $w(t)$, Total time steps $T$.
    \Require \textbf{Mode}: \emph{Variational} (optimize particles $\boldsymbol{\mu}$ for fixed $\boldsymbol{y}$) or \emph{Amortized} (optimize reconstruction network $g_{\boldsymbol{\varphi}}$ for observations $\mathcal{Y}$).
    
    \State \textbf{Initialize} $\boldsymbol{\varphi}$:
    \If{\textbf{Mode} is \emph{Variational}}
        \State $\boldsymbol{\varphi} \gets \{\boldsymbol{\mu}_k\}_{k=1}^K$ initialized with measurement $\boldsymbol{y}$.
    \ElsIf{\textbf{Mode} is \emph{Amortized}}
        \State $\boldsymbol{\varphi} \gets$ weights of reconstruction network $g_{\boldsymbol{\varphi}}$ initialized from $\boldsymbol{s}_{\theta}$.
    \EndIf

    \While{not converged}
        \State \textbf{1. Sampling:}
        \If{\emph{Variational}} \Comment{\emph{Variational inference}}
            \State Use target measurement $\boldsymbol{y}$. Sample $\boldsymbol{x}_0 = \boldsymbol{\mu}_k$ from current particles.
        \Else \Comment{\emph{Amortized inference}}
            \State Sample batch $\boldsymbol{y} \sim \mathcal{Y}$. Add noise $\boldsymbol{\sigma} \sim \mathcal{N}(\mathbf{0}, \boldsymbol{h}\mathbf{I})$ to y. $\boldsymbol{y'} = \boldsymbol{y} +  \boldsymbol{\sigma}$.
            \State Generate $\boldsymbol{x}_0 = g_{\boldsymbol{\varphi}}(\boldsymbol{y'})$.
        \EndIf
        
        \State Sample $t \sim \mathcal{U}(0, T), \boldsymbol{\epsilon} \sim \mathcal{N}(\mathbf{0}, \mathbf{I})$. Compute $\boldsymbol{x}_t = \alpha_t \boldsymbol{x}_0 + \sigma_t \boldsymbol{\epsilon}$.
        
        \State \textbf{2. Auxiliary Score Update (Learning $\nabla \log q_{\varphi,t}(\boldsymbol{x_t}|\boldsymbol{y})$ with auxiliary network):}
        \State $\mathcal{L}_{\text{aux}} \gets w(t) \| \boldsymbol{s}_\phi(\boldsymbol{x}_t, t) - \nabla_{\boldsymbol{x}_t} \log p(\boldsymbol{x}_t|\boldsymbol{x}_0) \|^2_2$.
        \State $\boldsymbol{\phi} \gets \text{OptimizerStep}(\boldsymbol{\phi}, \mathcal{L}_{\text{aux}})$.
        
        \State \textbf{3. Reconstruction Optimization (Update $\boldsymbol{\varphi}$):}
        \State $\mathcal{L}_{\text{prior}} \gets - d'\left( \boldsymbol{s}_\phi(\boldsymbol{x}_t, t) - \boldsymbol{s}_{\theta}(\boldsymbol{x}_t, t) \right)^T \left\{ \boldsymbol{s}_\phi(\boldsymbol{x}_t, t) - \epsilon \right\} + d\left( \boldsymbol{s}_\phi(\boldsymbol{x}_t, t) - \boldsymbol{s}_{\theta}(\boldsymbol{x}_t, t) \right)$ via Theorem~\ref{THE:PPM}.
        \State $\mathcal{L}_{\text{data}} \gets \|\boldsymbol{y} - \mathcal{A}(\boldsymbol{x}_0)\|^2$.
        \State $\boldsymbol{\varphi} \gets \text{OptimizerStep}(\boldsymbol{\varphi}, \mathcal{L}_{\text{prior}} + \lambda\mathcal{L}_{\text{data}})$.
    \EndWhile
    
    \State \Return Optimized results $\boldsymbol{\varphi}$ (Particles $\boldsymbol{\mu}$ or reconstruction network $g_{\boldsymbol{\varphi}}$)
\end{algorithmic}
\end{algorithm*}

\subsection{Unified Optimization Framework}
\label{sec:unified_optimization}

Based on the gradient equivalence in Theorem~\ref{THE:PPM}, we formulate the total objective function for PPM as a combination of a data fidelity term and an unbiased prior regularization term:

\begin{equation}
\label{eq:total_objective}
\begin{split}
    &\mathcal{L}_{\text{PPM}}(\boldsymbol{\varphi}) = \mathcal{L}_{\text{prior}}(\boldsymbol{\varphi}) + \lambda \mathcal{L}_{\text{data}}(\boldsymbol{\varphi}), \\
    &\mathcal{L}_{\text{data}}(\boldsymbol{\varphi}) = \mathbb{E}_{\boldsymbol{x}_0 \sim q_{\boldsymbol{\varphi}}(\cdot|\boldsymbol{y})} \left[ \|\boldsymbol{y} - \mathcal{A}(\boldsymbol{x}_0)\|^2 \right], \\
    &\mathcal{L}_{\text{prior}}(\boldsymbol{\varphi}) = \mathbb{E}_{\substack{t \sim \mathcal{U}(0, T) \\ \boldsymbol{x}_0 \sim q_{\varphi}(\cdot|\boldsymbol{y}) \\ \boldsymbol{x}_t | \boldsymbol{x}_0 \sim p(\boldsymbol{x}_t | \boldsymbol{x}_0)}} \Big[ - d'\big( \boldsymbol{s}_\phi(\boldsymbol{x}_t, t) - \boldsymbol{s}_{p}(\boldsymbol{x}_t, t) \big)^\top \\
    &\big( \boldsymbol{s}_\phi(\boldsymbol{x}_t, t) - \nabla_{\mathbf{x}_t}\log p_{t}(\mathbf{x}_t|\mathbf{x}_0) \big) + d\big( \boldsymbol{s}_\phi(\boldsymbol{x}_t, t) - \boldsymbol{s}_{p}(\boldsymbol{x}_t, t) \big) \Big],
\end{split}
\end{equation}

where $\mathcal{L}_{\text{data}}$ enforces consistency with the measurement $\boldsymbol{y}$, and $\mathcal{L}_{\text{prior}}$ aligns the reconstruction with the diffusion prior using the unbiased gradient estimator derived in Theorem~\ref{THE:PPM}.
While our standard implementation utilizes the $L_2$ norm (Fisher divergence) where $d(\boldsymbol{u}, \boldsymbol{v}) = \|\boldsymbol{u}-\boldsymbol{v}\|_2^2$, our framework is theoretically general: it naturally extends to other divergences by selecting different convex distance metrics $d$, offering scalability to various score matching variants.

PPM provides a unified training logic for both VI and AI inference paradigms. The unified training procedure is summarized in Algorithm~\ref{alg:ppm_unified}. Despite the difference in parameterization, both paradigms operate via an identical alternating two-stage process:

\begin{itemize}
    \item \textbf{Stage 1: Auxiliary Score Learning (Update $\phi$).} We update the auxiliary network $\boldsymbol{s}_\phi$ to minimize $\mathcal{L}_{\text{aux}}$ (Eq.~\ref{eq:aux_loss}). This step effectively learns the score $\nabla \log q_{\boldsymbol{\varphi},t}$ of the current variational distribution (defined either by particles or a reconstruction network).
    \item \textbf{Stage 2: Reconstruction Optimization (Update $\varphi$).} We update the variational parameters $\boldsymbol{\varphi}$ to minimize $\mathcal{L}_{\text{PPM}}$ (Eq.~\ref{eq:total_objective}). This step utilizes the gradient provided by the now-fixed $\boldsymbol{s}_\phi$ to drive the posterior estimate towards the true prior and measurement.
\end{itemize}

\textbf{Variational Inference (Particle-based).}
In the VI setting, we optimize for a specific single observation $\boldsymbol{y}$. The variational parameters are defined as a set of image particles $\boldsymbol{\varphi} = \{\boldsymbol{\mu}_k\}_{k=1}^K$, initialized as $\boldsymbol{\mu}_k = \boldsymbol{y}$ (or a rough inverse). The distribution $q_{\boldsymbol{\varphi}}(\boldsymbol{x}|\boldsymbol{y})$ is represented empirically by these particles. The optimization refines $\boldsymbol{\mu}_k$ to capture the complex, multi-modal posterior landscape specific to $\boldsymbol{y}$.

\textbf{Amortized Inference (Neural network-based).}
In the AI setting, we learn a global mapping for any observation $\boldsymbol{y} \sim p(\boldsymbol{y})$. The variational parameter $\boldsymbol{\varphi}$ denotes the weights of a neural network $g_{\boldsymbol{\varphi}}$, such that $\boldsymbol{x} = g_{\boldsymbol{\varphi}}(\boldsymbol{y})$. To accelerate convergence, we implement $g_{\boldsymbol{\varphi}}$ as a copy of the pre-trained diffusion U-Net (initialized with $\boldsymbol{\theta}$), enabling efficient, single-step reconstruction. This amortizes the optimization cost, allowing rapid inference at test time.

Beyond the standard formulation presented above, we highlight the inherent extensibility of the PPM framework. While this work primarily employs the squared $L_2$ norm as the metric function $d(\cdot)$—which corresponds to the standard Fisher divergence and minimizes the Kullback-Leibler divergence—our theoretical derivations (Theorem~\ref{THE:PPM}) are not restricted to this choice. The metric $d(\cdot)$ can be substituted with a broader class of convex distance functions. This flexibility allows PPM to be naturally generalized to measure and minimize a wider spectrum of divergences, positioning it as a versatile foundation for score-based variational inference.
\section{Theoretical Analysis}
\label{sec:TheoreticalAnalysis}

In this section, we provide a rigorous theoretical comparison between PPM and three leading categories of baselines of diffusion-based inverse problem solvers. We demonstrate that while these methods offer practical utility, they all rely on biased approximations of the true posterior objective.

\subsection{Bias in Optimization-based VI (RED-Diff and RLSD)}
Optimization-based methods like RED-Diff~\cite{mardani2023variational} and RLSD~\cite{zilberstein2024repulsive} formulate the inverse problem as a variational optimization but share a fundamental theoretical flaw in their handling of entropy.

\subsubsection{Missing Exact Entropy (RED-Diff)}
RED-Diff implicitly models the posterior as a Dirac delta distribution. This effectively removes the entropy term $H(q)$ from the KL divergence, collapsing the problem to Maximum A Posteriori (MAP) estimation. Consider the standard decomposition:
\begin{equation}
D_{KL}(q(\boldsymbol{x})||p(\boldsymbol{x}|\boldsymbol{y})) = -\mathbb{E}_{q}[\log p(\boldsymbol{x}|\boldsymbol{y})] - H(q).
\end{equation}
Under the Dirac assumption $q(\boldsymbol{x})=\delta(\boldsymbol{x}-\boldsymbol{\mu})$, the entropy vanishes ($\nabla_{\boldsymbol{\mu}}H=0$) and the energy term collapses to $\log p(\boldsymbol{\mu}|\boldsymbol{y})$. Consequently, the minimization problem becomes mathematically equivalent to MAP estimation:
\begin{equation}
\min_{\boldsymbol{\mu}} D_{KL}(q||p) \iff \max_{\boldsymbol{\mu}} \log p(\boldsymbol{\mu}|\boldsymbol{y}).
\end{equation}
This confirms that without the entropy term, the optimization inherently seeks the single most probable mode rather than the full distribution, explaining the severe mode-seeking behavior observed in Figure 1.

\subsubsection{Surrogate Entropy via Repulsion (RLSD)}
While RLSD mitigates mode collapse using particles, it introduces ad-hoc repulsive regularization instead of optimizing the true entropy. This repulsion acts as a heuristic proxy and does not correspond to the true score of the variational distribution, resulting in artificial uncertainty dependent on hyperparameters.

\subsection{Bias in MCMC Sampling (DPS)}
Diffusion Posterior Sampling (DPS)~\cite{chung2022diffusion} approximates samples by modifying the reverse diffusion score with a likelihood guidance term. Since the likelihood $p(\boldsymbol{y}|\boldsymbol{x}_t)$ is intractable, DPS approximates it using a clean data estimate $\hat{\boldsymbol{x}}_0(\boldsymbol{x}_t) = \mathbb{E}[\boldsymbol{x}_0|\boldsymbol{x}_t]$. This violates Jensen's inequality by treating the expectation of the likelihood as the likelihood of the expectation. This introduces systematic score estimation error, particularly in early diffusion stages, which accumulates over the trajectory.

\subsection{Bias in Amortized Inference (DAVI)}
DAVI~\cite{lee2024diffusion} employs the Integral KL (IKL) divergence~\cite{luo2023diff} to train amortized generators. Here, we formally prove that replacing the standard prior KL divergence with the IKL objective fundamentally alters the optimization target.

\subsubsection{Problem Formulation}
Standard VI minimizes $\mathcal{J}_{VI} = \mathbb{E}_{q_{\varphi}}[-\log p(\boldsymbol{y}|\boldsymbol{x})] + D_{KL}(q_{\varphi}||p)$. IKL-based methods replace the prior term with an integrated objective $\mathcal{J}_{Mod} = \mathbb{E}[-\log p(\boldsymbol{y}|\boldsymbol{x})] + \int \omega(t) D_{KL}(q_{t}||p_{t})dt$.

\subsubsection{KL Contraction and Implicit Prior}
Assuming a Variance Preserving (VP) schedule and a Gaussian Mean Shift assumption ($p(\boldsymbol{x})=\mathcal{N}(\boldsymbol{0},\boldsymbol{I}), q(\boldsymbol{x}|\boldsymbol{y})=\mathcal{N}(\boldsymbol{\Delta}, \boldsymbol{I})$), the KL divergence scales quadratically: $D_{KL}(q_{t}||p_{t}) \approx \alpha_{t}^{2} \cdot D_{KL}(q_{0}||p_{0})$.

Substituting this scaling law into the IKL integral, we rewrite the modified objective as:
\begin{equation}
\mathcal{J}_{IKL} \approx \left(\int_{0}^{T}\omega(t)\alpha_{t}^{2}dt\right) \cdot D_{KL}(q_{\varphi}(\boldsymbol{x}|\boldsymbol{y})||p(\boldsymbol{x})) = \beta D_{KL}(q_{\varphi}||p).
\end{equation}
This effectively scales the prior weight by $\beta$. Expanding the terms reveals the implicit posterior target $p'(\boldsymbol{x}|\boldsymbol{y})$:
\begin{equation}
\mathcal{J}_{Mod} = \beta \mathbb{E}_{q_{\varphi}}\left[\log q_{\varphi}(\boldsymbol{x}|\boldsymbol{y}) - \left(\frac{1}{\beta}\log p(\boldsymbol{y}|\boldsymbol{x}) + \log p(\boldsymbol{x})\right)\right].
\end{equation}
Exponentiating this result implies optimization against a Distorted Prior $p'(\boldsymbol{x}) \propto p(\boldsymbol{x})^{\beta}$. Since typically $\beta < 1$, the effective prior is a flattened, high-temperature version of the true prior, proving that IKL leads to biased posterior estimation.
\vspace*{-0.1cm}
\section{Experiment}
\label{sec:experiment}


In this section, we compare our proposed method, PPM, with state-of-the-art (SoTA) diffusion model-based methods for solving inverse problems, particularly those employing variational inference. Our experiments aim to demonstrate PPM's capability to generate diverse reconstructions while maintaining fidelity, thereby recovering the full posterior and surpassing baselines that typically yield homogeneous results. 

\subsection{Toy Examples of 2D Posterior Estimation}
\label{subsec:2dtoy}
We first validated PPM on two simple 2D posterior‐estimation tasks. In each task, the hidden variable $x\in \mathbb{R}^{2 \times 1}$ follows a Gaussian‐mixture prior, and measurements obey the linear model
\begin{equation}
y = F x + n,
\end{equation}
where $F$ is a linear projection matrix and $n\sim\mathcal{N}(0,\sigma^2I)$ is Gaussian noise. We compare PPM against two related baselines, RED-Diff \cite{mardani2023variational} and RLSD \cite{zilberstein2024repulsive}, using the same pre-trained pixel-space diffusion prior. To ensure fair compaison, we adapt RLSD by replacing its latent diffusion backbone with our pixel-space prior and incorporating its repulsive term into a pixel-space SDS loss, so that all methods operate under identical conditions.

Figure \ref{fig:teaser} highlights the significant differences in each method’s ability to capture multimodal posteriors. RED-Diff exhibits severe mode collapse: its samples converge to accurate point estimates but lack diversity. RLSD produces more varied samples, yet its empirical repulsive term leads to either measurement inconsistency (top example) or poor prior adherence (bottom example). In contrast, PPM faithfully recovers every posterior mode, yielding both accurate and diverse samples that closely match the ground‐truth bimodal distribution.

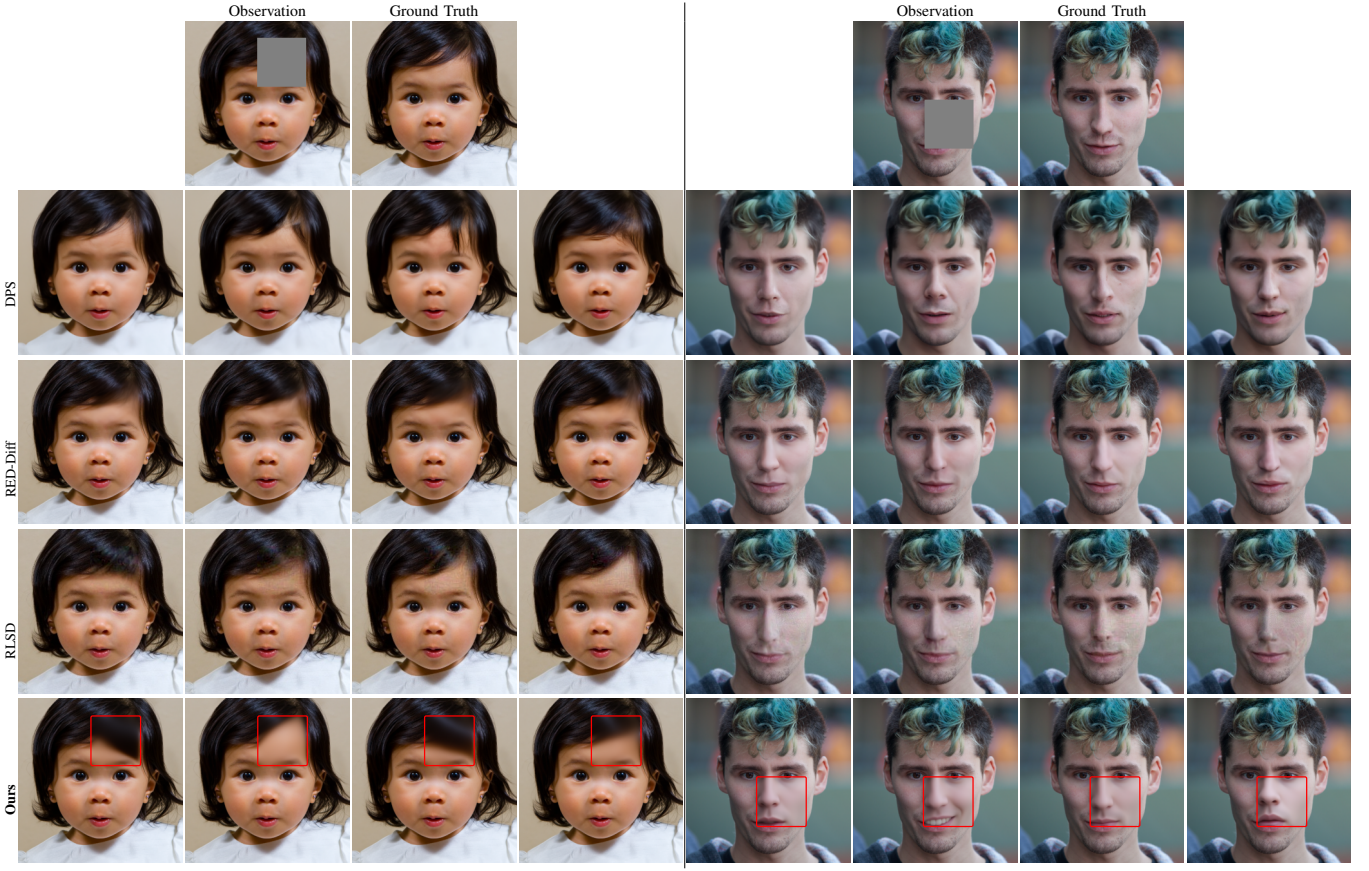
\begin{figure*}[!ht]
    \centering
    \setlength{\tabcolsep}{1pt}
    \setlength{\fboxrule}{1pt}
    \resizebox{0.99\textwidth}{!}{
    \begin{tabular}{c}
    \begin{tabular}{ccccc|cccc}
        & & Observation & Ground Truth & & & Observation & Ground Truth &
        \\
        & &
        \includegraphics[width=0.2\textwidth]{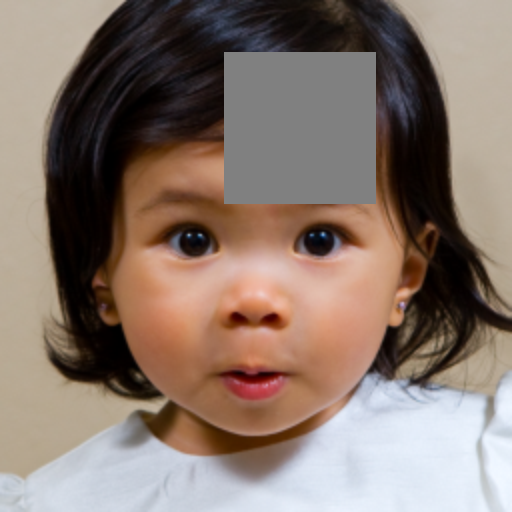} &
        \includegraphics[width=0.2\textwidth]{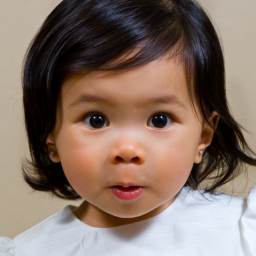} & & &
        \includegraphics[width=0.2\textwidth]{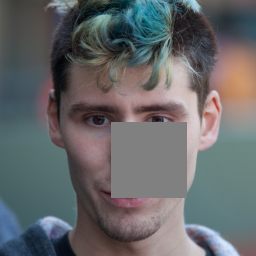} &
        \includegraphics[width=0.2\textwidth]{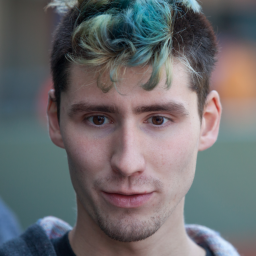} &
        \\
        \begin{turn}{90} \,\,\,\,\,\,\,\,\,\,\,\,\,\,\,\,\,\,\small{DPS} \end{turn} &
        \includegraphics[width=0.2\textwidth]{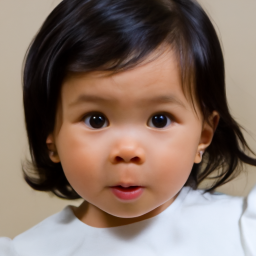} &
        \includegraphics[width=0.2\textwidth]{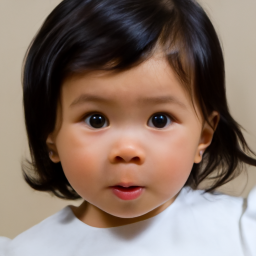} &
        \includegraphics[width=0.2\textwidth]{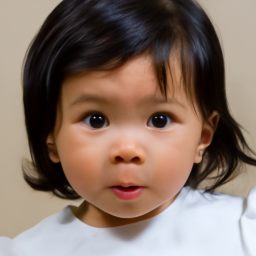} &
        \includegraphics[width=0.2\textwidth]{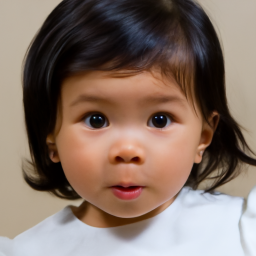} &
        \includegraphics[width=0.2\textwidth]{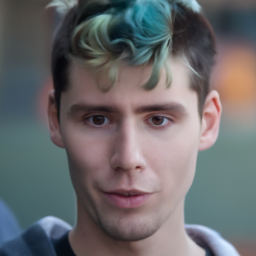} &
        \includegraphics[width=0.2\textwidth]{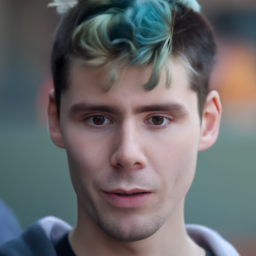} &
        \includegraphics[width=0.2\textwidth]{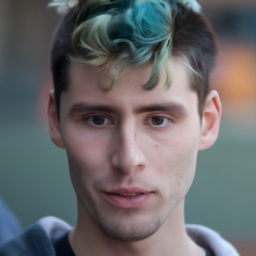} &
        \includegraphics[width=0.2\textwidth]{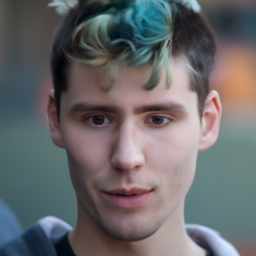}
        \\
        \begin{turn}{90} \,\,\,\,\,\,\,\,\,\,\small{RED-Diff} \end{turn} &
        \includegraphics[width=0.2\textwidth]{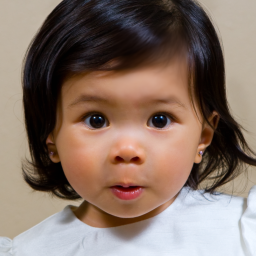} &
        \includegraphics[width=0.2\textwidth]{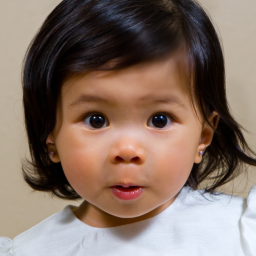} &
        \includegraphics[width=0.2\textwidth]{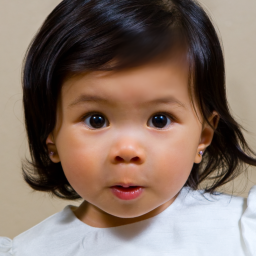} &
        \includegraphics[width=0.2\textwidth]{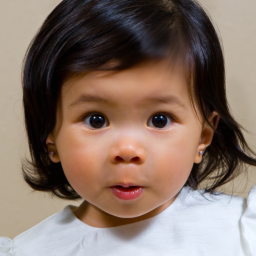} &
        \includegraphics[width=0.2\textwidth]{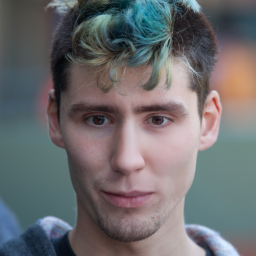} &
        \includegraphics[width=0.2\textwidth]{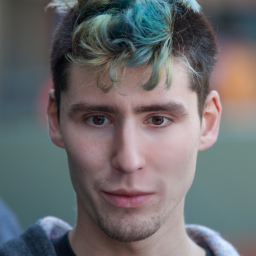} &
        \includegraphics[width=0.2\textwidth]{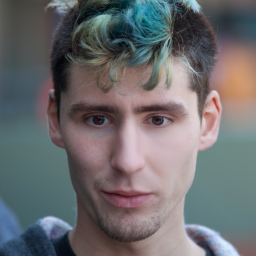} &
        \includegraphics[width=0.2\textwidth]{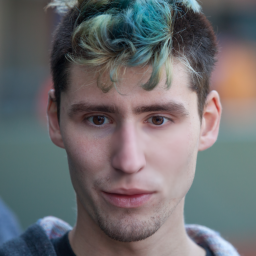} 
        \\
        \begin{turn}{90} \,\,\,\,\,\,\,\,\,\,\,\,\,\small{RLSD} \end{turn} &
        \includegraphics[width=0.2\textwidth]{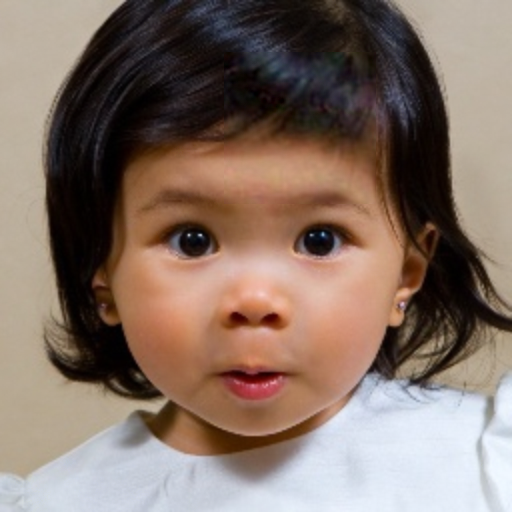} &
        \includegraphics[width=0.2\textwidth]{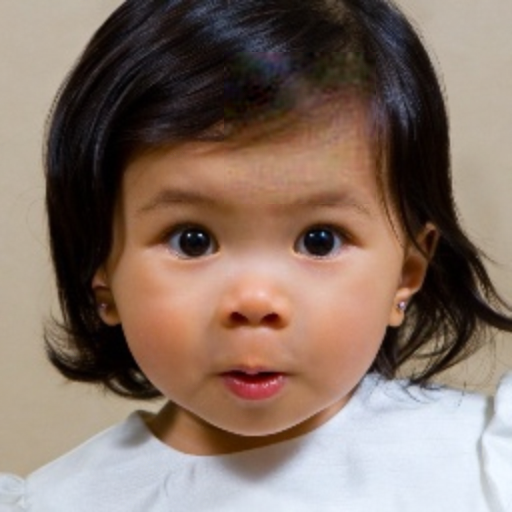} &
        \includegraphics[width=0.2\textwidth]{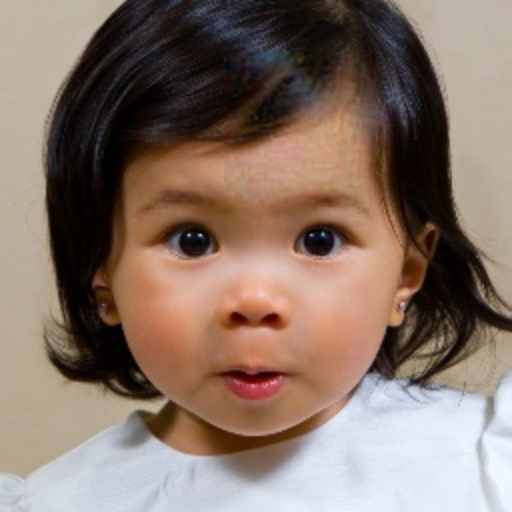} &
        \includegraphics[width=0.2\textwidth]{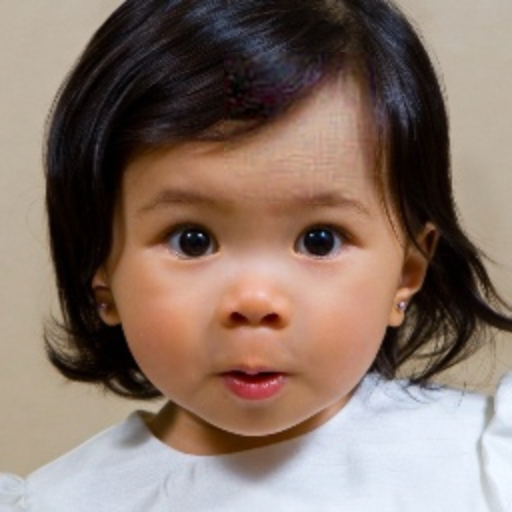} &
        \includegraphics[width=0.2\textwidth]{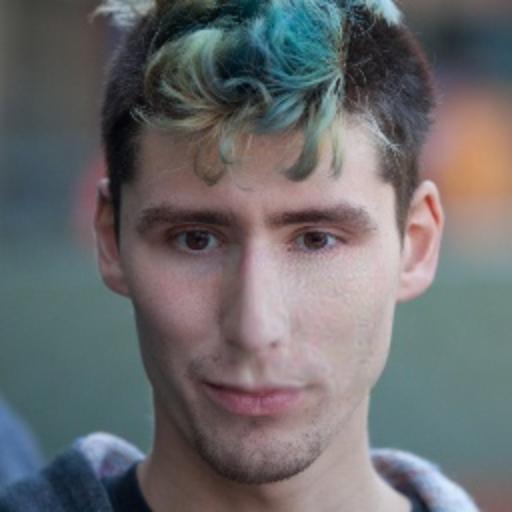} &
        \includegraphics[width=0.2\textwidth]{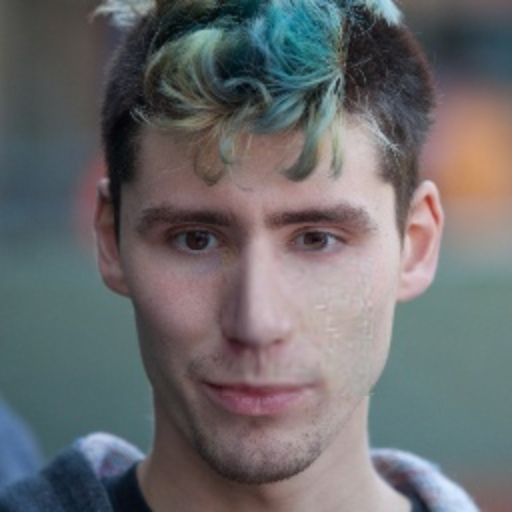} &
        \includegraphics[width=0.2\textwidth]{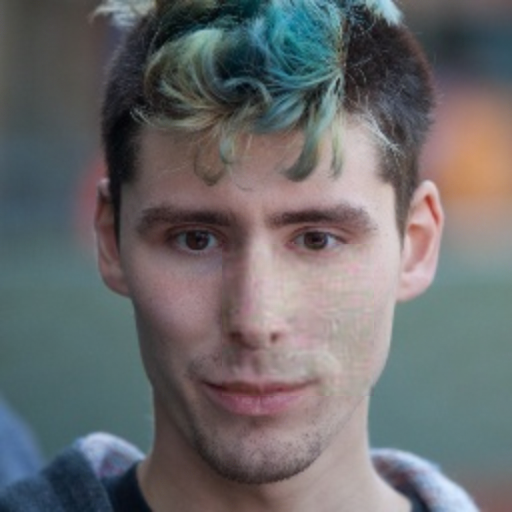} &
        \includegraphics[width=0.2\textwidth]{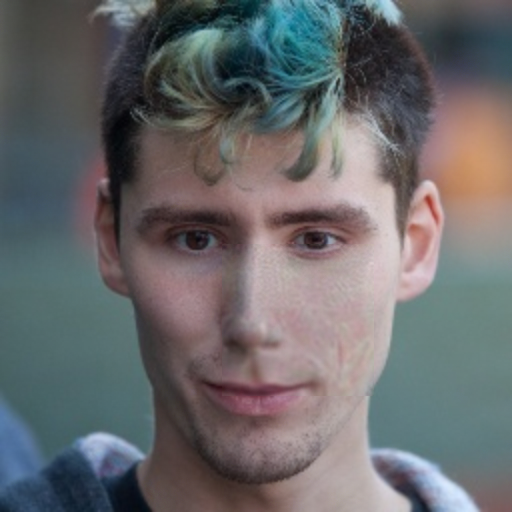} 
        \\
        \begin{turn}{90} \,\,\,\,\,\,\,\,\,\,\,\,\,\,\,\,\,\,\small{\textbf{Ours}} \end{turn} &
        \begin{tikzpicture}
        \node[anchor=south west, inner sep=0pt, outer sep=0pt] (image) at (0,0)
            {\includegraphics[width=0.2\textwidth]{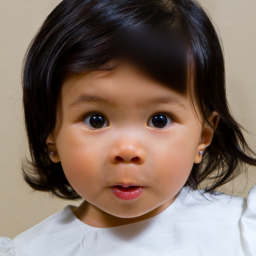}};
        \draw[red, thick, rounded corners=1pt] (1.6cm, 2.15cm) rectangle (2.7cm, 3.25cm);
        \end{tikzpicture} &
        \begin{tikzpicture}
        \node[anchor=south west, inner sep=0pt, outer sep=0pt] (image) at (0,0)
            {\includegraphics[width=0.2\textwidth]{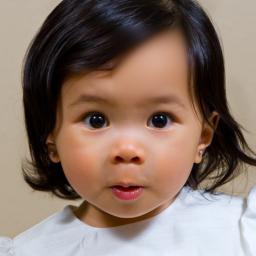}};
        \draw[red, thick, rounded corners=1pt] (1.6cm, 2.15cm) rectangle (2.7cm, 3.25cm);
        \end{tikzpicture} &
        \begin{tikzpicture}
        \node[anchor=south west, inner sep=0pt, outer sep=0pt] (image) at (0,0)
            {\includegraphics[width=0.2\textwidth]{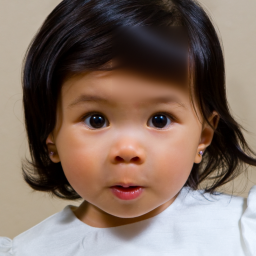}};
        \draw[red, thick, rounded corners=1pt] (1.6cm, 2.15cm) rectangle (2.7cm, 3.25cm);
        \end{tikzpicture} &
        \begin{tikzpicture}
        \node[anchor=south west, inner sep=0pt, outer sep=0pt] (image) at (0,0)
            {\includegraphics[width=0.2\textwidth]{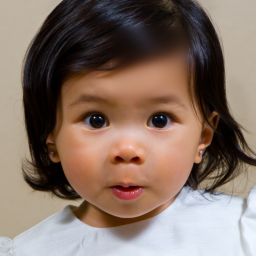}};
        \draw[red, thick, rounded corners=1pt] (1.6cm, 2.15cm) rectangle (2.7cm, 3.25cm);
        \end{tikzpicture} &
        \begin{tikzpicture}
        \node[anchor=south west, inner sep=0pt, outer sep=0pt] (image) at (0,0)
            {\includegraphics[width=0.2\textwidth]{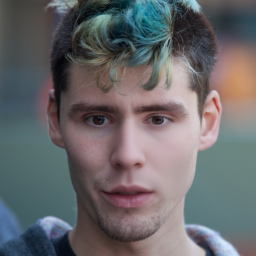}};
        \draw[red, thick, rounded corners=1pt] (1.55cm, 0.8cm) rectangle (2.65cm, 1.9cm);
        \end{tikzpicture} &
        \begin{tikzpicture}
        \node[anchor=south west, inner sep=0pt, outer sep=0pt] (image) at (0,0)
            {\includegraphics[width=0.2\textwidth]{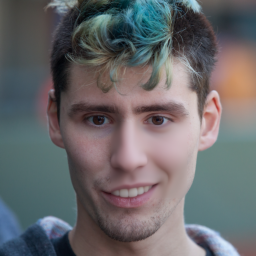}};
        \draw[red, thick, rounded corners=1pt] (1.55cm, 0.8cm) rectangle (2.65cm, 1.9cm);
        \end{tikzpicture} &
        \begin{tikzpicture}
        \node[anchor=south west, inner sep=0pt, outer sep=0pt] (image) at (0,0)
            {\includegraphics[width=0.2\textwidth]{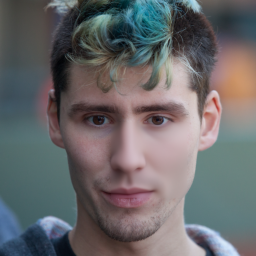}};
        \draw[red, thick, rounded corners=1pt] (1.55cm, 0.8cm) rectangle (2.65cm, 1.9cm);
        \end{tikzpicture} &
        \begin{tikzpicture}
        \node[anchor=south west, inner sep=0pt, outer sep=0pt] (image) at (0,0)
            {\includegraphics[width=0.2\textwidth]{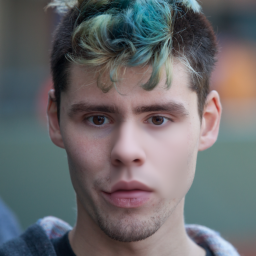}};
        \draw[red, thick, rounded corners=1pt] (1.55cm, 0.8cm) rectangle (2.65cm, 1.9cm);
        \end{tikzpicture}
    \end{tabular}
    \end{tabular}}
    \vspace{-0.1em}
    \caption{\textbf{Comparison of PPM and diffusion-based VI baselines (DPS, RED-Diff and RLSD) on box inpainting with FFHQ.} From top to bottom: the masked observation and the ground truth, followed by four random posterior samples from DPS, RED-Diff, RLSD, and PPM. Although all methods yield plausible completions, PPM produces markedly more diverse and higher-fidelity samples within the inpainted region (red box). By contrast, the baselines generate nearly identical outputs, indicating a failure to capture posterior uncertainty.}
    \label{fig:box_inpainting}
    \vspace{-0.1em}
\end{figure*}

\subsection{Computational Photography}
We then evaluated PPM on various computational photography tasks, including box inpainting, motion deblurring, and super-resolution, using natural images.
\paragraph{ Datasets and Pretrained Models }
We primarily evaluate on two natural-image datasets with distinct characteristics—FFHQ (256 × 256) \cite{karras2019style} and ImageNet (256 × 256) \cite{deng2009imagenet}—using 32 randomly selected images from each validation set. For FFHQ, we use the diffusion prior released by DPS \cite{chung2022diffusion}, and for ImageNet we employ the model from \cite{dhariwal2021diffusion}. Both are used off-the-shelf, without task-specific fine-tuning. RLSD \cite{zilberstein2024repulsive} requires a latent diffusion backbone, so we use Stable Diffusion v2.1 \cite{rombach2022high} following its original setup.

\paragraph{Baselines} 
We compare PPM against state‐of‐the‐art computational imaging methods that leverage diffusion priors, spanning both variational‐inference and MCMC‐sampling paradigms. Our VI baselines include RED-Diff \cite{mardani2023variational} and RLSD \cite{zilberstein2024repulsive}, while our MCMC baselines are $\Pi$GDM \cite{song2022pseudoinverse} and DPS \cite{chung2022diffusion}. All methods use the same pixel-space diffusion prior—except RLSD, which retains its Stable Diffusion prior—and are evaluated with hyperparameters set to their original defaults to ensure a fair comparison.

\paragraph{Evaluation Metrics} 
We assess reconstruction fidelity using PSNR and SSIM, and quantify sampling diversity using the  pairwise cosine similarity among $N$ reconstruction samples from the same observation. The final diversity is the average of all $O$ observations:
\[
1 \;-\; \frac{1}{|O|}\sum_{o}\frac{2}{N(N-1)} \sum_{i<j} \frac{\langle x_{oi}, x_{oj}\rangle}{\|x_{oi}\|_2\,\|x_{oj}\|_2}\,.
\]
For quantitative and efficient comparison, we estimate over 64 observations in FFHQ validation dataset and reconstruct 8 samples from each observation.

\paragraph{Inpainting}
Box-inpainting results are assessed on FFHQ validation images.
Qualitative comparisons with $80\times80$ masked boxes are visualized in Fig.~\ref{fig:box_inpainting}, while quantitative metrics evaluated on larger $128\times128$ center masks are reported in Table~\ref{tab:quan-results}.
We compare PPM against two VI baselines, RED-Diff \cite{mardani2023variational} and RLSD \cite{zilberstein2024repulsive}.
PPM significantly surpasses both baselines in diversity while simultaneously delivering superior reconstruction quality (PSNR and SSIM).
In challenging cases where critical features like hair (Fig.~\ref{fig:box_inpainting} left) or facial contours (Fig.~\ref{fig:box_inpainting} right) are masked, RED-Diff collapses to nearly identical outputs and RLSD struggles to trade off data fidelity against prior adherence. By contrast, PPM produces varied yet plausible reconstructions (e.g., different hairstyles and lip shapes) thanks to its principled score-based divergence and tailored optimization strategy. As Table~\ref{tab:quan-results} shows, PPM attains the top scores across all metrics on the validation set—demonstrating superior quality and diversity without extensive tuning—and also outperforms MCMC-based methods such as DPS.

\begin{figure*}[t]
    \centering
    \setlength{\tabcolsep}{1pt}
    \setlength{\fboxrule}{1pt}
    \resizebox{0.99\textwidth}{!}{
    \begin{tabular}{c}
    \begin{tabular}{cccccccccc}
        & Observation & DPS~\cite{chung2022diffusion} & 
        \multicolumn{2}{c}{{RLSD~\cite{zilberstein2024repulsive}}} &
        \multicolumn{2}{c}{{RED-Diff~\cite{mardani2023variational}}} & \multicolumn{2}{c}{{\textbf{Ours}}} & Ground Truth
        \\
        &
        \includegraphics[width=0.2\textwidth]{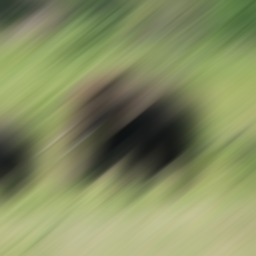} &
        \includegraphics[width=0.2\textwidth]{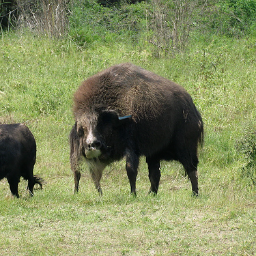} &
        \includegraphics[width=0.2\textwidth]{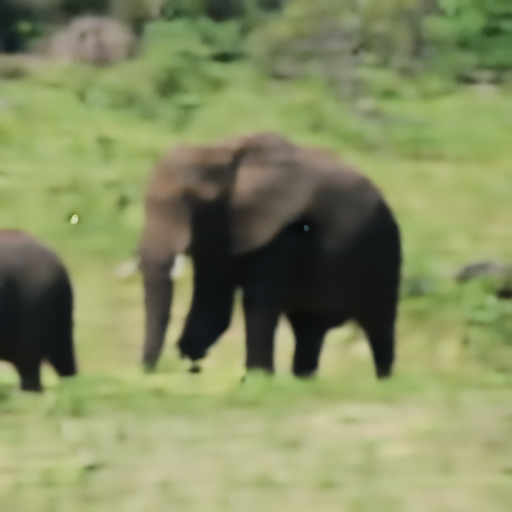} &
        \includegraphics[width=0.2\textwidth]{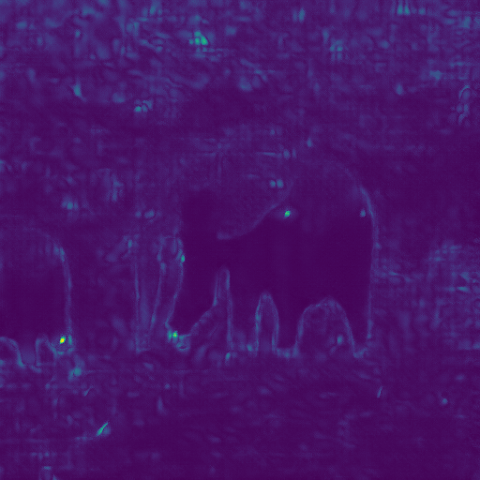} &
        \includegraphics[width=0.2\textwidth]{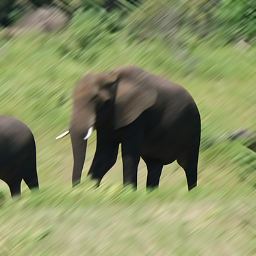} &
        \includegraphics[width=0.2\textwidth]{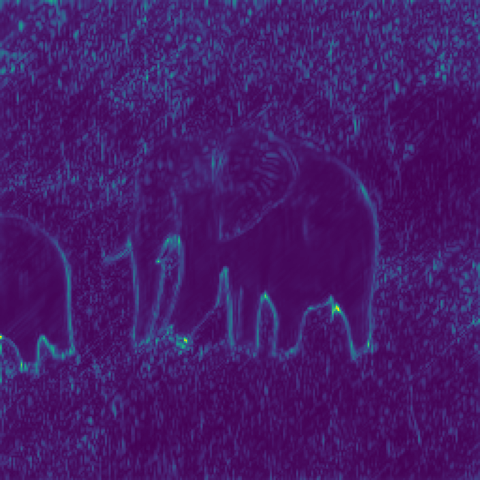} &
        \includegraphics[width=0.2\textwidth]{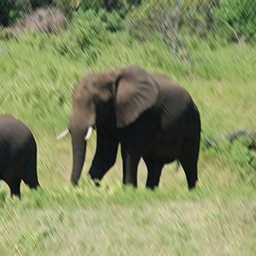} &
        \includegraphics[width=0.2\textwidth]{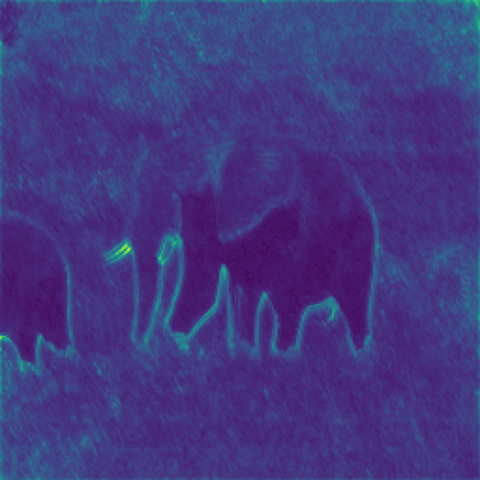} &
        \includegraphics[width=0.2\textwidth]{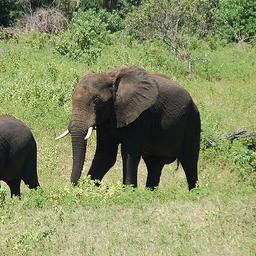} 
        \\
        &
        \includegraphics[width=0.2\textwidth]{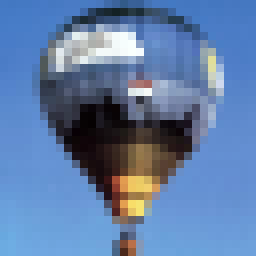} &
        \includegraphics[width=0.2\textwidth]{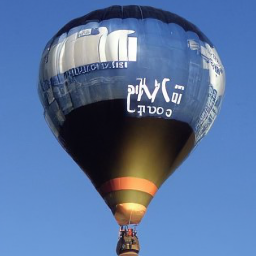} &
        \includegraphics[width=0.2\textwidth]{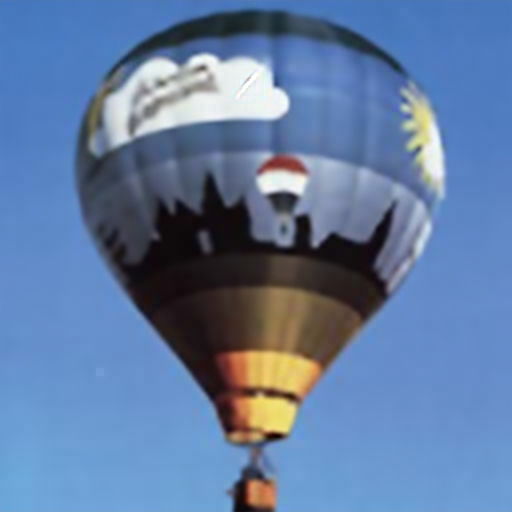} &
        \includegraphics[width=0.2\textwidth]{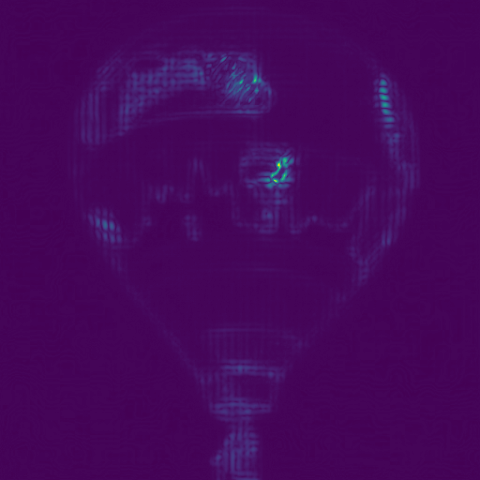} &
        \includegraphics[width=0.2\textwidth]{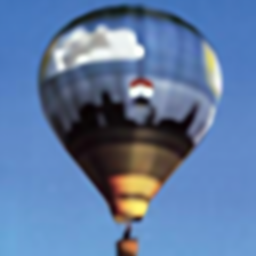} &
        \includegraphics[width=0.2\textwidth]{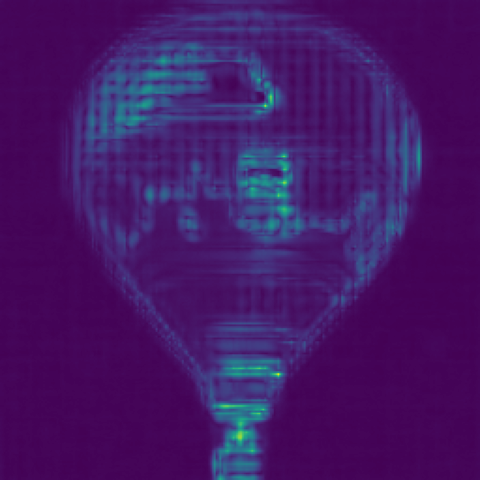} &
        \includegraphics[width=0.2\textwidth]{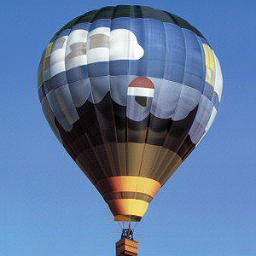} &
        \includegraphics[width=0.2\textwidth]{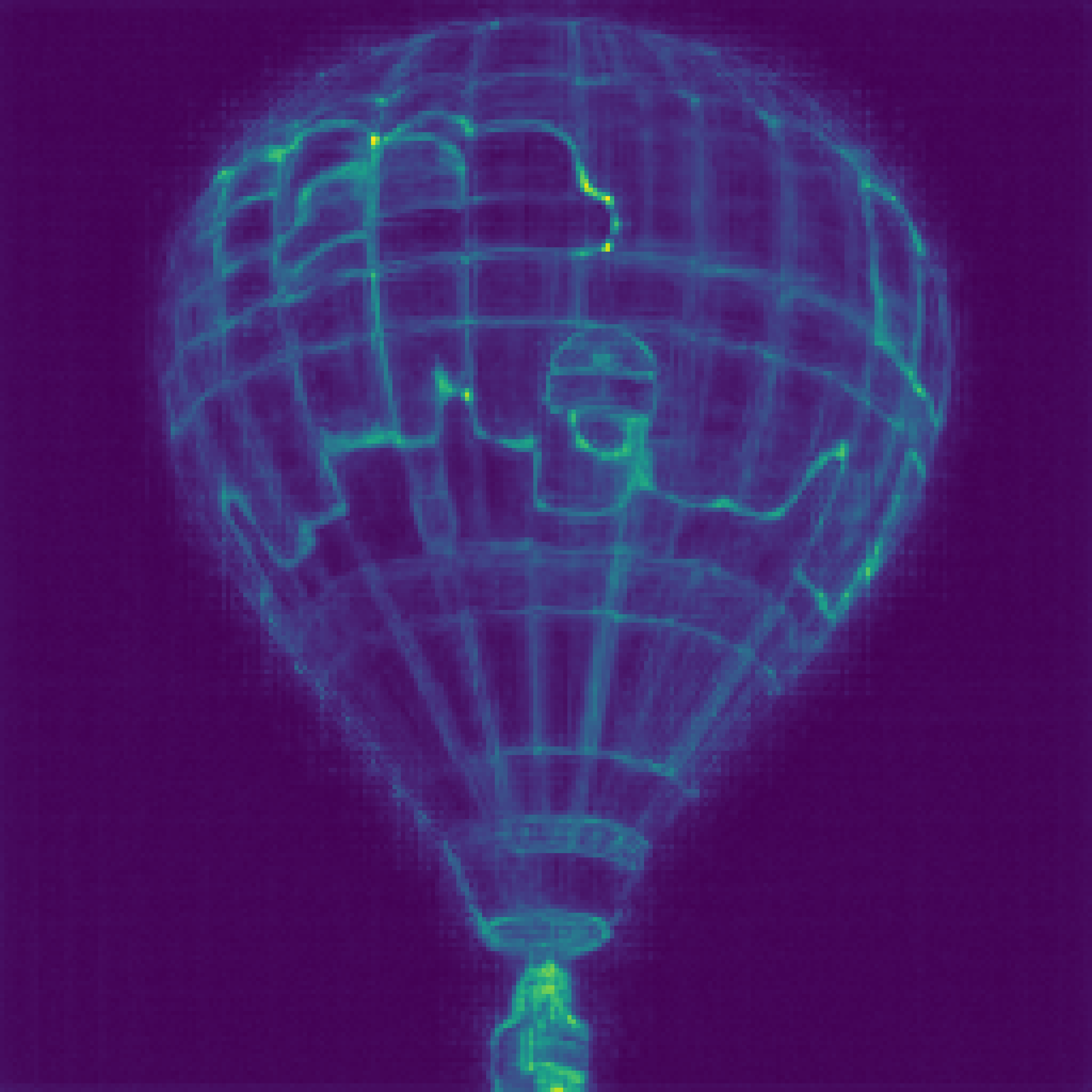} &
        \includegraphics[width=0.2\textwidth]{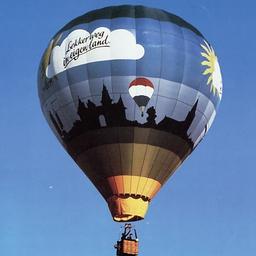} 
        \\
        &
        \includegraphics[width=0.2\textwidth]{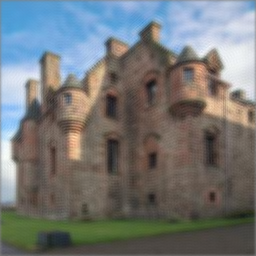} &
        \includegraphics[width=0.2\textwidth]{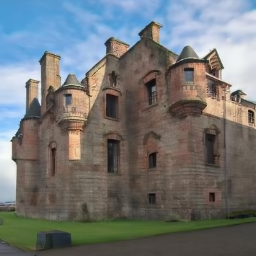} &
        \includegraphics[width=0.2\textwidth]{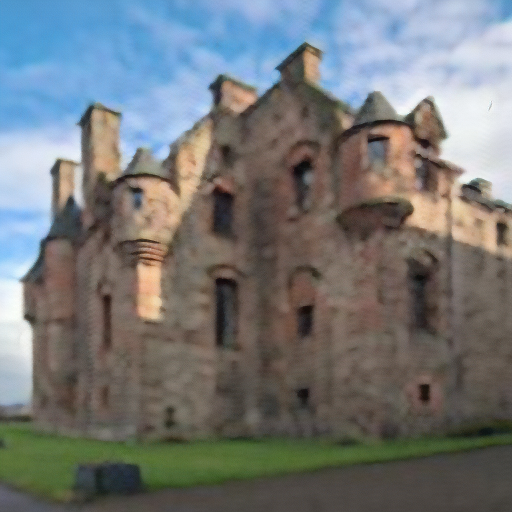} &
        \includegraphics[width=0.2\textwidth]{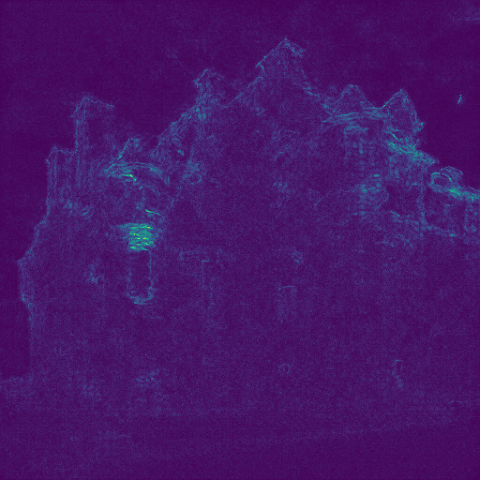} &
        \includegraphics[width=0.2\textwidth]{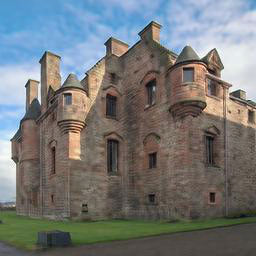} &
        \includegraphics[width=0.2\textwidth]{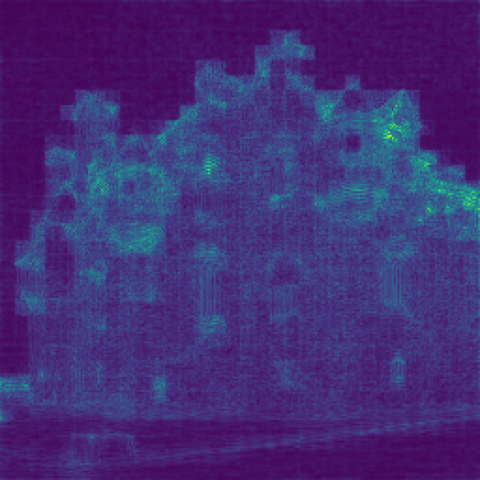} &
        \includegraphics[width=0.2\textwidth]{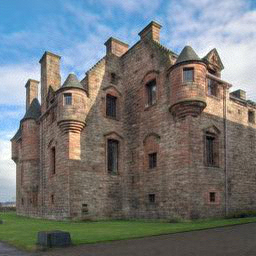} &
        \includegraphics[width=0.2\textwidth]{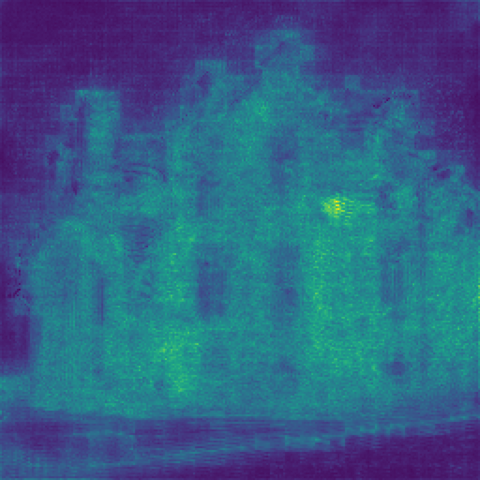} &
        \includegraphics[width=0.2\textwidth]{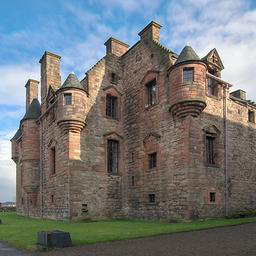} 
    \end{tabular}
    \end{tabular}}
    \label{subfig:mesh}
    \caption{\textbf{Comparison of PPM, DPS, and diffusion-based VI baselines (RED-Diff and RLSD) on motion deblurring, super‐resolution and gaussian deblurring tasks with ImageNet}. For each VI method, we show one sample reconstruction (left) alongside its standard‐deviation uncertainty map (right). Compared to DPS, PPM achieves higher fidelity—accurately rendering details like the elephant’s tusks, balloon logos, and human faces. Compared to VI methods, PPM produces better calibrated uncertainty: RED-Diff’s standard‐deviation maps reveal mode collapse, and RLSD’s contain pronounced artifacts.}
    \label{fig:denoisingdeblurring}
    \vspace{-0.1em}
\end{figure*}

\begin{figure*}[t]
    \centering
    \setlength{\tabcolsep}{1pt}
    \setlength{\fboxrule}{1pt}
    \resizebox{0.99\textwidth}{!}{
    \begin{tabular}{c}
    \begin{tabular}{cccccc}
        Observation &  
        \multicolumn{2}{c}{{DAVI~\cite{lee2024diffusion}}} &
        \multicolumn{2}{c}{{Ours}} &
        Ground Truth
        \\
        \includegraphics[width=0.2\textwidth]{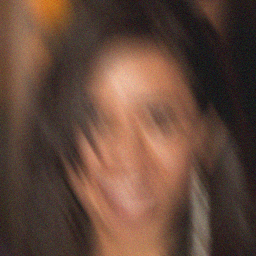} &
        \includegraphics[width=0.2\textwidth]{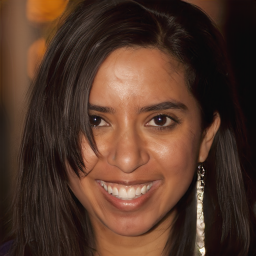} &
        \includegraphics[width=0.2\textwidth]{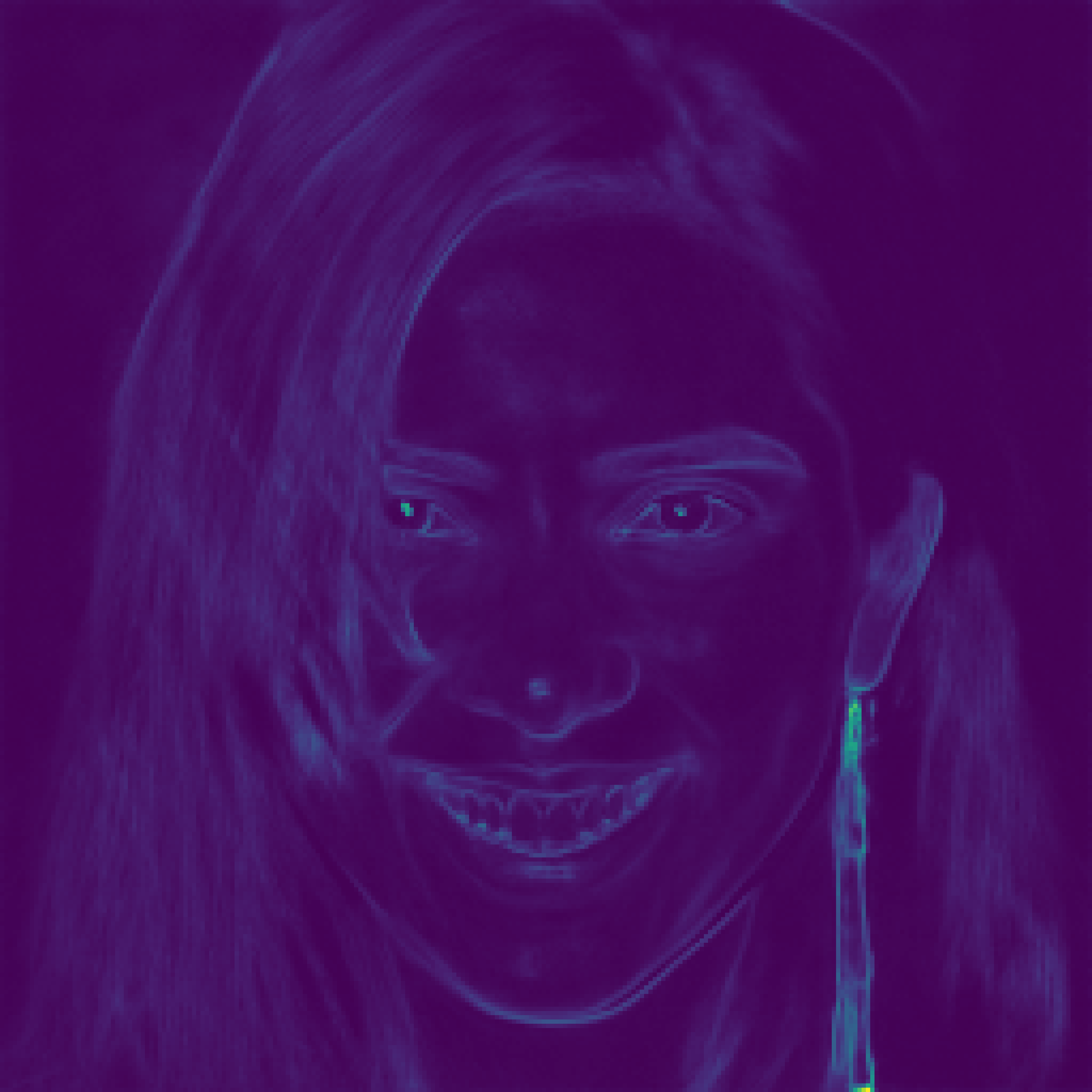} &
        \includegraphics[width=0.2\textwidth]{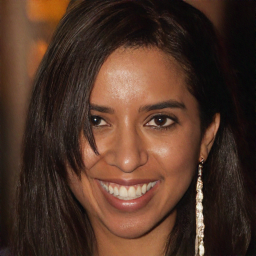} &
        \includegraphics[width=0.2\textwidth]{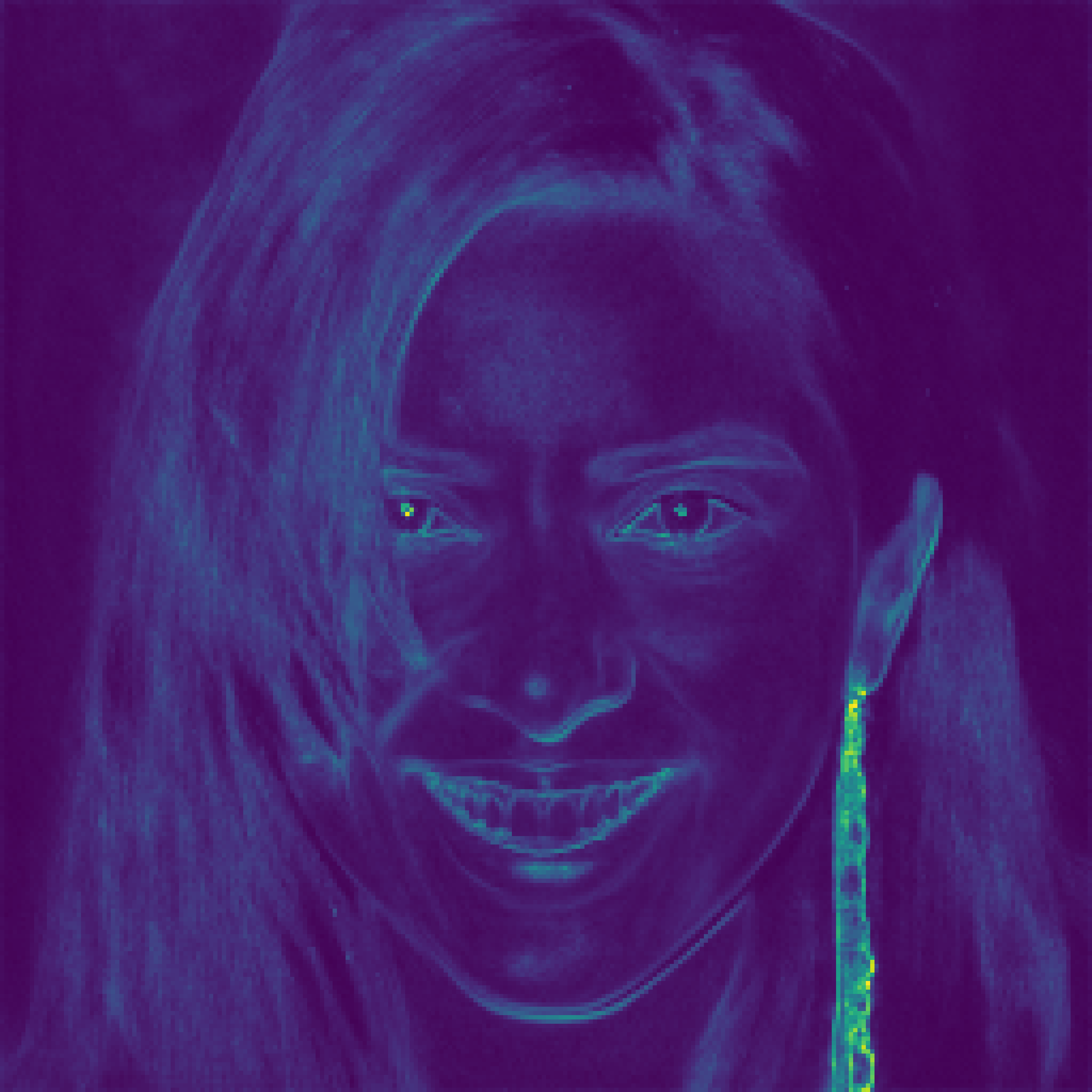} &
        \includegraphics[width=0.2\textwidth]{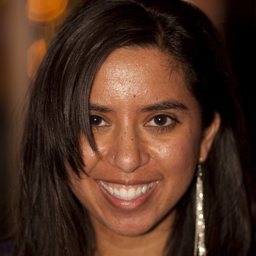}
        \\
        \includegraphics[width=0.2\textwidth]{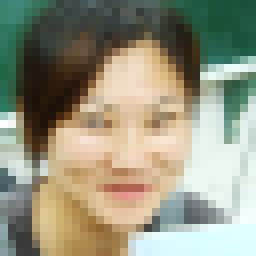} &
        \includegraphics[width=0.2\textwidth]{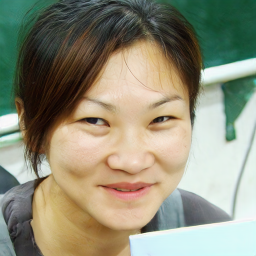} &
        \includegraphics[width=0.2\textwidth]{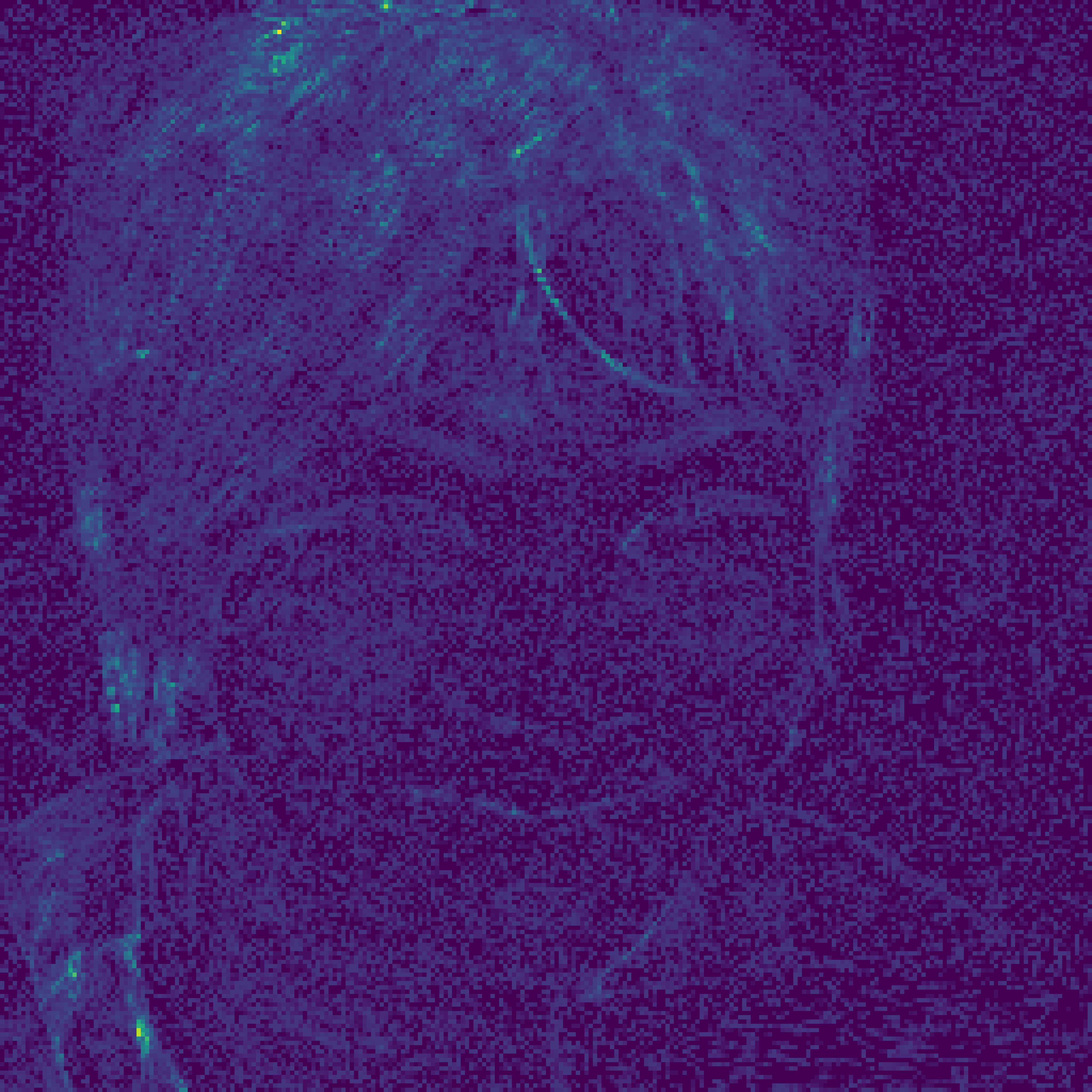} &
        \includegraphics[width=0.2\textwidth]{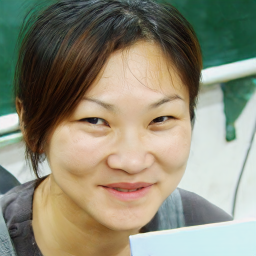} &
        \includegraphics[width=0.2\textwidth]{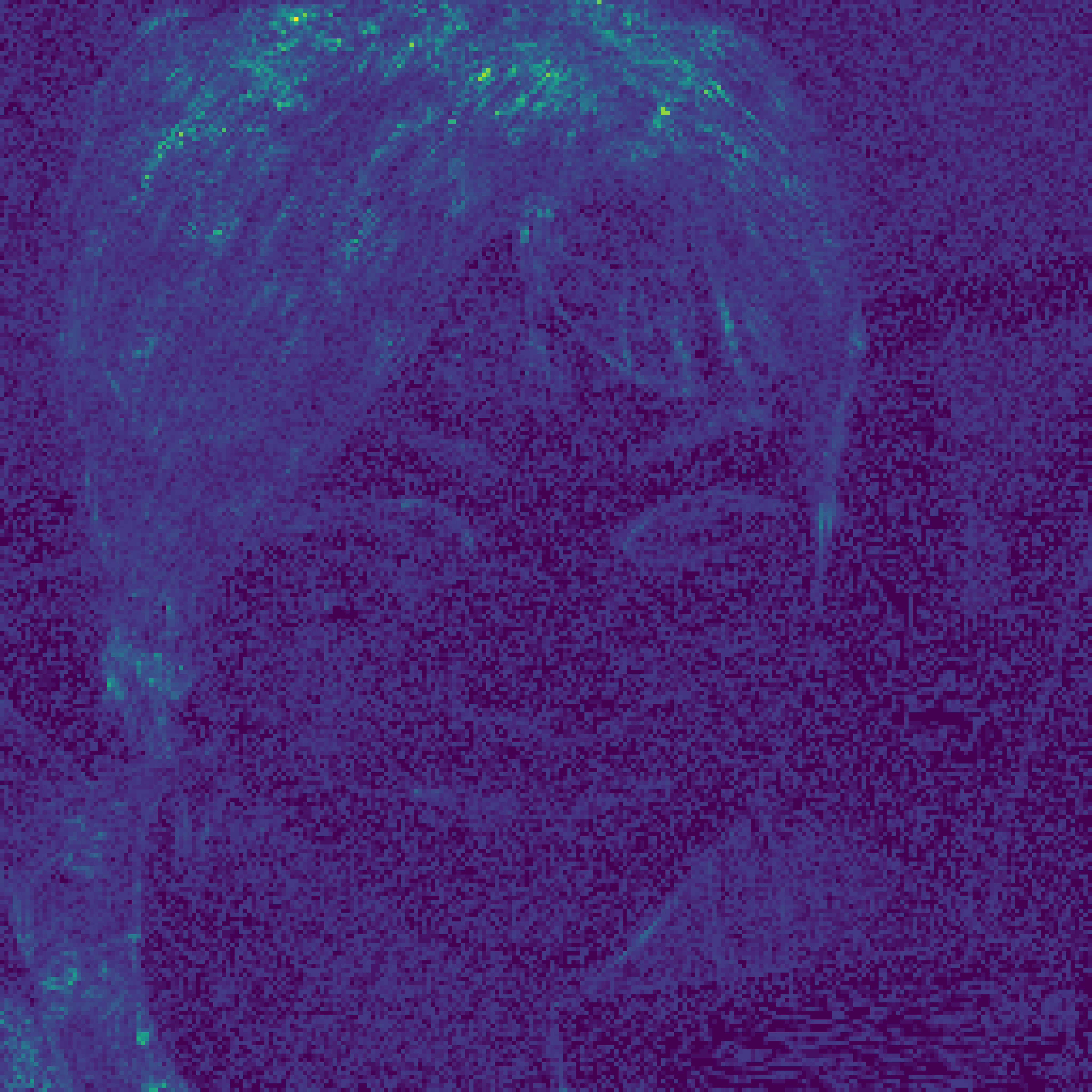} &
        \includegraphics[width=0.2\textwidth]{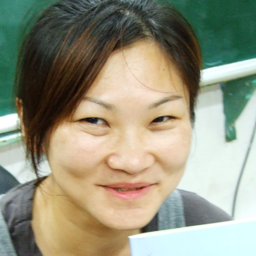} 
        \\ 
        \includegraphics[width=0.2\textwidth]{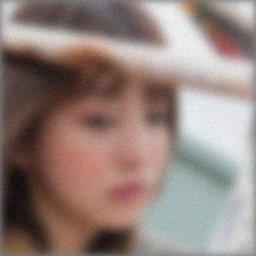} &
        \includegraphics[width=0.2\textwidth]{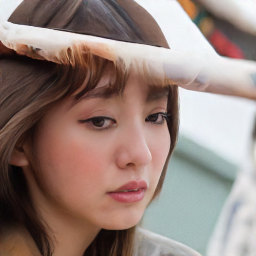} &
        \includegraphics[width=0.2\textwidth]{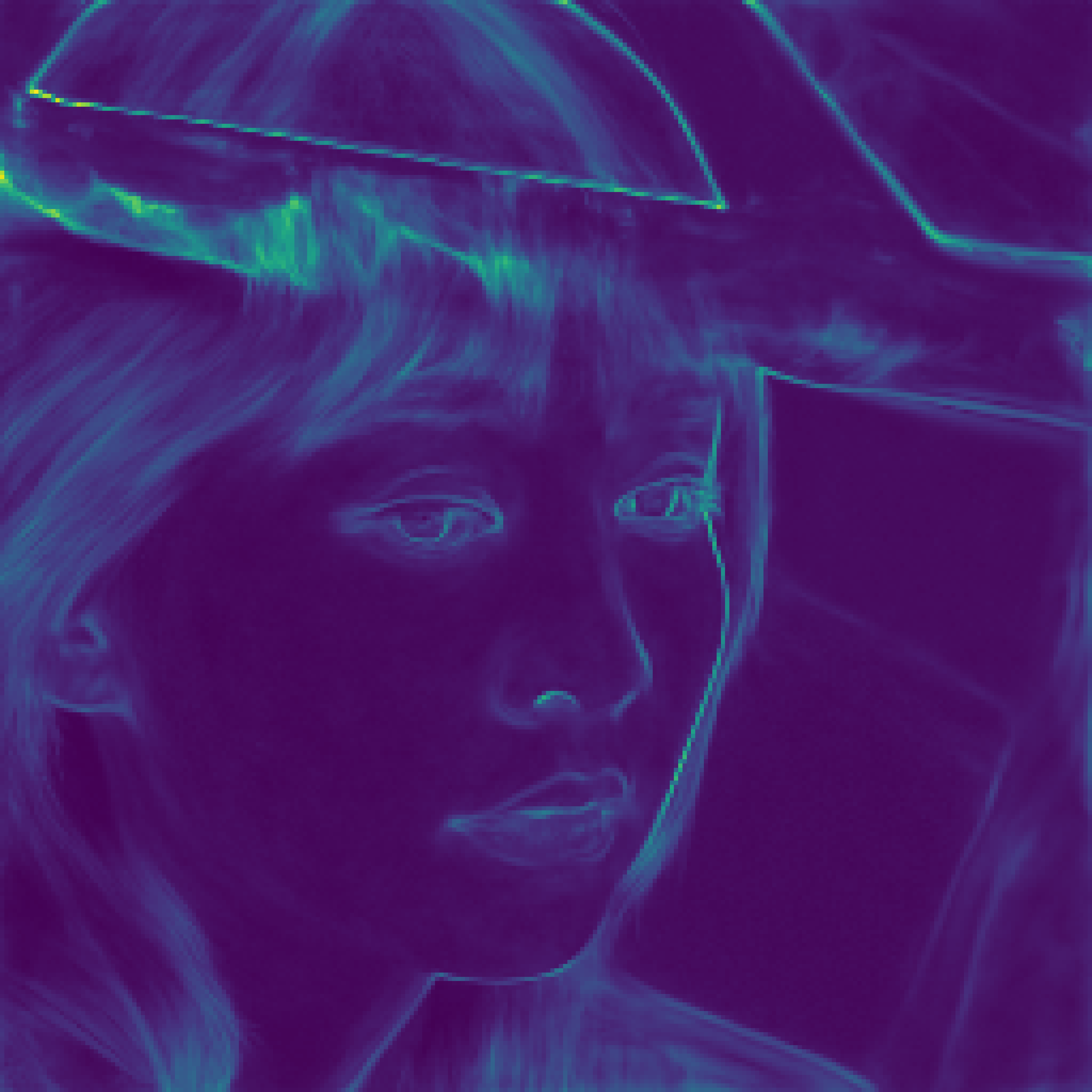} &
        \includegraphics[width=0.2\textwidth]{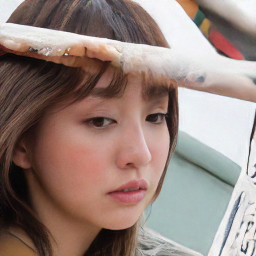} &
        \includegraphics[width=0.2\textwidth]{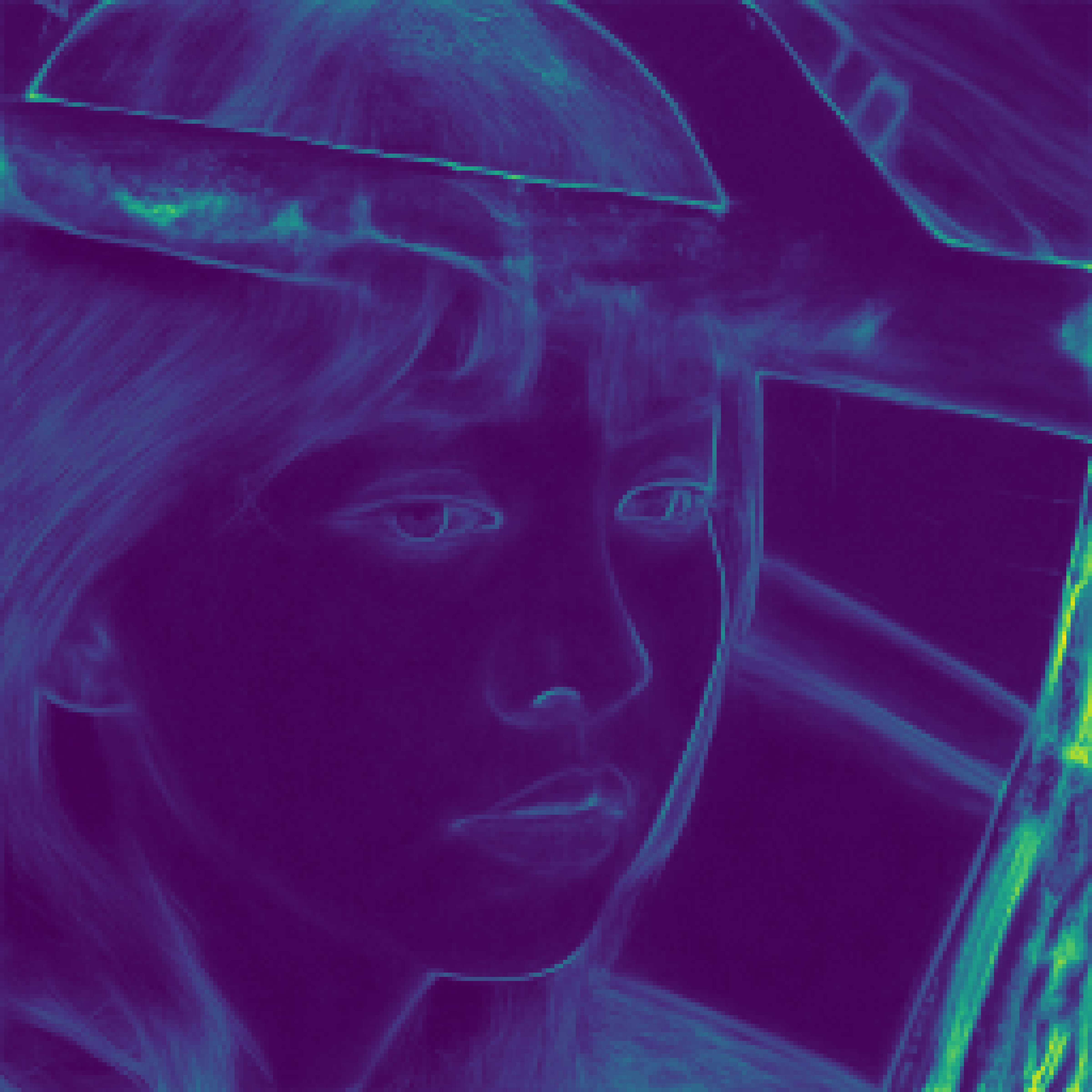} &
        \includegraphics[width=0.2\textwidth]{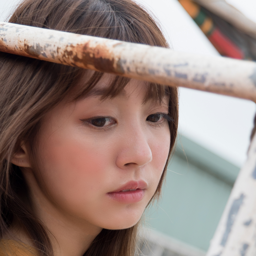}
        
    \end{tabular}
    \end{tabular}}
    \label{subfig:mesh}
    \caption{\textbf{Qualitative comparison of amortized inference on the FFHQ validation set.} 
  We compare our PPM framework against the KL-divergence-based baseline, DAVI~\cite{lee2024diffusion}. 
  The rows correspond to three inverse problems: \textbf{Motion Deblurring} (top), \textbf{$8\times$ Super-Resolution} (middle), and \textbf{Gaussian Denoising} (bottom). 
  For each task, we visualize the degraded observation, the reconstruction and pixel-wise uncertainty map from DAVI, the reconstruction and uncertainty map from our method, and the ground truth. 
  Our method yields significantly sharper structural details (e.g., facial features and hair textures) and provides more informative uncertainty estimates that capture the rich diversity of local details, whereas the IKL-based approach tends to over-smooth the results.}
    \label{fig:amortized-infer}
    \vspace{-0.1em}
\end{figure*}

\begin{table*}[htbp]
  \centering
  \caption{Quantitative comparison of PPM and baseline methods across computational photography tasks, including super-resolution, motion deblurring, and box inpainting, on FFHQ \cite{karras2019style} and ImageNet \cite{deng2009imagenet} (256×256 resolution). All methods use the same pretrained unconditional diffusion model, except RLSD \cite{zilberstein2024repulsive}, which employs Stable Diffusion \cite{rombach2022high}. Best results are shown in \textbf{bold}, and second-best are \underline{underlined}.}
    \begin{tabular}{cccccccccc}
    \toprule
    \multirow{2}[2]{*}{Method} & \multicolumn{3}{c}{Super Resolution} & \multicolumn{3}{c}{Motion Deblurring} & \multicolumn{3}{c}{Box Inpainting} \\
          & PSNR $\uparrow$ & SSIM $\uparrow$ & Diversity $\uparrow$ & PSNR $\uparrow$ & SSIM $\uparrow$ & Diversity $\uparrow$ & PSNR $\uparrow$  & SSIM $\uparrow$ & Diversity $\uparrow$\\
    \midrule
    \multicolumn{10}{l}{\textit{Sampling}} \\
    DDRM~\cite{chung2022diffusion}   & 22.62     & 0.62     & 0.001    & -     & -     & -     & 22.11     & 0.78     &  0.004 \\
    DPS~\cite{chung2022diffusion}   & 21.02     & 0.57     & 0.010     & 20.34     & 0.55     & \underline{0.006}     & 23.43     & 0.80     & 0.009 \\
    $\Pi$GDM~\cite{song2022pseudoinverse}  & 23.92     & 0.67     & \underline{0.011}     & 25.82     & 0.75     & \underline{0.006}     & 23.25     & 0.86     & \underline{0.010} \\
    \midrule
    \multicolumn{10}{l}{\textit{Particle variational inference}} \\
    RED-Diff~\cite{mardani2023variational} & \underline{26.54}     & \underline{0.76}     & 0.001     & \underline{29.02}     & \underline{0.84}     & 0.001     & 24.69     & 0.87     & 0.002 \\
    RLSD~\cite{zilberstein2024repulsive}  & \textbf{27.28}     & \textbf{0.79}     &0.003     & 25.64     & 0.82     & 0.002     & \underline{28.27}     & \underline{0.93}     & 0.006 \\
    Ours(VI)  & 25.63     & 0.72     & \textbf{0.013}     & 28.18     & 0.83     & \textbf{0.009}     & \textbf{28.73}     & \textbf{0.97}     & \textbf{0.016} \\
    \midrule
    \multicolumn{10}{l}{\textit{Amortized inference}} \\
    DAVI~\cite{lee2024diffusion}  & 24.49     & 0.70     & 0.006     &  27.69     & 0.82     & 0.003     & 24.58     & 0.83     & 0.005 \\
    Ours(AI)  & \textbf{24.85}     & \textbf{0.73}     & \textbf{0.008}     &  \textbf{29.17}     & \textbf{0.85}     & \textbf{0.004}     & \textbf{25.26}     & \textbf{0.88}     & \textbf{0.008} \\
    \bottomrule
    \end{tabular}%
  \label{tab:quan-results}%
\end{table*}%

\paragraph{Motion Deblurring, Super Resolution, and Gaussian Denoising}

For motion deblurring, we follow \cite{chung2022diffusion, zilberstein2024repulsive} by convolving each image with a randomly sampled $61 \times 61$ motion kernel (variance = $0.3^2$). For super-resolution, we downsample images by a factor of 8. Additionally, for Gaussian denoising, we corrupt the images with additive white Gaussian noise with a standard deviation of $\sigma=0.2$. We evaluate PPM on these tasks using FFHQ and ImageNet validation sets, comparing against DPS \cite{chung2022diffusion}, $\Pi$GDM \cite{song2022pseudoinverse}, RED-Diff \cite{mardani2023variational}, RLSD \cite{zilberstein2024repulsive}, and the amortized baseline DAVI~\cite{lee2024diffusion}.

Figure \ref{fig:denoisingdeblurring} shows that PPM delivers sharper, more observation‐consistent reconstructions than DPS, evident in fine details such as the elephant’s tusks and the balloon’s logo—thanks to PPM’s exact likelihood term versus DPS’s approximation. Compared to VI methods (RED-Diff and RLSD), PPM also excels at uncertainty quantification: its standard‐deviation maps accurately reflect positional uncertainty (e.g., the elephant’s tusks, the boy’s facial edge), whereas RED-Diff collapses modes and RLSD introduces p persistent speckle artifacts due to latent‐space optimization. 

Figure~\ref{fig:amortized-infer} presents a qualitative comparison of the amortized inference results on the FFHQ dataset. We contrast our PPM framework against the KL-divergence-based baseline, DAVI. The pixel-wise uncertainty maps demonstrate that our method captures meaningful posterior diversity, particularly in ambiguous regions (e.g., edges and textures), whereas the KL-based approach exhibits signs of posterior collapse with largely suppressed uncertainty.

Table \ref{tab:quan-results} confirms that PPM achieves the best overall balance of PSNR, SSIM, and diversity. While RLSD slightly outperforms in PSNR/SSIM for super‐resolution—due to its higher‐resolution Stable Diffusion prior (512 × 512)—its artifacts undermine visual quality. MCMC sampling methods like DPS match PPM’s diversity but fall short in fidelity. In summary, PPM consistently outperforms both variational and sampling-based baselines across computational photography tasks, delivering superior reconstruction quality and reliable uncertainty estimates.

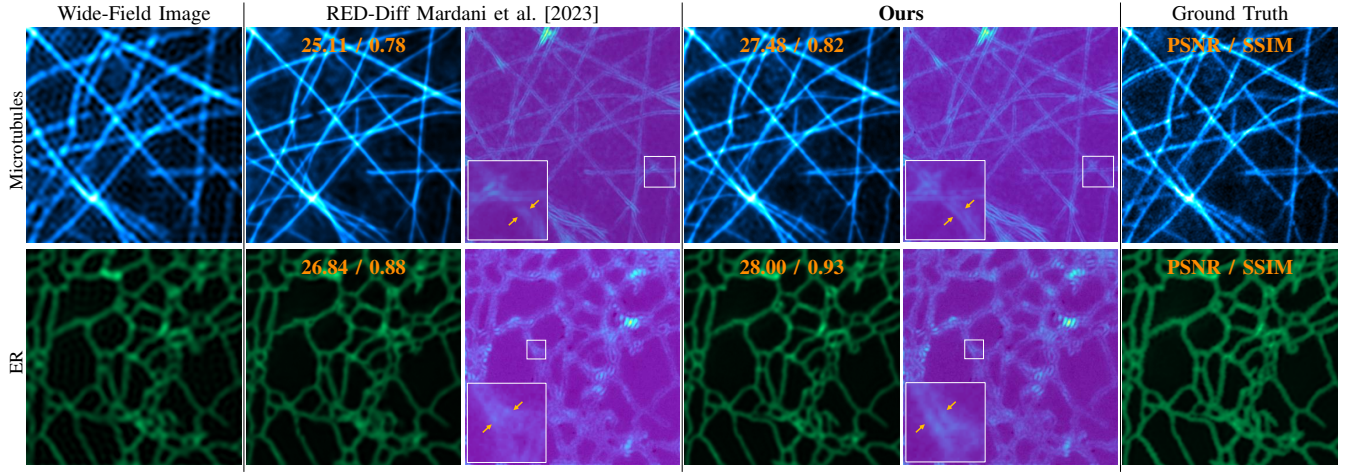
\begin{figure*}[t]
    \centering
    \setlength{\tabcolsep}{1pt}
    \setlength{\fboxrule}{1pt}
    \resizebox{0.98\textwidth}{!}{
    \begin{tabular}{c}
    \begin{tabular}{cc|cc|cc|c}
        & Wide-Field Image & \multicolumn{2}{c|}{RED-Diff~\cite{mardani2023variational}} & \multicolumn{2}{c|}{{\textbf{Ours}}} & Ground Truth
        \\
        \begin{turn}{90} \,\,\,\,\,\,\,\,\,\,\,\,\,\,\,\small{{Microtubules}} \end{turn} &
        \includegraphics[width=0.2\textwidth]{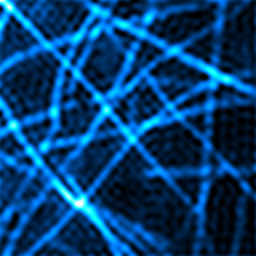} &
        \begin{tikzpicture}
            \node[anchor=south west, inner sep=0] (image) {\includegraphics[width=0.2\textwidth]{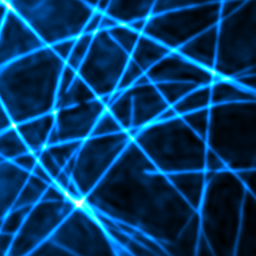}};
            \node[text=darkorange, font=\normalsize\bfseries, anchor=center] at ([yshift=-0.3cm]image.north) {25.11 / 0.78};
        \end{tikzpicture} &
        \includegraphics[width=0.2\textwidth]{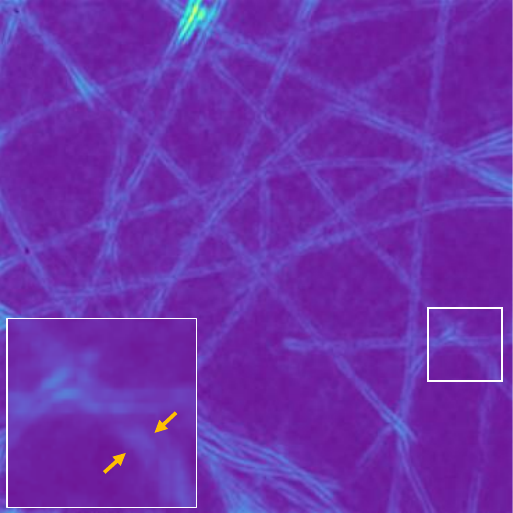} &
        \begin{tikzpicture}
            \node[anchor=south west, inner sep=0] (image) {\includegraphics[width=0.2\textwidth]{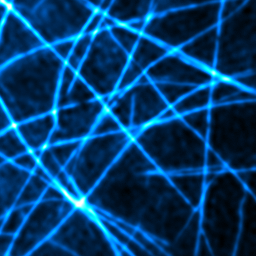}};
            \node[text=darkorange, font=\normalsize\bfseries, anchor=center] at ([yshift=-0.3cm]image.north) {27.48 / 0.82};
        \end{tikzpicture} &
        \includegraphics[width=0.2\textwidth]{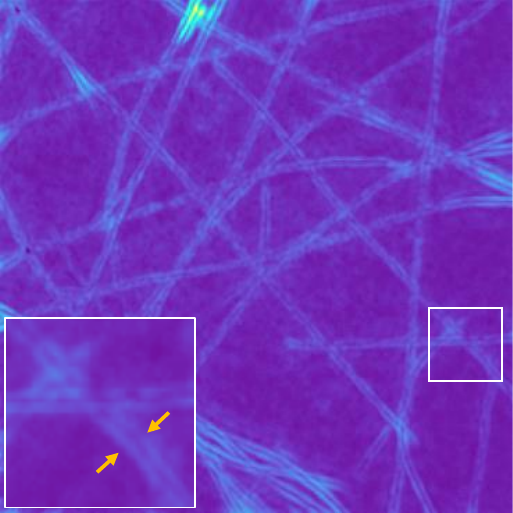} &
        \begin{tikzpicture}
            \node[anchor=south west, inner sep=0] (image) {\includegraphics[width=0.2\textwidth]{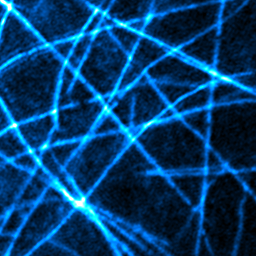}};
            \node[text=darkorange, font=\normalsize\bfseries, anchor=center] at ([yshift=-0.3cm]image.north) {PSNR / SSIM};
        \end{tikzpicture}
        \\
        \begin{turn}{90} \,\,\,\,\,\,\,\,\,\,\,\,\,\,\,\,\,\,\,\,\,\,\,\,\,\,\small{{ER}} \end{turn} &
        \includegraphics[width=0.2\textwidth]{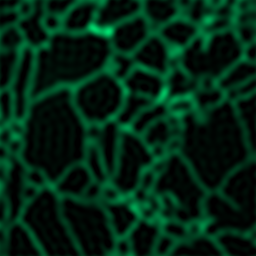} &
        \begin{tikzpicture}
            \node[anchor=south west, inner sep=0] (image) {\includegraphics[width=0.2\textwidth]{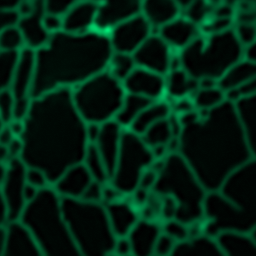}};
            \node[text=darkorange, font=\normalsize\bfseries, anchor=center] at ([yshift=-0.3cm]image.north) {26.84 / 0.88};
        \end{tikzpicture} &
        \includegraphics[width=0.2\textwidth]{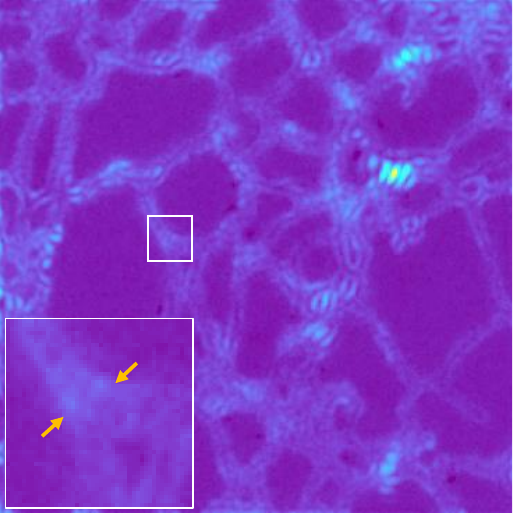} &
        \begin{tikzpicture}
            \node[anchor=south west, inner sep=0] (image) {\includegraphics[width=0.2\textwidth]{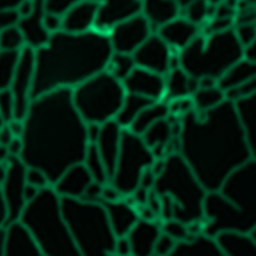}};
            \node[text=darkorange, font=\normalsize\bfseries, anchor=center] at ([yshift=-0.3cm]image.north) {28.00 / 0.93};
        \end{tikzpicture} &
        \includegraphics[width=0.2\textwidth]{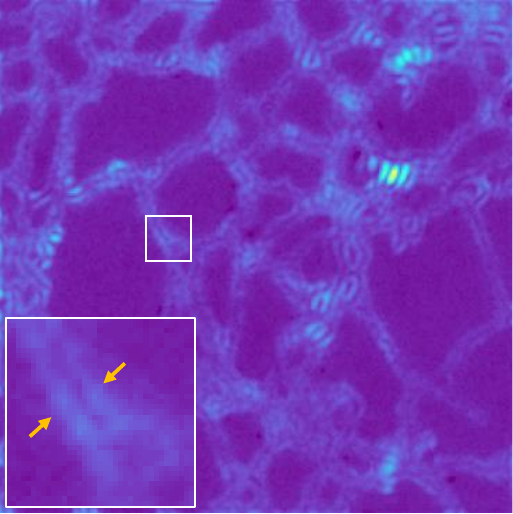} &
        \begin{tikzpicture}
            \node[anchor=south west, inner sep=0] (image) {\includegraphics[width=0.2\textwidth]{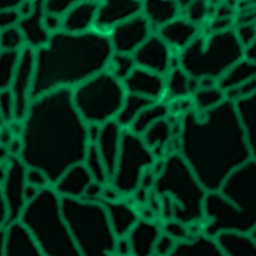}};
            \node[text=darkorange, font=\normalsize\bfseries, anchor=center] at ([yshift=-0.3cm]image.north) {PSNR / SSIM};
        \end{tikzpicture}
    \end{tabular}
    \end{tabular}}
    \caption{\textbf{Fluorescent super-resolution microscopic imaging results.} We compare our method with RED-Diff on microscopic images of Microtubules and ER samples. For each method, the reconstruction (with PSNR/SSIM scores) and its corresponding uncertainty map are reported. Our uncertainty maps accurately characterize the physical blur caused by the Point Spread Function in biological imaging: the uncertainty is lower at the center of the reconstructed structures and higher at the edges, effectively capturing the transition between the confirmed structures and the background.}
    \label{fig:microscopic}
    \vspace{-0.1em}
\end{figure*}

\subsection{Super-resolution Fluorescent Microscopic Imaging}
Beyond standard computational photography tasks, we also applied our method to real-world scientific imaging challenges in biomedicine and astronomy. We evaluated PPM on super-resolution fluorescent microscopy, a critical tool for visualizing subcellular structures. Here, the observation $y$ is a wide-field microscope image (approximately 200 nm resolution), whose measurement model is
\begin{equation}
y = \text{PSF} \circledast x + n,
\end{equation}
where $x$ is the underlying high-resolution fluorescence signal, PSF is the microscope point-spread function, and $n$ is additive Gaussian noise. Accurately recovering $x$, along with precise uncertainty quantification, is essential for resolving the fine-grained dynamics of subcellular structures, such as organelles, and their interactions. In our experiment, we primarily benchmark PPM with RED-Diff. Both methods use the same diffusion prior pretrained on the BioSR dataset \cite{aali2023solving}, which comprises over 10,000 $256\times256$ super-resolution images of diverse subcellular structures—microtubules, endoplasmic reticulum (ER), clathrin-coated pits (CCPs), and F-actin—captured with a structured illumination microscope (approximately 100 nm resolution).

Figure~\ref{fig:microscopic} demonstrates that PPM delivers higher-fidelity reconstructions than RED-Diff, faithfully rendering thin filaments in microtubules and the mesh-like ER. Quantitatively, PPM also achieves superior PSNR and SSIM scores. Crucially, PPM’s uncertainty maps align with imaging physics: confidence peaks at structure centers and decreases toward blurred edges, revealing hollow structures. This behavior reflects the influence of microscope’s PSF, which preserves feature presence while smearing precise boundaries. PPM accurately captures this boundary uncertainty, whereas RED-Diff’s estimates fail to indicate edge ambiguity. These results underscore PPM’s reliability for nanometer-scale biomedical imaging, where uncertainty quantification is indispensable.

\subsection{Radio Interferometric Black Hole Imaging}

We applied PPM to reconstruct and quantify uncertainty in black hole images from very long baseline interferometry (VLBI) measurements. Using a general‐relativistic magnetohydrodynamics (GRMHD) simulated Sagittarius A$^\star$ black hole image, we emulate a synthetic observation of the Event Horizon Telescope (EHT) array, which comprises nine telescopes worldwide to form an Earth‐sized interferometer. Ignoring atmospheric turbulence, the measurement model that maps the true image $x$ to the observed visibilities $y$ is
\begin{equation}
y = M \mathcal{F}\{x\} + n,
\end{equation}
where $\mathcal{F}$ is the Fourier transform, $M$ selects the measured frequency components, and $n$ is additive Gaussian noise. Because the EHT sampling is extremely sparse in the Fourier domain (Fig.~\ref{fig:eht}(a)), this defines a highly ill‐posed inverse problem: enforcing only data consistency yields the classical dirty image, riddled with sidelobe artifacts (top right of Fig.~\ref{fig:eht}).

Robust uncertainty quantification is therefore critical before making scientific inferences. 
Our PPM reconstruction follows the InverseBench~\cite{zheng2025inversebench} protocol, using a diffusion prior trained on approximately 50,000 synthetic black hole images. 
Figure~\ref{fig:eht} (b) shows the ground‐truth GRMHD image, the target blurred to the EHT’s resolution, 16 independent PPM posterior samples, and the resulting mean reconstruction alongside its standard‐deviation map. 
We further compare PPM with baselines in Fig.~\ref{fig:eht}(c). 
While RED-Diff suffers from severe mode collapse (indicated by the suppressed uncertainty map) and DPS yields blurry reconstructions, only PPM faithfully captures the key morphology—ring diameter, azimuthal position of the bright crescent, and the black hole’s swirling signature, demonstrating PPM’s ability to deliver accurate reconstructions with reliable uncertainty estimates in challenging VLBI black hole imaging scenarios. 

\begin{figure}[htbp]
  \centering 
  \includegraphics[width=0.48\textwidth]{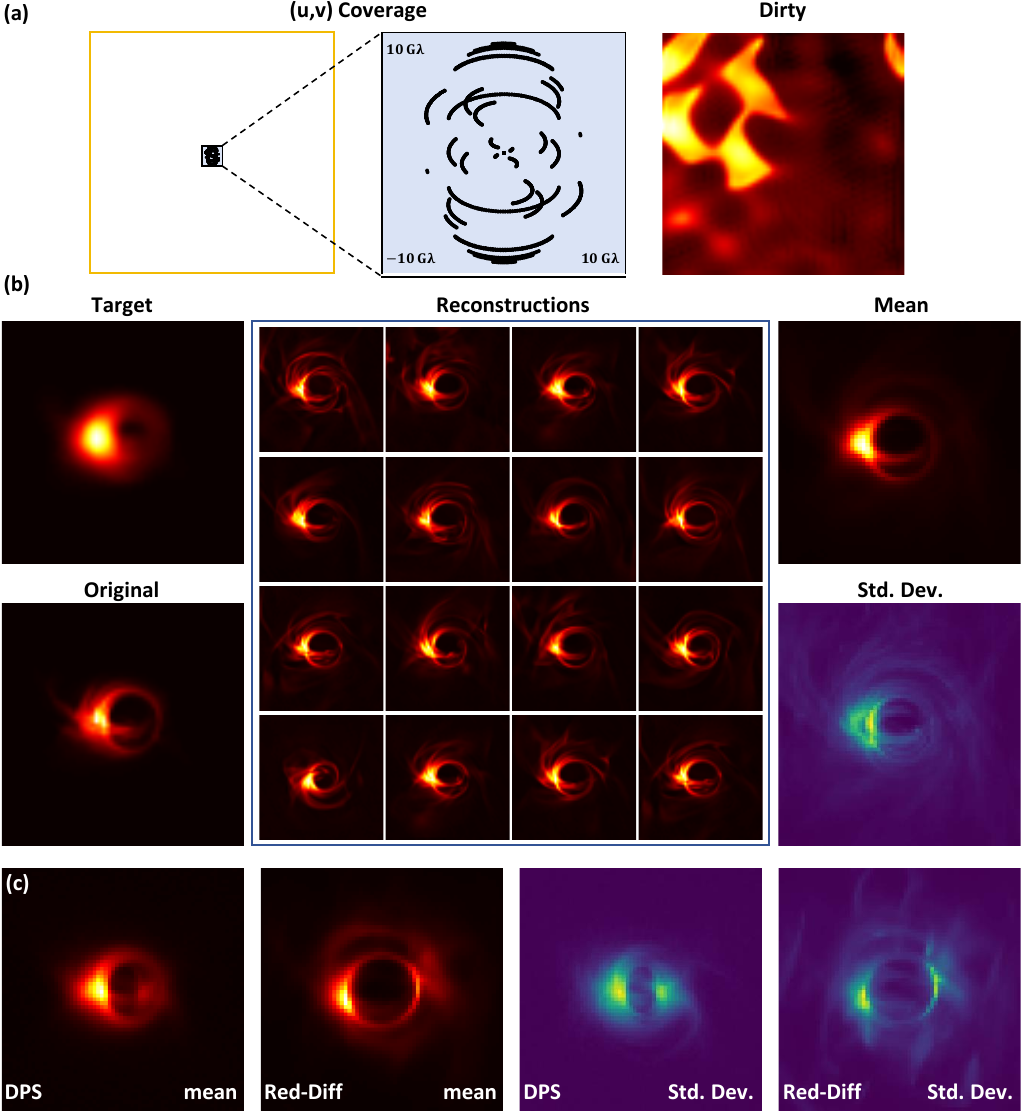} 
  \caption{\textbf{Black hole interferometric imaging from synthetic EHT observations.} This highly ill-posed inverse problem recovers an image from the sparse Fourier samples of a VLBI array (top left). (a) shows the EHT’s $(u, v)$ coverage and the ``dirty'' image reconstructed solely from observations, without any image priors. (b) presents the ground-truth GRMHD image, the target blurred to EHT resolution, 16 independent PPM posterior samples, and the resulting mean reconstruction with its standard-deviation map. 
  PPM accurately captures critical features, the ring structure and bright crescent, while providing reliable uncertainty estimates. 
  (c) Comparison with baselines. We report the mean reconstruction and pixel-wise standard deviation for baselines. 
    Consistent with our theoretical analysis, {RED-Diff exhibits severe mode collapse}, characterized by a suppressed standard deviation map. 
    While DPS captures uncertainty, its reconstruction lacks sharpness. 
In contrast, PPM achieves superior fidelity with a physically meaningful uncertainty distribution that accurately captures the structural variance of the black hole shadow.
  }
  \label{fig:eht} 
\end{figure}

\section{Conclusion}
\label{sec:conclusion}
In this paper, we presented Principled Posterior Matching (PPM), a principled framework that addresses the fundamental limitations of existing variational diffusion-based inverse problem solvers. By identifying that the mode collapse in prior works stems from biased approximations of the KL divergence, we proposed a rigorous alternative based on the integration of Fisher divergence. This enables an unbiased gradient estimator, allowing for the exact minimization of the variational objective without structural collapse.
PPM unifies variational and amortized inference, enabling both faithful posterior recovery and efficient, unsupervised generation. Validated across computational and scientific imaging tasks—including microscopy and black-hole imaging—PPM consistently outperforms baselines. Its superior fidelity and reliable uncertainty quantification establish it as a robust foundation for trustworthy imaging.




\bibliography{citation}
\bibliographystyle{plainnat}



\begin{IEEEbiography}[{\includegraphics[width=1in,height=1.25in,clip,keepaspectratio]{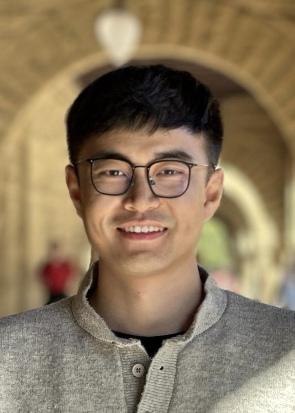}}]{Weimin Bai}
received the B.S. degree from China Agricultural University and the M.E. degree from the School of Computer Science, Southeast University. He is currently pursuing a PhD degree with the Academy for Advanced Interdisciplinary Studies, Peking University, China. His main research interests include generative models and inverse problems.
\end{IEEEbiography}

\begin{IEEEbiography}[{\includegraphics[width=1in,height=1.25in,clip,keepaspectratio]{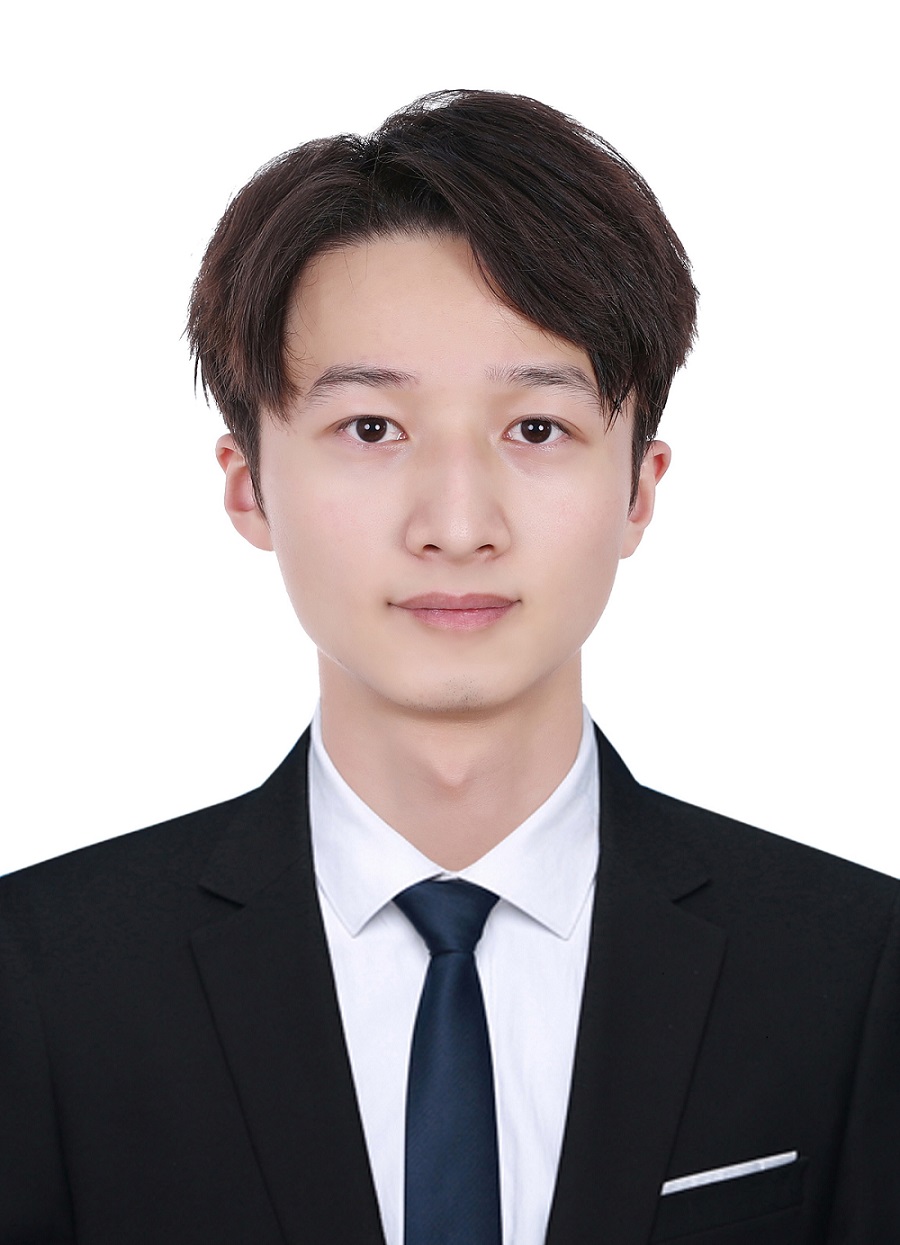}}]{Yuxuan Gu} is currently pursuing a master's degree at the School of Software and Microelectronics, Peking University, where his research interests primarily focus on the distillation of diffusion models, inverse problem solving, and the development of intelligent agents.
\end{IEEEbiography}

\begin{IEEEbiography}[{\includegraphics[width=1in,height=1.25in,clip,keepaspectratio]{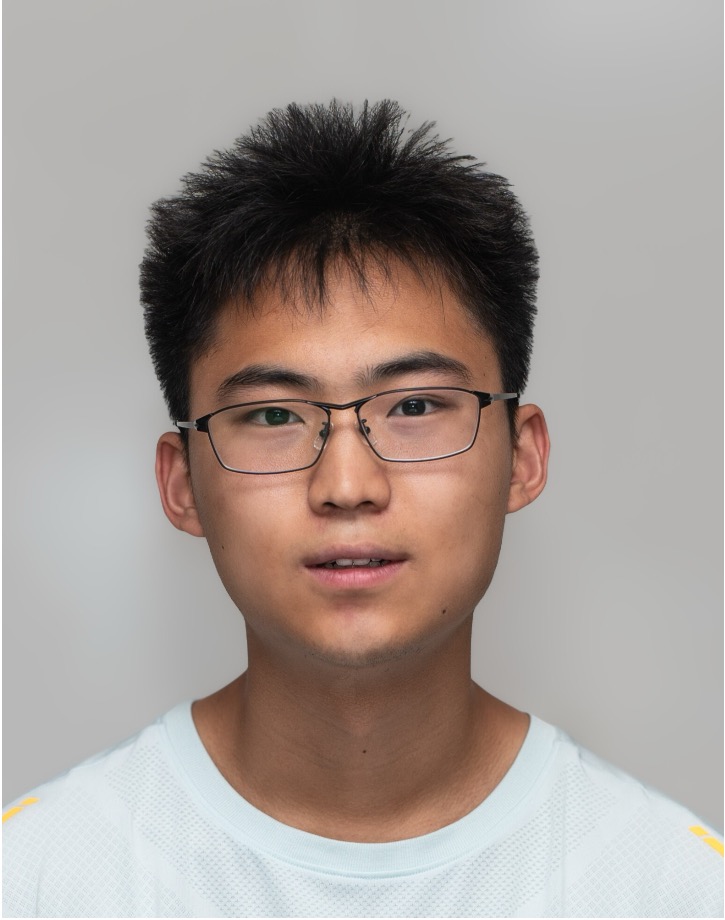}}]{Yifei Wang} received the BS degree in artificial intelligence from Peking University, Beijing, China, in 2025. He is currently working toward the doctoral degree with Rice University. His current research interests include representation learning and generative modeling.
\end{IEEEbiography}

\begin{IEEEbiography}[{\includegraphics[width=1in,height=1.25in,clip,keepaspectratio]{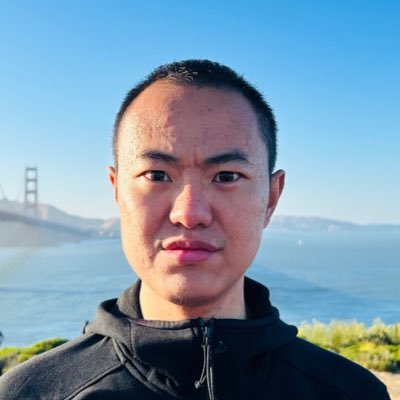}}]{Weijian Luo}
is a RedStar Senior Research Scientist at the Humane Intelligence (hi) Lab, Xiaohongshu Inc. He received his B.S. degree from the University of Science and Technology of China and his M.S. and Ph.D. degrees in statistics and generative modeling from Peking University. His research focuses on building generic AGI systems. He has published more than 15 academic papers in international journals and conferences, such as IEEE TPAMI, Transactions on Machine Learning Research, NeurIPS, ICML, ICLR, CVPR, etc.
\end{IEEEbiography}

\begin{IEEEbiography}[{\includegraphics[width=1in,height=1.25in,clip,keepaspectratio]{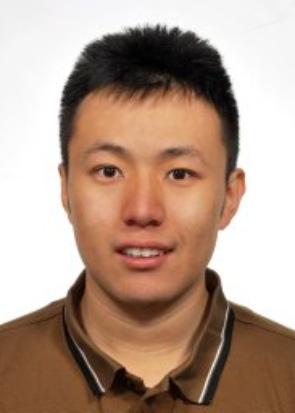}}]{He Sun}
is an Assistant Professor in the College of Future Technology and the National Biomedical Imaging Center, Peking University. He obtained his Ph.D. from Princeton University in 2019 and his bachelor's degree from Peking University in 2014. Prior to joining the faculty of Peking University, he was a postdoctoral researcher and an Amazon AI4Science Fellow at California Institute of Technology. His research primarily focuses on computational imaging, which tightly integrates optics, control, signal processing and machine learning to push the boundary of scientific imaging. His past work has contributed to multiple challenging science missions, including the Event Horizon Telescope for black hole interferometric imaging, as well as to a range of biomedical imaging modalities such as ultrasound and computational microscopy.
\end{IEEEbiography}

\end{document}